\documentclass[11pt]{article}
\usepackage{amssymb}
\usepackage{amsmath}
\usepackage{amsthm}
\usepackage{url} 
\usepackage{cleveref}
\usepackage{natbib}
\usepackage[normalem]{ulem}

\usepackage{float}
\usepackage{lipsum}
\usepackage[ruled]{algorithm2e}
\usepackage{geometry}
\usepackage[title]{appendix}

\usepackage{booktabs}
\usepackage{soul}
\usepackage[authormarkup=none,todonotes={textsize=scriptsize}
]{changes}
\usepackage{subcaption} 
\usepackage{multirow}
\usepackage{xr}
\usepackage{tablefootnote}
\usepackage{pifont}
\newtheorem{theorem}{Theorem}[section]

\newtheorem{lemma}[theorem]{Lemma}
\newtheorem{corollary}[theorem]{Corollary}
\newtheorem{definition}[theorem]{Definition}
\newtheorem{assumption}[theorem]{Assumption}
\newtheorem{remark}[theorem]{Remark}
\input{macros.tex}

\setdeletedmarkup{\IfIsColored{\color{authorcolor!20}}((#1)}

\definechangesauthor[color=blue]{KL}
\definechangesauthor[color=purple]{SH}

\title{Sparse Max-Affine Regression}

\author{Haitham Kanj, Seonho Kim, and Kiryung Lee\thanks{The authors are with the Department of Electrical and Computer Engineering at the Ohio State University (corresponding author: kanj.7@osu.edu). This work was supported in part by NSF CAREER Award CCF-1943201.} \\
Department of Electrical and Computer Engineering\\
The Ohio State University, Columbus, OH, USA}

\begin{document}
\maketitle
\begin{abstract}
This paper presents Sparse Gradient Descent as a solution for variable selection in convex piecewise linear regression where the model is given as the maximum of $k$-affine functions $\mb x \mapsto \max_{j \in [k]} \langle \mb a_j^\star, \mb x \rangle + b_j^\star$ for $j = 1,\dots,k$. 
Here, $\{\mb a_j^\star\}_{j=1}^k$ and $\{b_j^\star\}_{j=1}^k$ denote the ground-truth weight vectors and intercepts. 
A non-asymptotic local convergence analysis is provided for Sp-GD under sub-Gaussian noise when the covariate distribution satisfies the sub-Gaussianity and anti-concentration properties.
When the model order and parameters 
are fixed, Sp-GD provides an $\epsilon$-accurate estimate given $\mathcal{O}(\max(\epsilon^{-2}\sigma_z^2,1)s\log(d/s))$ observations where $\sigma_z^2$ denotes the noise variance. This also implies the exact parameter recovery by Sp-GD from $\mathcal{O}(s\log(d/s))$ noise-free observations. 
Since optimizing the squared loss for sparse max-affine is non-convex, an initialization scheme is proposed to provide a suitable initial estimate within the basin of attraction for Sp-GD, i.e. sufficiently accurate to invoke the convergence guarantees. The initialization scheme uses sparse principal component analysis to estimate the subspace spanned by $\{\mb a_j^\star\}_{j=1}^k$, then applies an $r$-covering search to estimate the model parameters. A non-asymptotic analysis is presented for this initialization scheme when the covariates and noise samples follow Gaussian distributions. When 
the model order and parameters are fixed, this initialization scheme provides an $\epsilon$-accurate estimate given $\mathcal{O}(\epsilon^{-2}\max(\sigma_z^4,\sigma_z^2,1)s^2\log^4(d))$ observations. A new transformation named Real Maslov Dequantization (RMD) is proposed to transform sparse generalized polynomials into sparse max-affine models. The error decay rate of RMD is shown to be exponentially small in its temperature parameter. Furthermore, theoretical guarantees for Sp-GD are extended to the bounded noise model induced by RMD.  
Numerical Monte Carlo results corroborate theoretical findings for Sp-GD and the initialization scheme. 
\end{abstract}

% REQUIRED

%% Section Files:
\section{Introduction} \label{Intro}
%\paragraph{Sparse max-affine model:}
We consider a multivariate regression problem where the target variable $y \in \mathbb{R}$  depends nonlinearly on covariates in $\mb  x \in \mathbb{R}^d$, and noise $z \in \mathbb{R}$ as
\begin{equation} \label{eq:dataset}
     y = f(\mb x; \mb \theta^\star) + z   
\end{equation}
through a max-affine function $f: \mathbb{R}^d \to \mathbb{R}$ given by
\begin{equation} \label{eq:classconf}
\mb x \in \mathbb{R}^d \mapsto f(\mb x; \mb \theta^\star) = \max_{1\leq j\leq k} \left(\langle \mb a^\star_j, \mb x \rangle + b^\star_j \right),
\end{equation}
where $\mb \theta^\star \in \mathbb{R}^{k(d+1)}$ collects all ground-truth parameters $\{(\mb a_j^\star,b_j^\star)\}_{j=1}^k \subset \mathbb{R}^d \times \mathbb{R}$. 
The max-affine structure in \eqref{eq:classconf} induces a class of convex piecewise linear functions. 
The set of max-affine functions provides an efficient approximation of a class of smooth convex functions \citep{balazs2015near}. 
It is also considered a special instance of tropical algebra \citep{maragos2021tropical}. 
Furthermore, we assume that only up to $s$ variables in $\mb x$ contribute to the evaluation of each linear model in \eqref{eq:classconf}, i.e. 
\begin{equation} 
\label{eq:ind_sparsity}
\left|\mathrm{supp}(\mb a^\star_j) \right| \leq s, \quad \forall j = 1,\dots,k, 
\end{equation}
where $\mathrm{supp}(\cdot)$ denotes the index set for non-zero entries of the input vector. We denote such models as \textit{sparse max-affine}. 
We refer to the estimation of the ground-truth parameters in $\mb \theta^\star$ from noisy samples $\{(\mb x_i, y_i)\}_{i=1}^n$ under the sparsity constraint in \eqref{eq:ind_sparsity} as \textit{sparse max-affine regression}. 
This is equivalent to convex piecewise linear regression with variable selection. 
The estimation procedure naturally implements variable selection as the support of the estimate identifies the active covariate variables. 

%\paragraph{Applications: } 
{\color{black}
\paragraph{Motivation}
Estimating nonlinear functions under convexity constraints, known as \textit{convex regression}, constitutes a fundamental problem across multiple disciplines, including econometrics \citep{merton1992continuous}, geometric programming \citep{magnani2009convex}, and reinforcement learning \citep{hannah2014semiconvex}. Additional applications of convex regression can be found in circuit design \citep{hannah2012ensemble} and queuing theory \citep{chen2001fundamentals}. 
Known approaches for convex regression can be categorized by their key assumptions on the true underlying function $f:\mathbb R^d \to \mathbb R$ as follows:
\begin{enumerate}
    \item Non-parametric: $f$ satisfies minimal conditions like continuity or smoothness.
    \item Semi-parametric: $f$ is expressed as $\mb x \mapsto f(\mb x) = g(\mb U\mb x)$, with \textit{unknown} $\mb U \in \mathbb R^{K\times d}$ and \textit{unknown} nonlinear map $g: \mathbb R^K \to \mathbb R$. This is referred to as the multi-index model if $K>1$ or the single-index model if $K=1$. 
    \item Parametric: $f$ has a \textit{known} form involving a few unknown parameters.
\end{enumerate}
Each approach involves trade-offs between its advantages and limitations.
Training methods for non-parametric regression typically solve a quadratic program with $\mathcal{O}(n)$ inequality constraints, where $n$ denotes the number of samples \citep{balazs2015near,hannah2013multivariate}. 
While non-parametric convex regression offers broad applicability, it suffers from exponentially increasing sample complexity on the order of $\mathcal{O}(e^d)$ in the number of covariates $d$ and incurs a high computational cost of $\mathcal{O}(\mathrm{poly}(nd))$. 
In contrast, a semi-parametric method provided a sharper result on a further restricted convex single-index model with the generalization error rate of $\mathcal{O}\left(\frac{d^{2/5}}{n^{2/5}}\right)$ \citep{kuchibhotla2023semiparametric}. 
With the same goal, parametric methods employed an alternative compact model composed of piecewise linear functions with a fixed number of linear components, which can be considered as a special case of multi-index models with $g$ fixed to the max function. 
Computationally efficient solutions such as alternating partitioning and minimization \citep{magnani2009convex} and first-order methods \citep{kim2023maxaffine} have been proposed. 
When the number of linear components is fixed, and under certain covariate conditions that generalize the standard Gaussian model, the parametric method can achieve a computational cost of $\mathcal{O}(nd)$, with the required sample size $n$ scaling linearly as $\mathcal{O}(d)$ \citep{ghosh2021max,kim2023maxaffine}. 

In many real-world problems, the outcome of interest is often driven by a \textit{small} subset of key input factors, whose effects can be highly nonlinear. Variable selection seeks to identify this \textit{unknown} subset of active inputs by promoting sparsity in the regression coefficients.
In the high-dimensional settings, where the number of covariates $d$ is large, variable selection enhances the generalization performance and improves the interpretability of the learned model \citep{wasserman2009high}. 
However, establishing theoretical performance guarantees for variable selection in general nonlinear regression models is widely recognized in the literature as a significant challenge \citep{lafferty2008rodeo, bertin2008selection}. 
In particular, existing theoretical results are limited and focus on special restrictive cases. 
For example, under the smoothness of the nonlinear function given by a norm induced by the Fourier series, asymptotic consistency for \textit{support recovery} can be achieved with $n = \mathcal{O}(e^s\log d)$ by a non-parametric method \citep{comminges2012tight}, where $s$ is the number of relevant variables.
Regarding the semi-parametric approach, a generalization error rate decaying as $\mathcal{O}(\frac{s\log d}{n})$ has been established for the single-index model \citep{alquier2013sparse,radchenko2015high}. 
However, for the multi-index model ($K>1$), existing semi-parametric results are limited to establishing asymptotic consistency for support recovery \citep{wang2015variable}.
For the parametric approach, the parameter estimation error rate of $ O\left(\frac{s \log d}{n}\right)$ is attainable under the stringent assumption of a \textit{known} monotone single-index model \citep{yang2016sparse}. 
Variable selection has been studied for deep neural networks, which are also parametric models, and support recovery is shown to be asymptotically consistent \citep{yang2024enns}.  
In summary, variable selection improves the dependence in sample complexity on $d$ by substituting it with $s$, where $s \ll d$, up to an extra logarithmic factor in $d$. 

Convex regression is another instance of nonlinear regression where variable selection has been shown to provably improve the generalization performance.
A seminal work by Xu et al. \citep{xu2016faithful} showed that variable selection in convex non-parametric regression can be achieved with $n = \mathcal{O}(s^7 \log^3 d)$ under the assumption that the true underlying function is expressed as a superposition of univariate convex functions. 
This paper presents a variable selection scheme for convex regression on a more flexible subclass of convex functions at a lower sample complexity of $n = \mathcal{O}(s\log d)$.  
The max-affine function in \eqref{eq:classconf} indeed provides a parametric model for convex piecewise linear multivariate functions.  
As we will discuss later, the max-affine model is known to be equivalent to a convex ReLU neural network (NN) \citep{zhang2018tropical}. The approximation power of ReLU-NN is well studied in the literature \citep{schmidt2020nonparametric}. 
To the author's best knowledge, this is the first work to consider a parametric approach to sparse convex regression.
The sample complexities of convex regression methodologies with and without variable selection are summarized in Table \ref{Table:convex_reg_ref}.

\begin{table}[h!]
\centering
\resizebox{0.95\textwidth}{!}{
\begin{tabular}{c|c|c|c}
\toprule
 & Variable selection &   Assumptions other than convexity&  Sample complexity \\ 
\midrule

    \citet{balazs2015near}  &  \ding{55}    &  continuous \& bounded gradients        & $\mathcal{O}(e^d)$ \\ \midrule

   \citet{xu2016faithful} & \ding{51}  &  superposition of univariate functions     & $\mathcal{O}(e^s \log d)$\tablefootnote{ Xu et al. also show that the support recovery can be done with $n= \mathcal{O}(s^7 \log^3 d)$. } \\ 
\midrule
\citet{kuchibhotla2023semiparametric}& \ding{55} & single-index model & $\mathcal{O}(d)$
\\ \midrule
 \citet{ghosh2021max}  & \ding{55}&  piecewise linear             & $\mathcal{O}(d)$ \\ \midrule

   \textbf{This paper} &\ding{51} & piecewise linear         & $\boldsymbol{\mathcal{O}(s \log d)}$ \\ 
\bottomrule
\end{tabular}}
\caption{Comparison of methodologies for convex regression with associated assumptions on the true convex function and non-asymptotic sample complexities. }
\label{Table:convex_reg_ref}
\end{table} 

Next, we delve into further detail regarding the connection between sparse max-affine models and deep ReLU-NNs. 
The work by \citep{zhang2018tropical} shows that a max-affine model is equivalent to an $L$-layer ReLU-NN with nonnegative hidden layer weights. Furthermore, \citep[Theorem 5.4]{zhang2018tropical} states that any max-affine model with order $k$ can be written as an $L$-layer ReLU-NN with $L \leq 2+ \lceil\log_2 k\rceil$. Therefore, training a sparse max-affine model is equivalent in spirit to ReLU-NN training with a regularization term that enforces sparsity (e.g., an $\ell_1$ penalty term on the network weights). Variable selection techniques for NN are employed to handle high-dimensional datasets with a small number of samples (see \citep{yang2024enns} for a recent review). However, non-asymptotic theoretical guarantees for NN are limited to restrictive cases such as the two-layer model in \citep{li2020learning} without sparsity. Variable selection in NN training is shown to be consistent only when the number of samples is infinitely large \citep[Theorem 4.2]{yang2024enns}. 
Furthermore, the sparse max-affine model provides better interpretability as the contributing weights of each covariate are clearly shown. We also note that the sparse max-affine model is a natural object from tropical algebra called the tropical rational functional. Therefore, compared to ReLU-NN, sparse max-affine models inherit the existing geometric results from tropical algebra \citep{zhang2018tropical}.

Another main motivation for studying the sparse max-affine model is that it provides an efficient approximation to a broad class of generalized sparse polynomials \citep{jameson2006counting}. This model is written as
\begin{equation} \label{eq:poly_into}
    w
    = g(u_1,\dots,u_d) = \sum_{j=1}^k c_j^\star \left(\prod_{\ell\in\mathcal{I}_j} u_l^{\alpha^\star_{j,\ell}}  \right),
\end{equation}
with \textit{real-valued exponents} $\{\alpha^\star_{j,\ell}\}_{\ell\in \mathcal{I}_j} \subset \mathbb{R}$ instead of the integer-valued exponents in a sparse polynomial, a real-valued coefficient $c_j$, and the active monomial indices $\mathcal{I}_j \subseteq [d]$ for all $j \in [k]$. Variable selection here refers to finding $\{\mathcal{I}_j\}_{j=1}^k$. When the coefficients $c_j$'s are \textit{positive}, \eqref{eq:poly_into} reduces to a posynomial (a special case of the generalized polynomial). Employing a $\log$-$\log$ mapping, called \textit{Maslov dequantization}, one can transform a sparse posynomial into a sparse max-affine function \citep{maragos2021tropical,boyd2004convex}. 
Even with the positivity constraint on the weights $c_j$'s, the model by \eqref{eq:poly_into} is not restricted to a convex function, but does not provide sufficient flexibility to model complex relations in real-world applications. 
Furthermore, there has been no non-asymptotic approximation analysis for the Maslov dequantization, even for sparse posynomials. 
This paper will also further extend this approach so that a sparse max-affine model can approximate the generalized sparse polynomial model in \eqref{eq:poly_into} without requiring the positivity constraint backed by a non-asymptotic error bound. 

}

\paragraph{Contributions} 
In this paper, we present theoretical convergence guarantees for the \textit{sparse gradient descent algorithm} (Sp-GD) that implements variable selection for max-affine regression. 
Since learning the regression parameters is cast as a nonconvex optimization problem, it is important for Sp-GD to start from a suitable initialization. 
We propose an initialization scheme leveraging the sparsity structure backed by performance guarantees on its error bound.  
The theoretical results for both Sp-GD and the initialization scheme, presented respectively in pseudo Theorems \ref{Theo:psuedolocal} and \ref{Theo:pseudo_init}, show that the sample complexity is governed by the number of active variables $s$ instead of the total number of variables $d$. 
{
\color{black} Furthermore, we propose the \textit{Real Maslov dequantization} to approximate any generalized sparse polynomial via a sparse max-affine function. The quantization error of this approximation is presented in pseudo Theorem \ref{theo:approx_intro}. We also extend the theoretical convergence guarantees of Sp-GD to this context in Section \ref{sec:poly}.
}
Monte Carlo simulations in Sections~\ref{Sec:numspgd} and \ref{Sec:numinit} corroborate these theoretical guarantees.

\paragraph{Optimization algorithm} Sp-GD is a variant of the projected gradient descent method. Sp-GD differs from standard projected gradient descent by utilizing a generalized gradient. This gradient is an extension to non-smooth functions, such as the piecewise linear functions we are working with \citep{hiriart1979new,clarke1990optimization}. Sp-GD moves along the generalized gradient with a set of adaptive step sizes and then projects to the feasible set of sparse parameters defined in \eqref{eq:ind_sparsity}. 
Since the nonlinear least squares problem for the model in \eqref{eq:dataset} is non-convex, Sp-GD provides a local convergence guarantee presented in the following pseudo-theorem.  

\begin{theorem}[Informal]
Let the covariates and noise be sampled independently from Sub-Gaussian distributions. For fixed $k$ and ground-truth $\mb \theta^\star$ satisfying \eqref{eq:ind_sparsity}, with high proability, a suitably initialized Sp-GD converges linearly to an $\epsilon$-accurate estimate of $\mb \theta^\star$ given $\mathcal{O}\left(  \max(\epsilon^{-2} \sigma_z^2, 1) s\log (d/s) \right)$ observations, where $\sigma_z^2$ denotes the noise variance. 
    \label{Theo:psuedolocal}
\end{theorem}
This result can be compared to 
a line of research on plain max-affine regression without the sparsity constraint. 
%developed theoretical analysis and efficient computational algorithms to estimate $\mb \theta^\star$ in . 
The authors in \citep{ghosh2021max} presented non-asymptotic convergence analyses of the alternating minimization algorithm by \citep{magnani2009convex} under random covariates and independent stochastic noise assumptions. 
%One direction proposed an alternating minimization (AM) algorithm %\citep{magnani2009convex} and presented non-asymptotic convergence analysis \citep{ghosh2021max} with random covariates and independent stochastic noise. 
%Another direction was proposed using first-order methods including Gradient Descent (GD) and Stochastic Gradient Descent (SGD) 
Later, the authors in \citep{kim2023maxaffine} showed that stochastic gradient descent provides comparable sample complexities and estimation errors with faster convergence. 
%SGD converges faster than AM with comparable sample complexities. 
\begin{table}[h!] 
\centering
\resizebox{0.95\textwidth}{!}{
\begin{tabular}{c|c|c|c} 
 \toprule
  & Algorithm & Sample complexity & Step-size for theoretical guarantees \\ \midrule
  \citet{ghosh2021max} & alternating minimization &${\mathcal{O}}(d)$ & NA \\ \midrule 
  \citet{kim2023maxaffine} &  first order methods & ${\mathcal{O}}(d)$ & unspecified constant \\ \midrule
 \textbf{This paper} & sparse gradient descent& $\boldsymbol{\mathcal{O}\left(s \log d\right)}$ & adaptive formula  \\ \bottomrule
\end{tabular}}
\caption{{Comparison of max-affine regression algorithms in the sample complexity for exact parameter recovery from noiseless observations.}}
\label{Table:Algorithms}
\end{table}
Table \ref{Table:Algorithms} compares the required sample complexity for Sp-GD and previous parameter estimation algorithms for max-affine models. Sp-GD drops the sample complexity from $\tilde{\mathcal{O}}(d)$ required by previous algorithms to $\tilde{\mathcal{O}}(s \log (d/s))$ which is sub-linear. This marks a significant improvement, especially when the number of active variables $s$ is significantly smaller than the total number of variables $d$.
It is also worth noting that the convergence result in Theorem \ref{Theo:psuedolocal} applies to a practical implementation of Sp-GD. 
Specifically, the step sizes are given in an explicit form determined by the parameter estimates of the previous iteration.
The step size for the $j$th block will concentrate around the inverse of the probability where the $j$th linear model in \eqref{eq:classconf} achieves the maximum. 
%i.e. the dataset represents the piecewise linear sections uniformly, the proposed step sizes become approximately equal to $k$ thus guaranteeing faster convergence results. 
In contrast, the authors in \citep{kim2023maxaffine} used an unspecified constant step size to prove the local convergence of their first-order methods. 
Therefore, our step size strategy improves the max-affine regression strategy in \citep{kim2023maxaffine} even in the non-sparse case.

%\paragraph{Initializtion: } 

%A spectral method for max-affine regression has been proposed in \citep{chen2021spectral}. 
\paragraph{Initialization}Recall that Theorem~\ref{Theo:psuedolocal} requires a suitable initial estimate to guarantee local convergence of Sp-GD. 
To obtain this desired initialization, one may use the spectral method for max-affine regression presented by the authors in \citep{ghosh2021max}. 
Simply stated, their initialization scheme first employs principal component analysis (PCA) to estimate the span of $\{\mb a_j^\star\}_{j=1}^k$. 
Then a discrete search over an $r$-covering in the span of the principal components is applied. 
It requires $\mathcal{O}(\epsilon^{-2}d \log^3 d)$ observations for their method to guarantee an $\epsilon$-accurate estimate with $r = \mathcal{O}(\epsilon)$ for fixed $\mb  \theta^\star$ and noise level. This voids the gain due to Theorem~\ref{Theo:psuedolocal}. 
We propose a variant of the initialization by \citep{ghosh2021max} that substitutes PCA with sparse PCA (sPCA) and provides the desired initial estimate from fewer observations in the following special scenario: 
Suppose that the sparse coefficient vectors $\{\mb a_j^\star\}_{j=1}^k$ are simultaneously supported within a set of cardinality $s$, i.e. 
\begin{equation}\label{eq:joint_sparsity}
    \left|\bigcup_{j=1}^k\mathrm{supp}(\mb a^\star_j) \right| \leq s. 
\end{equation}
Then the following pseudo-theorem quantifies the gain via the modified initialization scheme under this scenario.
\begin{theorem}[Informal] \label{Theo:pseudo_init}
    Let the covariates and noise be sampled independently from Gaussian distributions. Fix $k$ and ground-truth $\mb \theta^\star$ to satisfy the joint sparsity assumption in \eqref{eq:joint_sparsity}. Then, with high probability, the initialization via sPCA and $r$-covering search provides an $\epsilon$-accurate estimate of $\mb \theta^\star$ given $\mathcal{O}(\epsilon^{-2}\max (\sigma_z^{4},\sigma_z^2,1)s^2\log^4d)$ observations and $r=\mathcal{O}(\epsilon)$ where $\sigma^2_z $ is the noise variance.
\end{theorem}
Compared to previous work \citep{ghosh2021max}, the initialization scheme leveraging the joint sparsity drops the linear dependence on the ambient dimension $d$ and replaces it with a quadratic dependence on $s$. 
Our implementation and theoretical analysis build on sPCA as a semi-definite program by \citep{vu2013fantope} where the estimation error decays as $\mathcal{O}(\sqrt{s^2 \log d/n})$. 
A recent framework improved the error bound to $\mathcal{O}(\sqrt{s \log d/n})$ under the assumption that data are obtained through a linear transform acting on special multivariate distributions \citep{Wang2014tighten}. 
However, this assumption is not satisfied by data generated with the max-affine model.

{\color{black}
\paragraph{Approximating generalized polynomials}
We overcome the main limitation of \textit{Maslov dequantization} by extending this transformation to real coefficients in \eqref{eq:poly_into} via the Real \textit{Maslov dequantization} (RMD), which is written as
\begin{equation}
    \label{eq:RMD}
    y = \mathrm{Re}\{\varsigma \log w\}, \quad x_l = \varsigma \log u_l, \quad \forall l \in [d]
\end{equation}
where $\varsigma>0$ is a chosen temperature parameter and $\log (re^{i\theta}) = \log r + i\mathrm{mod}(\theta, 2\pi)$. We will use these definitions to state the following theorem on the approximation error of RMD.
\begin{theorem}[Informal]
    \label{theo:approx_intro}
    Consider the generalized polynomial model in \eqref{eq:poly_into} and its transformation via RMD in \eqref{eq:RMD} for some $\varsigma>0$. Collect the transformed covariates in $\mb x = [x_1; \ldots;x_d]$. The transformed target can be written as
    \begin{equation*}
        y = \max_{j \in [k]}\langle\mb \theta_j^\star, [\mb x;1] \rangle+z_\varsigma,
    \end{equation*}
    where $\mb \theta_j^\star = [\alpha_{j,1}^\star;\ldots; \alpha_{j,d}^\star; \varsigma \log|c_j^\star|]$ and $z_\varsigma$ is the dequantization error which decays exponentially in $\varsigma$, i.e. $z_\varsigma\leq c\exp(\varsigma^{-1})$. Consequently, when $\varsigma \to 0$, RMD reduces to a sparse-max affine model.
\end{theorem}
The class of models in \eqref{eq:poly_into} is very rich and has been applied to various contexts in finance and economics. 
For example, in labor economics \citep[Eq. 15]{zhao2016distributed}, the Gross Domestic Product (GDP), $G \in \mathbb R$, is written as a multivariate generalized polynomial 
\begin{equation}
\label{eq:labor_eq_gdp}
G = c_1L^\alpha H^\theta S^\gamma D^\delta +c_2K + c_3SD/K + c_4
\end{equation}
with coefficients $\{c_i\}_{i=1}^4 \subset \mathbb{R}$ and exponents $\{\alpha, \theta, \gamma,\delta\} \subset \mathbb{R}$, where $L,H,S,D,K$ respectively denote labor, human capital, innovation, investment, and capital stock. %The usual approach to estimating these exponents is using numerical solvers. 
Using RMD as defined in \eqref{eq:RMD} with $\varsigma \to 0^+$, we get
\[
y =  \underset{j \in [4]}{\mathrm{max}}\langle\mb  \theta_j, [\mb  x,1] \rangle,
\]
where
\begin{equation}
[\mb \theta_1, \ldots, \mb \theta_4] \triangleq
\begin{bmatrix}
    \alpha & 0 & 0 & 0 \\
    \theta  & 0 & 0 & 0 \\
    \gamma & 0 & 1 & 0 \\
    \delta & 0 & 1 & 0 \\
    0      & 1 & -1& 0 \\
    0     & 0& 0& 0
\end{bmatrix}\nonumber,
\end{equation}
which is clearly a sparse max-affine model. Another example of such modeling is the \textit{McCallum Gravity Equation} \citep{anderson2003gravity} of party "$1$",  written as
\begin{equation*}
    E_{1} = \sum_{i\neq 1}c_iG_1^\alpha G_i^\beta D_{1\to i}^\delta,
\end{equation*}
where $E_{1}$ denotes the exports of party $1$, $G_i$ is the GDP of the $i$th party, and $D_{1\to i}$ is the distance between parties $1$ and $i$. Furthermore, we can also find applications of the generalized polynomial models in other fields, such as differential equation modeling for fluid mechanics \citep{ranganathan2024modified}, characteristic equations for permeability \citep{siddiqui2008towards}, and power control in cellular systems \citep{chiang2017power}. Finally, consider the generalized rational function that is the division of two models from \eqref{eq:poly_into}. Applying RMD to this rational function leads to a \textit{difference of max-affine model} which we are currently investigating in an independent work. 
 
}

\paragraph{Paper organization}
Section \ref{sec:local} presents Sp-GD and its non-asymptotic theoretical guarantees as a local analysis of sparse max-affine regression.   
Section \ref{Sec:init} describes the initialization scheme and the corresponding non-asymptotic theoretical guarantees. {\color{black} Section \ref{sec:poly} presents the theoretical guarantees for the dequantization error by RMD for generalized sparse polynomials and the convergence guarantee for Sp-GD in this context.}
Section \ref{sec:all_num} presents the numerical results that corroborate the theoretical guarantees for Sp-GD and the initialization scheme.
Section \ref{Conclusion} provides final remarks and future directions.

\paragraph{Notation}
We use lightface characters to denote scalars, lowercase boldface characters to denote column vectors and uppercase boldface to denote matrices. 
We also adopt the symbols for the max and min operators in the lattice theory, i.e. $a \vee b = \max(a,b)$ and $a \wedge b = \min(a,b)$ for $a,b \in \mathbb{R}$. We use multiple matrix norms. 
The Frobenius norm, the spectral norm, and the largest magnitude of entries will be respectively denoted by $\norm{\cdot}_\mathrm{F}$, $\norm{\cdot}$, and $\norm{\cdot}_\infty$.
%For notation of norms, we use $\norm{\cdot}_\mathrm{F}$ for the Frobenius norm, and $\norm{\cdot}_\infty$ denotes the largest magnitude over all entries. 
%We also use $\norm{\cdot}$ interchangeably for matrix operator norm or vector euclidean norm. 
We use a shorthand notation $[d]$ for the set $\{1,2,\dots,d\}$.
For a column vector $\mb x \in \mathbb{R}^d$, its sub-vector with the entries indexed by $\mathcal{S} \subset [d]$ is denoted by $[\mb x]_\mathcal{S}$. 
Similarly, for a matrix $\mb X \in \mathbb R^{d\times d}$, its submatrix with the entries indexed by $\mathcal{S}_1 \times \mathcal{S}_2 \subset [d]\times [d]$ is denoted by $[\mb X]_{\mathcal{S}_1,\mathcal{S}_2}$. Finally, we denote by $C, C_1, C_2, \dots$ universally absolute constants, not necessarily the same at each occurrence.  
% Local Descent, Sparse Max-Affine Parameter Estimation
% Sparse Max-Affine Parameter Initialization
% 
\section{Local Analysis of Sparse Max-Affine Regression} \label{sec:local}
\subsection{Sparse Gradient Descent Algorithm}\label{Sec:spgd}
This section discusses the details of the Sp-GD algorithm. To simplify notation, we rewrite the max-affine model in \eqref{eq:dataset} into a max-linear model 
\begin{equation}
\label{eq:def_max_affine_model_xi}
y = \max_{j \in [k]} \langle \mb  \xi,\mb {\theta}_j^\star \rangle + z , 
\end{equation}
where $\mb {\theta}_j^\star \triangleq [\mb  a_j^\star;b_j^\star] $ and $\mb {\xi} \triangleq [\mb  x;1] $ with the semicolon denoting vertical concatenation. 
Then the target sample $y_i$ is generated by \eqref{eq:def_max_affine_model_xi} from the concatenated covariate sample $\mb \xi_i = [\mb  x_i;1]$ and noise sample $z_i$ for all $i \in [n]$. Let $\mb  \theta^\star \triangleq [\mb  \theta_1^\star ; \ldots ;\mb  \theta_k^\star]$ denote the vertical concatenation of all $k$ hyper-plane coefficient vectors $\{\mb  \theta_j^\star\}_{j=1}^k \subset \mathbb{R}^{d+1}$.   
We consider an estimator of $\mb \theta^\star$ that minimizes the Mean Squared Error (MSE) loss function
\begin{equation}
\label{eq:loss_def_max0}
\ell\left( [\mb  \theta_1 ; \ldots ;\mb  \theta_k] \right) \triangleq
\frac{1}{2n}\sum_{i=1}^{n}\left(y_i- \max_{j \in [k]} \langle\mb  \xi_i,\mb  \theta_j\rangle\right)^2, 
\end{equation}
under the constraint that all $\mb  \theta_j$, for $j \in [k]$, belongs to $\Gamma_s$ defined by 
%\begin{equation}
%\Gamma_s \triangleq \left\{ [\mb{\alpha}_1;\dots;\mb{\alpha}_k] \in \mathbb{R}^{k(d+1)} : \norm{\left(\sum_{j=1}^k [\mb{\alpha}_j]_l^2\right)_{l=1}^d}_0 \leq s \right\},
%\end{equation}
\begin{equation}\label{eq:Gamma_s}
\Gamma_s \triangleq \left\{ \mb{\varphi} \in \mathbb{R}^{(d+1)} : \norm{ [\mb{\varphi}]_{1:d}}_0 \leq s,  \right\},
\end{equation}
where $\norm{\cdot}_0$ counts the number of nonzero entries and $[\mb{\varphi}]_{1:d}$ denotes the sub-vector of $\mb \varphi \in \mathbb{R}^{d+1}$ with the last entry omitted. %In other words, each block of an element in $\Gamma_s$ is $s$-sparse not counting the last entry.
\DontPrintSemicolon
\SetAlFnt{\small}
\begin{algorithm}

\caption{Sparse Gradient Descent (Sp-GD)}

%  \vspace{0.1in}

  \KwIn{dataset $\{\mb x_i, y_i\}_{i=1}^{n} $, sparsity level $s$, model rank $k$, %maximum iteration number $M$, 
  and initial estimate $\mb{\theta}^0$}
%  \vspace{0.1in}
  $t \gets 0$ \\ 
  \While{stop condition is not satisfied}{
  \For{$j \in \{1,\dots,k\}$}{
  \vspace{-0.4cm}
  \begin{flalign}
      \label{eq:pjt}
&\pi_j^{t} = \frac{1}{n} \sum_{i=1}^n \bbone_{\left\{\mb x_i\in \mathcal{C}_j(\mb \theta^t)\right\}}&
\end{flalign} 
\eIf{$\pi_j^t >0$}{
\vspace{-0.5cm}
\begin{flalign} \label{p_gradient} 
 &\nabla_{\mb \theta_j}\ell(\mb \theta^t)
    \gets\frac{1}{n} \sum_{i=1}^n \bbone_{\left\{\mb x_i\in \mathcal{C}_j(\mb \theta^t)\right\}} \left(\langle\mb \xi_i,\mb \theta^t_j\rangle-y_i\right)\mb \xi_i &
\end{flalign}  
\vspace{-0.5cm}
\begin{flalign}  \label{algo:grad_step}
    & ~~\mb{\alpha}_j^{t+1} \gets  \mb{\theta}_j^{t} - \left(\pi_j^t\right)^{-1}  \nabla_{\mb \theta_j} {\ell} \left( \mb{\theta}^{t}\right) &
\end{flalign}
 $~~\mb{\theta}_j^{t+1} \gets \mb{\Psi}_s(\mb{\alpha}_j^{t+1})$ 
}{
$~~\mb{\alpha}_j^{t+1} \gets  \mb{\theta}_j^{t}$
}
   }
    $t \gets t+1$
  }  
  % \vspace{0.1in}
  %$\hat{\mb{\theta}} \gets \mb{\theta}^{t+1}$
\KwOut{final estimate $\hat{\mb \theta} \gets \mb \theta^t$}
   \vspace{0.05in}      
  \label{algo:SparseGD}
\end{algorithm}

% \begin{algorithm}
% \caption{Sparse Gradient Descent (Sp-GD)}

%   \vspace{0.1in}

%   \KwIn{Dataset $\{\mb x_i, y_i\}_{i=1}^{n} $, Sparsity Number $s$, Model Rank  $ K$, Step Size $\mu$, $\gammar$ and initial estimate $\mb{\theta}^0$}
%   \vspace{0.1in}
%   \KwOut{Final Estimate $\hat{\mb \theta}$}
%    \vspace{0.1in}
%   $t \gets 0$ \\
%   \While{$\lVert \mb{\theta}^{t+1} - \mb{\theta}^{t} \rVert_2 > \gamma \lVert  \mb{\theta}^{t}\rVert_2$}{
%     \vspace{0.1in}
%     $\mb{\alpha}^{t+1} \gets  \mb{\theta}^{t} - \mu  \nabla_{\mb \theta} \mb{\ell} \left( \mb{\theta}^{t}\right) $ \\
    
%     \vspace{0.1in}
%     $\mb{\theta}^{t+1} \gets \mb{\Pi}_s(\mb{\alpha}^{t+1})$
%     %     $\mb{\theta}^{t+1} \gets  \mb{\alpha}^{t+1} \odot \omega_s(\mb{\alpha}^{t+1}) $

%     \vspace{0.1in}
%     $t \gets t + 1$
%   }
% $\hat{\mb \theta} \gets \mb \theta^{t}$  
%   %\vspace{0.1in}
%   %$\hat{\mb{\theta}} \gets \mb{\theta}^{t+1}$  
%   %\vspace{0.1in}
  
%   \label{algo:SparseGD}
% \end{algorithm}
Sp-GD is a variant of the projected gradient descent algorithm to pursue the above estimators. In Sp-GD the gradient is substituted by the generalized gradient \citep{hiriart1979new} and the step size varies across blocks adaptively with the iterates. 
We introduce a geometric object to describe the Sp-GD algorithm. 
Let
\begin{equation}
\label{eq:def_calCjstar}
\mathcal{C}_j([\mb  \theta_1 ; \dots ; \mb  \theta_k]) \triangleq \left\{ \mb  x \in \mb b{R}^{d} ~:~ \langle [\mb  x ;\ 1], \mb  \theta_{j} \rangle > \langle [\mb  x ;\ 1], \mb  \theta_{l} \rangle, ~ \forall l \neq j \right\}
\end{equation}
denote an open set in $\mathbb{R}^d$ where the $j$th linear model $\mb x \mapsto \langle [\mb  x ;\ 1], \mb  \theta_{j} \rangle$ achieves the unique maximum for $j \in [k]$. 
The ties in the maximum occur on the set 
\begin{equation*}
\mathcal{V}([\mb  \theta_1 ; \dots ; \mb  \theta_k]) \triangleq \bigcup_{l \neq j} \left\{ \mb  x \in \mb b{R}^{d} ~:~ \langle [\mb  x ;\ 1], \mb  \theta_{j} \rangle = \langle [\mb  x ;\ 1], \mb  \theta_{l} \rangle \right\}.
\end{equation*}
For any $\mb  \theta \in \mathbb{R}^{k(d+1)}$, the open sets $\{\mathcal{C}_j(\mb  \theta)\}_{j=1}^k$ and their boundary in $\mathcal{V}(\mb  \theta)$ constructs a partition of $\mathbb{R}^d$ by satisfying
\begin{align*}
\left( \bigcup\limits_{j=1}^{k} \mathcal{C}_j(\mb  \theta) \right) \cup \mathcal{V}(\mb \theta) = \mathbb{R}^{d}, \quad
\left( \bigcup\limits_{j=1}^{k} \mathcal{C}_j(\mb  \theta) \right) \cap \mathcal{V}(y) = \varnothing, \quad \mathcal{C}_j(\mb  \theta)\cap\mathcal{C}_l(\mb  \theta) = \varnothing, \quad 
\forall l \neq j. 
\end{align*} 
The max-affine function in \eqref{eq:def_max_affine_model_xi} is a special instance of tropical polynomials in the max-plus algebra \citep{maragos2021tropical}. 
In this perspective, the sets $\{\mathcal{C}_j(\mb  \theta)\}_{j=1}^{k}$ and $\mathcal{V}(\mb  \theta)$ are called respectively as tropical open cells and tropical zero set. 

Algorithm~\ref{algo:SparseGD} presents a pseudo code for the Sp-GD algorithm. 
Each iteration of the algorithm starts by updating the empirical probability $\pi_j^t$ that the covariates belong to the open set $\mathcal{C}_j(\mb  \theta^t)$ determined by the previous
iterate $\mb  \theta^t$ for all $j \in [k]$. Note that the evaluation of the indicator function in \eqref{eq:pjt} can be saved and reused in the subsequent step in \eqref{p_gradient}. If $\pi_j^t$ is non-zero, then we update the $j$th block by generalized gradient descent in \eqref{p_gradient} followed by the orthogonal projection to $\Gamma_s$ given by 
\[
\mb{\Psi}_s(\mb{\alpha}) = \mathop{\mathrm{argmin}}_{\tilde{\mb{\alpha}} \in \Gamma_s} ~ \norm{\mb{\alpha} - \tilde{\mb{\alpha}}}_2^2.
\]
Otherwise, the $j$th block remains unchanged from the previous iterate. 
This update rule applies recursively until the algorithm converges by satisfying $\lVert \mb {\theta}^{t+1} - \mb {\theta}^{t} \rVert_2 / \lVert \mb {\theta}^{t}\rVert_2$ less than a given threshold. 
Note that the step size for the gradient descent is adaptively evaluated as the reciprocal of $\pi_j^t$ which is determined solely by the previous iterate $\mb  \theta^t$. In other words, this algorithm does not require tuning for the step size. 
In fact, the step size in \eqref{p_gradient} is always larger than $1$, which is quite different from the typical choices of step size for gradient descent (e.g. a small constant or a diminishing sequence). The next section presents theoretical guarantees of Sp-GD with this proposed step size. 
Furthermore, as shown in Section~\ref{Sec:numspgd}, this step size rule also makes Sp-GD empirically converge fast to the desired estimate.

\subsection{Theoretical Analysis of Sp-GD} \label{Section:Theo}
In this section, we present a local convergence analysis of Sp-GD under a set of covariate distributions determined by the following two properties. 
\begin{assumption}[Sub-Gaussianity]\label{assum:subg}
    The covariate vector $\mb x \in \mathbb{R}^d$ is zero-mean, isotropic, and $\eta$-sub-Gaussian, i.e. there exists $\eta >0$ such that %the following statement holds deterministically 
    \begin{equation}
        \sup\limits_{\mb u \in \mathbb{S}^{d-1}, t \in \mathbb R}\mathbb E \left[\exp (t\langle \mb u , \mb x\rangle)\right]  \leq \exp\left({{\eta^2t^2}/{2}}\right), \nonumber
    \end{equation}
    where $\mathbb{S}^{d-1}$ denotes the unit sphere in $\ell_2^d$, respectively. %, and $\eta > 0 $ is some numerical constant.
\end{assumption}  
\begin{assumption}[Anti-Concentration]\label{assum:anti}
    There exist $\gamma,\zeta >0$ such that the covariate vector $\mb x \in \mathbb{R}^d$ satisfies
    \[
    \sup\limits_{\begin{subarray}{c} \mb u \in \mathbb{S}^{d-1}, \lambda \in \mathbb{R} \end{subarray}} 
    \mathbb{P}\left[ \left( \langle\mb u, \mb x \rangle +\lambda\right)^2\leq \epsilon \right]\leq \left( \gamma \epsilon\right)^\zeta, \quad \forall \epsilon >0.
    \]
\end{assumption}
We also introduce a set $\Theta(s,\kappa,\pi_{\min})$ that collects all $\mb \theta^\star \in \mathbb{R}^{k(d+1)}$ satisfying the following conditions: 
Recall that $\mb \theta^\star$ collects the weight vectors $\{\mb a_j^\star\}_{j=1}^k \subset \mathbb{R}^d$ and the bias terms $\{b_j^\star\}_{j=1}^k \in \mathbb{R}$ of the max-affine function \eqref{eq:classconf} by $\mb \theta^\star = [ [\mb a_1^\star ; b_1^\star] ; \ldots ; [\mb a_k^\star ; b_k^\star]]$.  
First, the weight vectors $\{\mb a_j^\star\}_{j=1}^k$ satisfy the sparsity condition
\[
\norm{\bm a_j^\star}_0 \leq s, \quad \forall j \in [k].
\]
Furthermore, any two distinct weight vectors are separated at least by a minimum discrepancy value $\kappa$, i.e. 
\begin{equation}
\label{eq:defkappa}
\min_{j'\neq j}\|\mb a_j^\star - \mb a_{j'}^\star \|_2 
\geq \kappa.
\end{equation}
Lastly, the probability that $j$th linear model achieves the maximum in \eqref{eq:classconf} with random $\bm x$ should exceed a minimum threshold $\pi_{\min} > 0$ for all $j \in [k]$, i.e.  
\begin{equation}
\label{eq:def_pimin_pimax}
\min_{j\in[k]}\P\left(\mb x\in \mathcal{C}_j(\mb \theta^\star) \right) \geq \pi_{\min}.
\end{equation}
Using the above definitions, we state the main result that shows the local convergence of Sp-GD in the following theorem. 
\begin{theorem}
\label{THM:MAIN}
Suppose that $\{(\mb x_i,z_i)\}_{i=1}^n$ are independent copies of a random vector $(\bm{x},z)$ where $\bm{x} \in \mathbb R^d$ and $z \in \mathbb{R}$ are independent, $\bm{x}$ satisfies Assumptions \ref{assum:subg}--\ref{assum:anti} with parameters $\eta,\gamma,\zeta > 0$, and $z$ is zero-mean sub-Gaussian with variance $\sigma_z^2$. 
%Let $\delta \in (0,1)$ be fixed. 
Then there exist absolute constants $C_1,C_2,R > 0$, for which the following statement holds with probability at least $1-\delta$ for all $\mb \theta^\star \in \Theta(s,\kappa, \pi_{\min})$. 
If the initial estimate $\mb \theta^0$ belongs to a neighborhood of $\mb \theta^\star$ given by 
\begin{equation}
\label{eq:defnbr}
\begin{aligned}
\mathcal{N}(\mb \theta^\star,\kappa \rho):= \Bigg\{ \mb \theta \in \mbb{R}^{k(d+1)} \,:\,  \norm{\mb \theta - \mb \theta^\star}_2 \leq \kappa\rho \Bigg\} 
\end{aligned}
\end{equation}
with 
\begin{equation}
\label{eq:choice_rho1}
\rho:= \left[\frac{R\pi_{\min}^{3/4}}{4k^{2}} \cdot \log^{-1/2}\left(\frac{k^{2}}{R\pi_{\min}^{3/4}}\right)\right] \wedge \frac{1}{4}
\end{equation}
and
\begin{equation}
\label{eq:cond:lem:lwb_gradient}
n \geq C_1 \left[s\log\left(\frac{n\vee d}{s}\right)+\log\left(\frac{k}{\delta}\right) \right]\left(\sigma_z^2\vee 1\vee \eta^4 \right) k^4\pi_{\min}^{-4(1+\zeta^{-1})}, 
\end{equation} 
then the sequence $\left(\mb \theta^t\right)_{t\in\mathbb{N}}$ generated by Sp-GD satisfies 
\begin{align}
\label{eq:errorbound_noise}
\left\|\mb \theta^{t}-\mb \theta^\star\right\|^2_2&\leq \tau^t\left\|\mb \theta^{0}-\mb \theta^\star\right\|^2_2 +C_2\sigma^2_z\left(\frac{sk\log\left(n/s\right)+s\log\left({d}/{s}\right)+\log\left({1}/{\delta}\right)}{n}\right)
\end{align}
for some $\tau \in (0,1)$ determined by $\pi_{\min}$, $k$, $\gamma$, $\zeta$ and $R$. 
\end{theorem}

\begin{remark} \label{Remark:spgd}
A few remarks on the conditions and parameters in Theorem~\ref{THM:MAIN} are in order. 
\begin{enumerate}
    \item The exact form of constant $\tau\in (0,1)$ that determines the speed of convergence in \eqref{eq:errorbound_noise} is provided as Equation \eqref{finaltau} in Section \ref{ProofSketch}.
    \item The parameter $\rho$ in \eqref{eq:choice_rho1} determines the size of the basin of attraction in \eqref{eq:defnbr} together with the minimum discrepancy value $\kappa$. In the ``well-balanced'' case, i.e. $\pi_{\min} = \Omega({1/k})$, we have that $\rho$ becomes $\Omega(k^{-11/4})$. Therefore, the basin of attraction shrinks only by the order $k$ of the max-affine model. %quickly with larger values of $k$ making the initialization requirement stronger for Sp-GD.
    \item In the well-balanced case, if $\mb x \sim \mathrm{Normal}(\bm 0, \bm I_d)$ (thus having $\zeta =1/2$, $\gamma = e $, and $\eta =1$), then the sample complexity requirement on $n$ becomes $\tilde{\mathcal{O}}((\sigma_z^2\vee 1) s k^8 \log d)$. Furthermore, for all $k\geq2$, via an upper bound on $\tau$ obtained by evaluating \eqref{finaltau} with $\rho = 1/4$, we obtain that $\tau < 10^{-3}$. Note that the expression of $\tau$ in \eqref{finaltau} as a function of $\rho$ is monotone increasing. 
\end{enumerate}     
\end{remark}

Theorem~\ref{THM:MAIN} implies local linear convergence of Sp-GD in the noiseless case when the algorithm is properly initialized for the sub-Gaussian covariate model.
%that Sp-GD guarantees linear convergence for the Gaussian covariate model in the noiseless case when the algorithm is properly initialized. 
The sample complexity scales linearly with $s$ significantly improving analogous results without variable selection \citep{ghosh2021max,kim2023maxaffine}. 
Importantly, %it is worth mentioning that 
Sp-GD does not inflate the degree of dependence on the model order $k$ and dataset imbalance parameter $\pi_{\min}$ and maintains the same order $k^4 \pi_{\min}^{-12}$ as plain GD and SGD \citep{kim2023maxaffine}. 
Therefore, Sp-GD outperforms these algorithms 
even in the absence of sparsity 
%regardless of the presence of sparsity 
in the max-affine model. 
%for any choice of $k$ and given any dataset, the application of Sp-GD is always superior to plain first-order methods.

\section{Initialization for Sparse Max-Affine Regression}\label{sec:maxinit}
\subsection{Initialization Algorithms} \label{Sec:init}
The local convergence guarantees of Sp-GD provided in Theorem \ref{THM:MAIN} require a suitable initial estimate. That is the case since the minimization of $\ell(\mb  \theta)$ in \eqref{eq:loss_def_max0} under the sparsity constraint by \eqref{eq:Gamma_s} is non-convex, there exist multiple local minimizers, which hinders Sp-GD from converging to the global minimizer from an arbitrary initialization. 

To bypass this issue, we compute an initial estimate in the basin of attraction of the ground truth so that the subsequent Sp-GD converges to the desired estimate. 
One may apply the initialization scheme for max-affine regression by \citep{ghosh2021max} while ignoring the sparsity constraint. 
However, their theoretical guarantees provide the desired accuracy when the sufficient number of observations scales at least linearly in the total number of variables $d$. 
Therefore, using this initialization scheme with Sp-GD yields a requirement on the sample complexity that is no longer dominated by the number of active variables $s$. 
To retain the gain due to Theorem~\ref{THM:MAIN}, we propose an initialization scheme modified from the spectral initialization for max-affine regression \citep{ghosh2021max} that provides an error rate depending on $s$ instead of $d$. 
To fulfill this objective, the initialization scheme additionally requires that the $s$-sparse weight vectors $\{\mb  a_j^\star\}_{j=1}^k$ are jointly supported on a set with cardinality at most $s$. 
In other words, the ground-truth parameter vector $\mb  \theta^\star = [ [\mb  a_1^\star ; b_1^\star] ; \ldots ; [\mb  a_k^\star ; b_k^\star]]$ belongs to the set defined by 
\begin{equation}\label{eq:Gamma_s_joint}
\Gamma_{\text{$s$-row-sparse}} \triangleq \left\{ [\mb{\alpha}_1;\dots;\mb{\alpha}_k] \in \mathbb{R}^{k(d+1)} : \norm{ \left(\sum_{j=1}^k 
 [\mb{\alpha}_j]_l^2 \right)_{l=1}^d}_0 \leq s \right\},
\end{equation}
which is the set of all possible $\mb  \theta^\star$ with jointly sparse weight vectors. 
The parameter initialization is a two-step process: (i) estimate the span of $\{\mb  a_j^\star\}_{j=1}^k$, (ii) then estimate individual weight vectors $\{[\mb  a_j^\star; b_j^\star]\}_{j=1}^k$ from the estimated subspace.

\DontPrintSemicolon
\SetAlFnt{\small}
\begin{algorithm}

\caption{Sparse Spectral Method for $k \leq s$}
    
%  \vspace{0.1in}

  \KwIn{dataset $\{\mb x_i, y_i\}_{i=1}^{n} $, sparsity level $s$, model order $k$, regularization parameter $\lambda$}
  \vspace{0.05in}
% \While{stop condition is not satisfied}{
    $\mb{\hat m}_1 \gets \frac{1}{n}\sum_{i=1}^n y_i \mb{x}_i,$
    $\quad\mb{\hat M}_2 \gets \frac{1}{n} \sum_{i=1}^n y_i \left(\mb x_i \mb x_i^\T -\mb I_d\right), $
    $\quad \mb{\hat M} \gets \mb{\hat m}_1\mb {\hat m}_1^\T + \hat{\mb M}_2$\;
    \vspace{0.05in}
    $\mb{\hat P} \gets \displaystyle \mathop{\mathrm{argmax}}_{\tilde{\mb{ P}} \in \mathcal{F}^k_d} ~\mathrm{tr}\left(\mb{\hat M}^\T \tilde{\mb{ P}}\right) -\lambda \sum\limits_{i,j} \left|[\tilde{\mb{ P}}]_{i,j}\right|$ \hfill \text{$\triangleright$ where $\mathcal{F}^k_d$ is defined in \eqref{eq:def_fantope}} \;
    \vspace{0.05in}
    $\mathcal{S}\gets \left\{ \text{indices of $s$-largest diagonal entries of $\hat{\mb P}$} \right\}, \quad \mb{\hat V} \gets \displaystyle \mathop{\argmin}_{\tilde{\mb V}^\T \tilde{\mb V} = \mb I_k} \norm{[\mb{\hat P}]_{\mathcal{S},\mathcal{S}}- \tilde{\mb V} \tilde{\mb V}^\T }_\mathrm{F}$\;
  \KwOut{$k$-principle subspace estimate $\hat{\mb V}  \in \mathbb{R}^{s\times k}$, support set $\mathcal{S}$}
   \vspace{0.1in}    
  \label{algo:sPCA}
\end{algorithm}

First, the subspace estimation algorithm is presented in Algorithm \ref{algo:sPCA}. This algorithm modifies upon the moment method for parameter estimation originally developed for mixtures of linear models \citep{chaganty2013spectral,zhang2014spectral,yi2016solving,sedghi2016provable}. 
In Algorithm \ref{algo:sPCA}, $\hat{\mb  m}_1$ and $\hat{\mb  M}_2$ are respectively the first and second central moments weighted by the target values.  
In the non-sparse case, applying PCA to $\hat{\mb  M}$ and taking the first $k$-dominant eigenvectors yields a basis estimate of the subspace spanned by the ground-truth weight vectors $\{\mb  a_j^\star\}_{j=1}^k$.
In our case, the joint sparsity of the dominant eigenvectors can be utilized to obtain a more accurate estimate of the subspace via sparse PCA (sPCA) \citep{zou2006sparse}. 
Several convex relaxations were proposed to solve the sPCA problem \citep{d2008optimal,amini2009high,zhang2012sparse,vu2013fantope,dey2018convex,li2020exact}. 
Algorithm~\ref{algo:sPCA} uses the convex relaxation of sPCA by \citep{vu2013fantope} that computes the $k$-dominant eigenvectors under the joint sparsity constraint. 
They formulated the estimation of sparse principal components as the following semidefinite program: 
\begin{equation}\label{eq:sPCA}
    \max_{\mb  P \in \mathcal{F}^k_d} \mathrm{tr}\left(\hat{\mb  M}^\top\mb  P\right) -\lambda \sum\limits_{i,j} \left|[{\mb { P}}]_{i,j}\right|
\end{equation}
for some positive constant $\lambda>0$, where $\mathcal{F}^k_d$ denotes the Fantope defined by
\begin{equation}\label{eq:def_fantope}
    \mathcal{F}^k_d \triangleq \left\{ \mb  V \mb  V^\T: \mb  0_d \preceq \mb  V \mb  V^\T \preceq \mb  I_d,~ \mathrm{tr}\left(\mb  V^\T \mb  V\right)=k \right\}.
\end{equation}
Indeed, the Fantope $\mathcal{F}^k_d$ is obtained as the convex hull of all rank $k$-projection matrices \citep{overton1992sum}.
On one hand, the first term in the maximization in \eqref{eq:sPCA} measures the similarity between an element in the Fantope and the empirical moment matrix $\hat{\mb  M}$. 
On the other hand, the second term in \eqref{eq:sPCA} encourages sparsity by penalizing with the $\ell_1$ norm of all entries of the matrix. 
This relaxed version of sPCA is applied in the second step of Algorithm \ref{algo:sPCA}.
The solution to this is calculated using the Alternating Direction Method of Multipliers (ADMM) algorithm presented in \citep[Algorithm 1]{vu2013fantope}. 
Since the result $\hat{\mb  P}$ is not necessarily a valid rank-$k$ projection matrix, we need post-processing in the following two steps. 
First, the algorithm recovers the support as $\mathcal{S}$ by thresholding the diagonal entries of $\hat{\mb  P}$. 
The final step in Algorithm \ref{algo:sPCA} is to find the optimal projection of $[\mb {\hat P}]_{\mathcal{S},\mathcal{S}}$ onto the set of all rank-$k$ projection matrices, which can be obtained by the $k$-dominant eigenvectors of $[\hat{\mb  P}]_{\mathcal{S},\mathcal{S}}$. Then $\hat{\mb  V} \in \mathbb{R}^{s \times k}$ will denote the Cholesky factor of the estimated rank-$k$ projection matrix. 

\DontPrintSemicolon
\SetAlFnt{\small}
\begin{algorithm}

\caption{Discrete Search over Estimated Subspace \citep{ghosh2021max}}
    
%  \vspace{0.1in}

  \KwIn{dataset $\{\mb x_i, y_i\}_{i=1}^{n} $, model order $k$, subspace basis $\mb{\hat V}$, separation $r\in (0,1)$, and  support $\mathcal{S}$}
  \vspace{0.05in}
% \While{stop condition is not satisfied}{
    $\mathcal{N} \gets r$-covering of $B_2^{k+1}$
   \;
    \vspace{0.05in}
    $\displaystyle \left( \{[\mb w_j^\sharp;b_j^\sharp]\}_{j=1}^k, c^\sharp \right) \gets \mathop{\mathrm{argmin}}_{\{[\mb w_j;b_j]\}_{j=1}^k \in \mathcal{N}, c\geq0}
    \frac{1}{n}  \sum\limits_{i=1}^{n}\left(y_i - c\cdot \underset{j \in [k]}{\max}\langle [\mb x_i]_{\mathcal{S}}^\mathrm{T}\hat{\mb V} w_j + b_j\rangle\right)^2 $\;
    \For{$j \in \{1,\dots,k\}$}{
    $ [\mb \theta_j^0]_{\mathcal{S}} \gets c^\sharp[\hat{\mb V} \mb w_j^\sharp ; b_j^\sharp], \quad [\mb \theta_j^0]_{[d] \setminus \mathcal{S}} \gets \bm 0 \in \mathbb{R}^{d-s}$
    }
    \vspace{0.05in}
    
  \KwOut{Initial model parameter estimate $\{\mb \theta_j^0\}_{j=1}^k$.}
   \vspace{0.1in}    
  \label{algo:ksearch}
\end{algorithm}

Algorithm~\ref{algo:sPCA} identifies the joint support of $\{\mb a_j^\star\}_{j=1}^k$ and estimates the subspace spanned by these vectors. 
This information is not sufficient to estimate the individual parameter vectors $\{[\mb a_j^\star,b_j^\star]\}_{j=1}^k$. 
To approximate the parameter vectors up to a global scaling, the authors in \citep{ghosh2021max} proposed a discrete search over subsets of an $r$-covering $\mathcal{N}$ of the unit ball that satisfies $\min_{\mb x\in \mathcal{N}}\|\mb x -\mb u \|\leq r$ for all $\mb u \in B_2^{k+1}$ 
and $B_2^{k+1} \subset \bigcup_{\mb w \in \mathcal{N}} \left(\mb  w + r B_2^{k+1}
\right)$. 
The cost of constructing an $r$-covering and searching over it grows exponentially in the dimension. 
Due to the dimensionality reduction by Algorithm~\ref{algo:sPCA}, the search dimension is the model order $k$, which is often much smaller than the ambient dimension $d$. 
If this is the case, then the discrete search is computationally feasible. 
This method also applies to the initialization of sparse max-affine regression. 
To make the manuscript self-contained, we summarize the discrete search algorithm by \citep{ghosh2021max} in Algorithm~\ref{algo:ksearch} using our notation. 

\begin{comment}
    This algorithm samples model parameter candidates from $\mathcal N$, the $r$-covering of the unit ball $ {B}_2^{k+1}$. Then an exhaustive search is performed over all possible model parameters from $\mathcal N$ and appropriate scaling factor $c\geq0$. The estimated model parameters are those that minimize the mean squared loss. Note that the search over the optimal value of the scaling $c$ is available as the closed-form solution  
\begin{equation}
 \hat{c} = \argmin\limits_{c \geq 0 }\norm{ \underbrace{\begin{bmatrix}
           y_{1} \\
           y_{2} \\
           \vdots \\
           y_{n}
         \end{bmatrix}}_{\mb  y} - c 
         \underbrace{\begin{bmatrix}
           \max\limits_{j \in [k]} \langle \mb  \xi_1 ,\mb  V \nu_j\rangle \\
           \max\limits_{j \in [k]} \langle \mb  \xi_2 ,\mb  V \nu_j\rangle \\
           \vdots \\
           \max\limits_{j \in [k]} \langle \mb  \xi_n ,\mb  V \nu_j\rangle
         \end{bmatrix}}_{\hat{\mb  y}}
         } = {\max \left\{\frac{\langle \mb  y, \hat{\mb  y}\rangle}{\norm{\hat{\mb  y}}^2},0 \right\}},
\end{equation}
for every fixed $\{\mb  \nu_j\}_{j=1}^k \in \mathcal{N}$. 
The performance of this algorithm depends on the separation $r$ of the covering points, and on 
\end{comment}

\subsection{Theoretical Analysis of Initialization}
We present theoretical guarantees for Algorithm \ref{algo:sPCA} in the following theorem. 
\begin{theorem}\label{Theo:Projection}
     Suppose that $k\leq s\leq d$. 
     Let $\mb  x \sim \mathrm{Normal}(\mb  0,\mb  I_d)$, $y$ be defined from $\mb  x$ according to \eqref{eq:def_max_affine_model_xi}, and $\mb  P \in \mathbb R^{d\times d}$ be the projection operator onto the span of $\{\mb  a^\star_j\}_{j=1}^k$ that are jointly supported on $\mathcal{S}^\star$ with $|\mathcal{S}^\star|=s$. 
     Then it holds with probability $1-n^{-11}$ that the estimates by Algorithm \ref{algo:sPCA} satisfies $\mathcal{S} = \mathcal{S}^\star$ and 
     \begin{equation}\label{eq:projectionerror}
        \norm{\hat{\mb  V} \hat{\mb  V}^\T - [\mb  P]_{\mathcal{S}^\star,\mathcal{S}^\star}}_\mathrm{F} \leq C s \cdot \frac{  \varsigma^2+ \sigma_z^2}{\delta_{\mathrm{gap}}} \cdot\left(\frac{\log^2(nd)}{n} \vee \frac{\log(nd)}{\sqrt{n}}\right)
     \end{equation}
     provided  
     \begin{equation}\label{eq:nprojectionerror}
         n\geq C s^2 \left(\frac{\varsigma^2\vee\sigma_z^2}{\delta_{\mathrm{gap}}} \vee \frac{\varsigma^4\vee\sigma_z^4}{\delta^2_{\mathrm{gap}}}\right)\left( \min_{j \in \mathcal{S}^\star} [\mb  P]_{jj} \right)^{-2} \log^2(nd)
     \end{equation}
     for $\delta_\mathrm{gap} > 0$, independent of $d$, where
     \begin{equation} \label{eq:varsigma}
    \varsigma \triangleq \max_{j \in [k]}\left(\norm{\mb  a^\star_j}_1+|b_j^\star|\right).
\end{equation}
\end{theorem}
The work in \citep{ghosh2021max} showed that their spectral initialization provides an $\epsilon$-accurate subspace estimation with $\tilde{\mathcal{O}}(\epsilon^{-2}d)$ samples (when the ground-truth and hence the model order $k$ are fixed).
In the sparse case of max-affine regression, Theorem \ref{Theo:Projection} reduces the dependence of the sample complexity on $d$ from linear to logarithmic. 
Note that the spectral gap $\delta_{\mathrm{gap}}$ in Theorem \ref{Theo:Projection} is also independent of $d$ similar to \citep[Theorem~2]{ghosh2021max}.

\begin{proof}[Proof of Theorem~\ref{Theo:Projection}]
Consider the population-level version of the empirical moment matrix $\hat{\mb M}$ defined as
\begin{align}
\label{init_lemma_def}
\mb  M = \mb  m_1 \mb  m_1^\T + \mb  M_2, \quad \text{where} \quad 
\mb  m_1 = \mathbb{E}[y \mb  x ], \quad \text{and} \quad \mb  M_2 = \mathbb E\left[y(\mb  x \mb  x^\T - \mb  I_d) \right]. 
\end{align}
We will use the following known results about $\mb  M$: 
i) The column space of $\mb  M$ coincides with the $k$-dimensional subspace spanned by the ground-truth weight vectors $\{\mb  a_j^\star\}_{j=1}^k$ \citep[Lemma~3]{ghosh2021max}; 
ii) There exists $\delta_\mathrm{gap}>0$, independent of $d$, such that the smallest nonzero eigenvalue of $\mb  M$ is bounded from below by $\delta_\mathrm{gap}$ \citep[Lemma 7]{ghosh2021max}. 
Then \citep[Theorem 3.1]{vu2013fantope} provides a perturbation bound given by
\begin{equation} \label{eq:vincebound}
    \|\hat{\mb P}-\mb P\|_{\mathrm{F}}
    \leq 
    \frac{4s}{\delta_{\mathrm{gap}}}  \norm{\mb {\hat M} - \mb {  M}}_\infty.
   %{\leq} C s\frac{  \varsigma^2+ \sigma_z^2}{\delta_{\mathrm{gap}}} \cdot\left(\frac{\log^2(nd)}{n} \vee \frac{\log(nd)}{\sqrt{n}}\right),
\end{equation}
Next, we derive an upper bound on $\|\hat{\mb  M} - \mb  M\|_\infty$. 
The following lemma, whose proof is deferred to Appendix \ref{App:Proofs}, provides the concentration of $\hat{\mb  M}$ around $\mb  M$.
\begin{lemma} \label{lem:Mtotal}
  Instate the assumptions in Theorem \ref{Theo:Projection}.We have that
    \begin{align}
   \mathbb P \left( \norm{\mb {\hat M} - \mb {  M}}_\infty\geq  C(  \varsigma^2+ \sigma_z^2)\left(\frac{\log^2(nd)}{n} \vee \frac{\log(nd)}{\sqrt{n}}\right)\right) \leq n^{-11}.
\end{align}
\end{lemma}
By plugging Lemma \ref{lem:Mtotal} to \eqref{eq:vincebound}, we obtain that it holds with probability $1-n^{-11}$ that
\begin{equation} \label{eq:vincebound2}
    \|\hat{\mb P}-\mb P\|_{\mathrm{F}}
   \leq C s\frac{  \varsigma^2+ \sigma_z^2}{\delta_{\mathrm{gap}}} \cdot\left(\frac{\log^2(nd)}{n} \vee \frac{\log(nd)}{\sqrt{n}}\right).
\end{equation}
We further proceed with the remainder of the proof under the event that \eqref{eq:vincebound2} holds. 
The first assertion $\mathcal{S} = \mathcal{S}^\star$ follows from \citep[Theorem 3.2]{vu2013fantope} if 
\begin{equation}\label{eq:support_req}
    \norm{\mb {\hat P} - \mb  P}_{\mathrm{F}}\leq \frac{1}{2} \cdot \min\limits_{\substack{j \in \mathcal{S}^\star}} [\mb  P]_{jj},
\end{equation}
which is satisfied by \eqref{eq:nprojectionerror} and \eqref{eq:vincebound2}. 
Next, by the triangle inequality and the optimality of $\hat{\mb  V}$, we have
\begin{align} \label{eq:projection_error}
\norm{\hat{\mb  V} \hat{\mb  V}^\T -[\mb  P]_{\mathcal{S}^\star,\mathcal{S}^\star}}_\mathrm{F} &\leq \norm{\hat{\mb  V} \hat{\mb  V}^\T - [\mb {\hat P}]_{\mathcal{S},\mathcal{S}}}_\mathrm{F} + \norm{[\mb {\hat P}]_{\mathcal{S},\mathcal{S}} - [\mb  P]_{\mathcal{S}^\star,\mathcal{S}^\star}}_{\mathrm{F}} \nonumber \\
&\leq 2 \, \norm{[\mb {\hat P}]_{\mathcal{S},\mathcal{S}} - [\mb  P]_{\mathcal{S}^\star,\mathcal{S}^\star}}_{\mathrm{F}} 
= 2\|\hat{\mb  P}-\mb  P \|_{\mathrm{F}}, 
\end{align}
where the last identity holds since we have shown $\mathcal{S} = \mathcal{S}^\star$.
Combining \eqref{eq:vincebound2} and \eqref{eq:projection_error} yields that the second assertion in \eqref{eq:projectionerror} also follows from \eqref{eq:nprojectionerror}, which concludes the proof.
\end{proof}

Once the support $\mathcal{S}^\star$ is exactly recovered by Algorithm~\ref{algo:sPCA}, the estimation accuracy of Algorithm \ref{algo:ksearch} is provided by \citep[Theorem 3]{ghosh2021max} with $d$ substituted by $s$. 
%The replacement is valid in this case \hl{because we assume that} \eqref{eq:support_req} holds, i.e. the ground-truth support is exactly recovered. . 
We provide the statement of this result in our notation for completeness. 
\begin{theorem}[\hspace{-0.005cm}{A paraphrase of \citep[Theorem 3]{ghosh2021max}}]\label{theo:ksearch}
    Instate the assumptions in Theorem~\ref{Theo:Projection}. Let $R_{\max} \triangleq \max_{j \in [k]} \|\mb  \theta^\star_j\|$. Then it holds with probability at least $1-n^{-11}$ that the initial parameter estimate $\{\mb  \theta^0_j\}_{j=1}^k$ by Algorithm \ref{algo:ksearch} with $\mathcal{S} = \mathcal{S}^\star$ satisfies     \begin{equation}\label{eq:ghosh1st}
        \min_{\substack{ \pi \in \mathrm{perm}([k])}}  \sum_{j=1}^k\norm{\mb  \theta^\star_j - \mb  \theta^0_{\pi(j)} }^2 \leq \frac{k^4}{\pi_{\min}^3} \left\{ R_{\max}^2 \left( r^2 +\norm{\hat{\mb  V} \hat{\mb  V}^\T - [\mb  P]_{\mathcal{S}^\star,\mathcal{S}^\star}}_\mathrm{F}^2 \right) + \frac{\sigma^2_z \log (1+ {1}/{r}) }{n}\right\} 
    \end{equation}
    provided that 
    \begin{gather}\label{eq:nghosh}
         n\geq C\left(\frac{k}{\pi_{\min}}\right)^4 \left\{\left[s\log\left({nk}\right)\log^2\left(\frac{k}{\pi_{\min}}\right)\right]\vee \left[\left(\frac{\sigma_z}{\Delta}\right)^2 \log\left(1+\frac{1}{r}\right)\right] \right\}, 
         \end{gather}
         \begin{gather}\label{eq:ghosh_proj_req}
         \norm{\hat{\mb  V}\hat{\mb  V}^\T -[\mb  P]_{\mathcal{S}^\star,\mathcal{S}^\star}}^2_{\mathrm{F}} \leq \frac{\Delta^4\pi_{\min}^3}{64R_{\max}^2k^4},
         \end{gather}
         \begin{gather}\label{eq:ghosh_r_req}
         r^2\leq \frac{\Delta^4 \pi_{\min}^5}{64R_{\max}^2k^6\log(k\pi^{-1}_{\min})}.
    \end{gather}
\end{theorem}
\begin{remark}
    The estimation error in Theorem \ref{theo:ksearch} is stated using the minimum distance between the ground-truth parameters $\{\mb  \theta^\star_j\}_{j=1}^k$ and all permutations of the estimates $\{\mb  \theta^0_j\}_{j=1}^k$. On the other hand, \citep[Theorem 3]{ghosh2021max} stated the estimation error as the minimum distance up to both permutation and scaling ambiguities. However, careful examination of their proof shows that their error bound applies to a particular scaling with the minimizer $c^\sharp$ provided by Algorithm \ref{algo:ksearch}. Therefore,
    \citep[Theorem 3]{ghosh2021max} also implies \eqref{eq:ghosh1st}. %It is worth mentioning that the partition in \eqref{eq:def_calCjstar} is invariant under positive scaling of all parameters. 
    %Therefore, \citep{ghosh2021max} were able to derive the local convergence of the alternating minimization algorithm in the presence of the scaling ambiguity in the initialization.
    %used, for local convergence, an alternating minimization algorithm that relies on the partitions defined in \eqref{eq:def_calCjstar}. The fact that the partitions are invariant under any constant scaling justifies their choice of the estimation error metric in their theorem statement. 
\end{remark}
\begin{comment}
    To the author's best knowledge, Algorithm \ref{algo:ksearch} inherited from \citep{ghosh2021max} is the only known method that calculates initial parameter estimates in finite time. However, the computational cost increases exponentially with the number of max-affine models $k$ through the cardinality of the $r
$-covering. This restricts the practicality of computing an initialization for max-affine models with large $k$.
\end{comment}
Finally, the performance of the entire initialization scheme using Algorithms \ref{algo:sPCA} and \ref{algo:ksearch} in succession is presented in the following theorem.
\begin{theorem}\label{theo:init}
    Instate the assumptions of Theorems \ref{Theo:Projection} and \ref{theo:ksearch}.  
    Let $\epsilon \in (0,1)$. Then it holds with probability at least $1- n^{-11}$ that applying Algorithms \ref{algo:sPCA} and \ref{algo:ksearch} in succession yields an initial estimate satisfying 
    %\begin{equation}
     %   \min_{\substack{ \pi \in \mathrm{perm}([k])}}  \sum_{j=1}^k\norm{\mb  \theta^\star_j - \mb  \theta^0_{\pi(j)} }^2 \leq \frac{k^4}{\pi_{\min}^3} \left\{ R_{\max}^2 \left( r^2 + s^2\frac{  \varsigma^4+ \sigma_z^4}{\delta_{\mathrm{gap}}^2} \cdot\left(\frac{\log^4(nd)}{n^2} \vee \frac{\log^2(nd)}{n}\right) \right) + \frac{\sigma^2_z \log (1+ {1}/{r}) }{n}\right\} 
    %\end{equation}
    \begin{equation} \label{eq:init_estimate}
        \min_{\substack{ \pi \in \mathrm{perm}([k])}}  \left( \sum_{j=1}^k\norm{\mb  \theta^\star_j - \mb  \theta^0_{\pi(j)} }^2 \right)^{1/2} \leq \epsilon
    \end{equation}
    provided that 
    \begin{align} \label{eq:ninit}
        n &\geq C s^2 \epsilon^{-2}\left(\frac{\varsigma^2\vee\sigma_z^2}{\delta_{\mathrm{gap}}} \vee \frac{\varsigma^4\vee\sigma_z^4}{\delta^2_{\mathrm{gap}}}\right) \left(\frac{R_{\max} k^2}{\Delta^2\pi_{\min}^2}\vee \frac{1}{\min\limits_{j\in[s]}[\mb  P]_{jj}}\right)^2 \log^4(nd) \log \left( 1+\frac{1}{r}\right)
    \end{align}
    and
    \begin{gather}\label{eq:ghosh_r_req_init}
         r^2\leq \frac{\Delta^4 \pi_{\min}^5 \epsilon^{2}}{32R_{\max}^2k^6\log(k\pi_{\min}^{-1})}.
    \end{gather}
\end{theorem}
\begin{proof}
The sample complexity condition \eqref{eq:ninit} implies \eqref{eq:nprojectionerror} and invokes Theorem \ref{Theo:Projection} to satisfy $\mathcal{S} = \mathcal{S}^\star$ and \eqref{eq:ghosh_proj_req}. 
Furthermore, \eqref{eq:ninit} and \eqref{eq:ghosh_r_req_init} respectively imply \eqref{eq:nghosh} and \eqref{eq:ghosh_r_req}. 
Therefore, Theorem~\ref{theo:ksearch} is invoked to provide \eqref{eq:ghosh1st}. 
Finally, \eqref{eq:ninit} implies the upper bound in \eqref{eq:ghosh1st} is less than $\epsilon$, which is the assertion in \eqref{eq:init_estimate}
\end{proof}

Theorem \ref{theo:init} reduces the sample complexity of the spectral initialization for max-affine regression by the method in \citep{ghosh2021max} when the weight vectors satisfy the joint $s$-sparsity structure. 
Specifically, the linear dependence on the ambient dimension $d$ drops to polynomial dependence on $s$, significantly reducing the sample complexity when $s\ll d$. 
%Next, one important note is that the actual 
The order of the polynomial in $s$ depends on the geometry of the model parameters $\{\mb  \theta^\star_j\}_{j=1}^k$. 
For example, if $\mb  1 \in \mathbb R^s$ is in the span of $\{\mb  \theta^\star_j\}_{j=1}^k$, then it can be shown that $\min_{j\in[s]}[\mb  P]_{jj}\geq 1/s$, hence the sample complexity in Theorem \ref{theo:init} becomes $\mathcal{O}(s^4)$.

{\color{black}
\section{Learning Sparse Generalized Polynomials Via Max-Affine Regression} \label{sec:poly}
In this section, we present theoretical guarantees for the dequantization error by Real Maslov Dequantization (RMD) defined in \eqref{eq:RMD}, and the non-asymptotic convergence guarantee for Sp-GD under the bounded additive noise model. 
For convenience, we restate the RMD transformation here as
\begin{equation}
    \label{eq:RMD2}
    y = \mathrm{Re}\{\varsigma \log w\}, \quad x_l = \varsigma \log u_l, \quad \forall l \in [d].
\end{equation}
In what follows, collect the exponents of the generalized polynomial model defined in \eqref{eq:poly_into} as $\mb \alpha^\star_j = [\alpha^\star_{j,1}; \ldots;\alpha^\star_{j,d}]$  for every $j\in [k]$. 
The dequantization error by RMD with temperature parameter $\varsigma>0$ is written as
\begin{equation}
    \label{eq:quantmax}
    z_\varsigma:= \max_{j \in [k]} \langle [ x_1; \ldots;x_d;1], [\mb \alpha^\star_{j};\log|c^\star_j|]\rangle - y.
\end{equation}
The following theorem provides a uniform bound for the dequantization error by RMD.
\begin{theorem} \label{Theo:maslov_approx}
    Given the generalized sparse polynomial relation $w =g (u_1, \ldots, u_d)$ defined in \eqref{eq:poly_into}, let $y = \mathrm{Re}\{\varsigma \log w\}$ and $\mb x = \varsigma\log \mb u$ be the transformed variables with $\varsigma> 0$. If $\mb x$ satisfies Assumption \ref{assum:anti} and $\varsigma>0$ is sufficiently small, then it holds with probability at least $1- (\gamma \vartheta)^\zeta$ that 
    \begin{equation}
        \label{eq:quant_Error_z}
        |z_\varsigma|\leq 2\varsigma (k-1) \exp\left( \frac{-\sqrt{\vartheta}\Delta}{\varsigma}\right),
    \end{equation}
    where $\Delta>0$ is the minimum separation parameter defined in \eqref{eq:defkappa}. 
\end{theorem}
The proof of this theorem is deferred to Appendix \ref{maslovproofapprox}. Theorem \ref{Theo:maslov_approx} shows that RMD has a dequantization error decaying exponentially in $\Delta/\varsigma$. 
This implies that for the dequantization error to be bounded as $|z_\varsigma|\leq \epsilon$, then we need to select $1/\varsigma \geq \mathcal{O}(\log\frac{k}{\epsilon})$.  
Therefore, via RMD, we can learn the real exponents $\alpha^\star_{j,l}$'s through sparse max-affine regression. By plugging in the exponent estimates into \eqref{eq:poly_into}, the coefficients $\{c_j^\star\}_{j=1}^k$ can be easily approximated by linear least squares. We note that the dequantization error bound in Theorem \ref{Theo:maslov_approx} has not been shown even in the original Maslov dequantization for the simpler posynomial model by \citep{maragos2021tropical}.

Building on the non-asymptotic theory in Theorem \ref{THM:MAIN}, we can show the analysis of estimating the exponents in \eqref{eq:poly_into}. This requires modifying the proof of Theorem \ref{THM:MAIN} to handle the bounded noise model of $z_\varsigma$ that is also dependent on the covariate $\mb x$.

\begin{corollary}
\label{theo:bounded_SPGD}
Suppose that $\{\mb u_i\}_{i=1}^n$ are independent copies of $\mb u$ is distributed such that $\mb x = \varsigma \log \mb u$ satisfies Assumptions \ref{assum:subg}--\ref{assum:anti}. 
Let the targets $\{w_i\}_{i=1}^n$ be generated according to \eqref{eq:poly_into}. 
Then there exists a sufficiently small $\varsigma>0$ for which the following statement holds with probability at least $1-\delta- (\gamma \vartheta)^\varsigma$. The final estimate generated by Sp-GD satisfies 
\begin{align} \label{eq:Spgd_bounded_tau}
\sum_{j=1}^k\left[\norm{\mb \alpha_{j}^t - \mb \alpha_{j}^\star }_2^2 + \log^2(c_j^t/c_j^\star) \right]\leq 
         Ck^3\pi_{\min}^{-2}\varsigma^2\exp\left(\frac{-2 \sqrt{\vartheta}\Delta}{\varsigma} \right).   
\end{align} 

\end{corollary}  
The proof of this theorem is deferred to Appendix \ref{app:maslov_spgd}.
\begin{remark}
    A few remarks on the statement of Corollary \ref{theo:bounded_SPGD}.
    \begin{itemize}
        \item Notice that for $\mb x$ to satisfy Assumption \ref{assum:subg}, i.e. $\mb x$ is subGaussian, $\mb u$ needs to be log subGaussian. For example, if $\mb x$ is normal, $\mb u$ needs to follow the well-known log-normal distribution.
        \item For the sake of clarity, the result of Corollary \ref{theo:bounded_SPGD} is stated where the distortion is only due to RMD. If we assume an additive sub-Gaussian noise model after the RMD transformation, then we can expect an additional error term \eqref{eq:Spgd_bounded_tau} scaling as $\mathcal{O}(\sigma_z^2 s \log(d)n^{-1})$ similar to the statement of Theorem \ref{THM:MAIN}. On the other hand, assuming additive sub-Gaussian noise in the generalized polynomial domain requires further investigation. 
    \end{itemize}
\end{remark}
In summary, Theorem \ref{Theo:maslov_approx} and Corollary \ref{theo:bounded_SPGD} prove that generalized sparse polynomial regression can be done via sparse max-affine regression via the Real \textit{Maslov dequantization}. We note that such results are the first of their kind. In other words, we are the first to show the approximation (quantization) error for a non-zero choice of $\varsigma>0$ by Theorem \ref{Theo:maslov_approx} and a convergence guarantee by Theorem \ref{theo:bounded_SPGD}. The original work by \citep{maragos2021tropical} only showed the connection between max-affine functions and \eqref{eq:poly_into} in the special case of the posynomial model. Still, this connection is established only at the limit where $\varsigma\to 0$. In other words, the approximation error for a practical choice of $\varsigma$ was not previously shown.
}
\section{Numerical Results} \label{sec:all_num}
\subsection{Phase Transitions of Sp-GD} \label{Sec:numspgd}
This section presents numerical results of the Sp-GD algorithm that corroborate the theoretical guarantees presented in Section~\ref{Section:Theo}. 
In the simulation, we initialized Sp-GD in two steps. First, we estimate the parameter subspace by Algorithm \ref{algo:sPCA}. Second, we apply a practical alternative to Algorithm~\ref{algo:ksearch} which randomly samples from the estimated subspace and chooses the best one that produces the smallest fit error after $10$ iterations of Sp-GD.  
We adopt this heuristic instead of Algorithm~\ref{algo:ksearch} for the following reasons. 
%First, the cost of constructing an $r$-covering and discrete search in Algorithm \ref{algo:ksearch} grows exponentially in $s$.
The first step of Algorithm~\ref{algo:ksearch} creates an $r$-covering of the unit $\ell_2$-ball $B^{k+1}$, where its cardinality scales as $\mathcal{O}(r^{-k})$ \citep[Corollary 4.2.12]{vershynin2018high}. 
Therefore, the exhaustive search in the second step over all elements in the $r$-covering becomes impractical as the parameter $r$ decreases. 
However, the accuracy of Algorithm~\ref{algo:ksearch} crucially depends on $r$.  
The estimation performance is evaluated via the median of the relative error between the true model coefficients $\mb {\theta}^\star \triangleq (\mb {\theta_j}^\star)_{j=1}^{k}$ and the estimated coefficients $\hat{\mb {\theta}} \triangleq (\hat{\mb {\theta}}_j)_{j=1}^{k}$ over $50$ Monte Carlo simulations.
The relative error is defined via the optimal permutation of the affine model indices as 
\begin{equation*}
\label{eq:def_relative_error}
\mathtt{err}(\hat{\mb {\theta}})
%\phi_e 
\triangleq \min_{\pi\in\mathrm{Perm}([k])}\log_{10}\left({\sum_{j=1}^{k}\|\hat{\mb  \theta}_{\pi(j)} - \mb  \theta_j^\star\|_2^2/\sum_{j=1}^{k}\|\mb  \theta_{j}^\star\|_2^2}\right),
\end{equation*} 
where $\mathrm{Perm}([k])$ denotes the set of all permutations on $[k]$. 
\begin{figure}
\centering
\begin{tabular}{c}
% First subfigure for Gaussian distribution
%\begin{subfigure}{0.45\textwidth}
%        \centering
\includegraphics[scale=1]{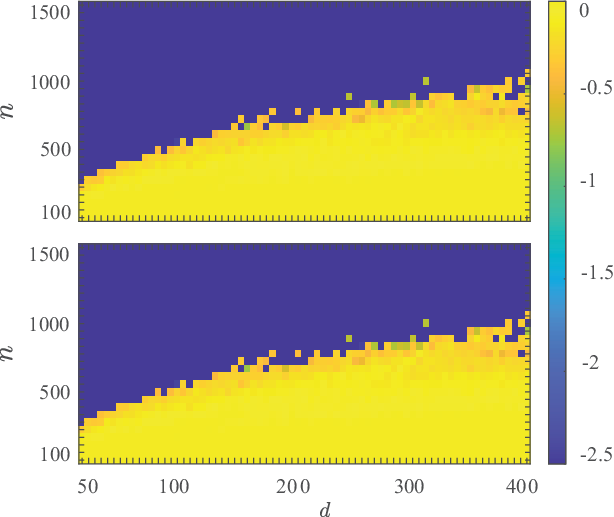} \\
%{\footnotesize (a) Gaussian} \\
%\caption{Gaussian}
%\end{subfigure}
%    \hfill % Space between the subfigures
% Second subfigure for Uniform distribution
%\begin{subfigure}{0.45\textwidth}
%        \centering
%\includegraphics[scale=0.5]{Figures/DimensionVsComplexity_uniform3.eps} \\ % Replace with the correct file path
%{\footnotesize (b) Uniform}
%\caption{{ Uniform }}
%    \end{subfigure}
\end{tabular}
\caption{Median of $\mathtt{err}(\hat{\mb {\theta}})$ for different ($n$,$d$) pairs using 50 Monte Carlo iterations for $k =3$ and $s=25$ with Gaussian (top) and Uniform (bottom) covariate distributions in the noiseless case.
}
\label{fig:P_vs_N}
\end{figure}
% \Cref{fig:P_vs_N}

Fig.~\ref{fig:P_vs_N} shows the empirical phase transition by Sp-GD per the total number of covariates $d$ when the number of active covariates is fixed to $s=25$ and the model order is fixed to $k=3$ in the noiseless case. 
We observe the empirical phase transition for Gaussian and uniform distributions both of which satisfy the assumptions of Theorem \ref{THM:MAIN}. 
The phase transition occurs when $n$ scales as a logarithmic function of $d$, corroborating the sample complexity in Theorem~\ref{THM:MAIN}. 
Next, Fig.~\ref{fig:S_vs_N_Noise} shows the empirical phase transition by Sp-GD per the number of active covariates $s$ when the total number of covariates and model order are fixed to $d=200$ and $k=3$, respectively. This figure corroborates that the sample complexity required to invoke the performance guarantee for Sp-GD scales sub-linearly in $s$ as $\mathcal{O}(s\log (d/s))$. 
We observe this scaling law when $s/k \geq 10$. 
On the other hand, when $s/k < 10$ the transition boundary increases as $s$ decreases. 
This is based on the observation of the left edges of the plots in Fig.~\ref{fig:S_vs_N_Noise}. 
The ground-truth parameters are randomly generated as independent and identically distributed with respect to the standard Gaussian distribution. 
In particular, the weight vectors are almost pairwise orthogonal when $s/k$ is sufficiently large (e.g. $s/k \geq 10$), which makes $\pi_{\min} \approx 1/k$. 
However, the correlations among the weight vectors increase as $s/k$ decreases, hence $\pi_{\min}$ decreases. 
This incurs the increase in the sample complexity in Theorem \ref{THM:MAIN}, which is aligned with the empirical observation on the phase transition boundary on the success regime, that is, $\mathtt{err}(\hat{\mb {\theta}}) \leq -2.5$. 
%This technical issue can be avoided if one explicitly controls generating the ground-truth parameter as an equiangular frame.
Finally, Fig.~\ref{fig:sig_vs_n} shows the empirical phase transition by Sp-GD per the noise variance $\sigma_z^2$ when the total number of covariates, the number of active covariates, and the model order are fixed to $d=200$, $s=50$ and $k=3$, respectively.
The empirical phase transition boundary is proportional to the noise variance once it exceeds a certain threshold. This corresponds to the multiplicative factor $\max(1,\sigma_z^2)$ in \eqref{eq:cond:lem:lwb_gradient}.

\begin{figure}
        \centering
        \includegraphics[scale=0.8]{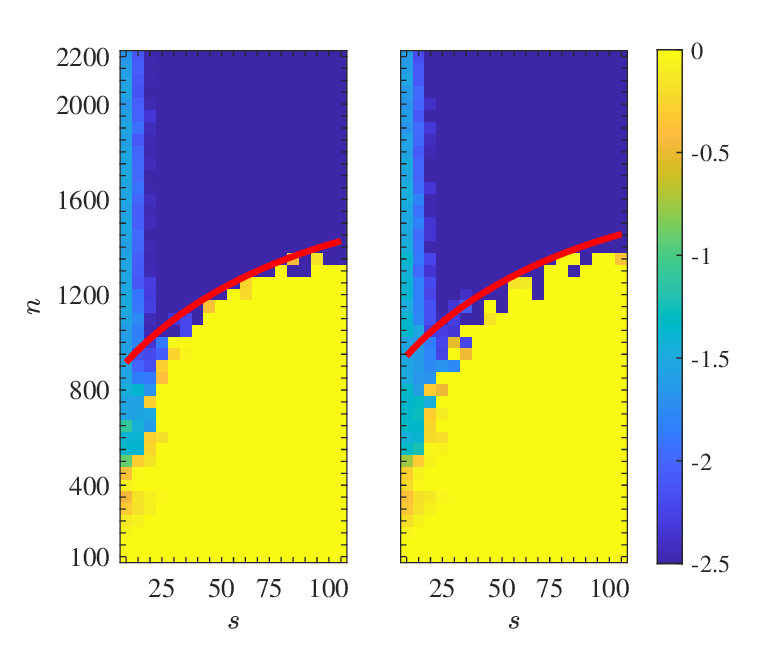} 
\caption{ Median of $ \mathtt{err}(\hat{\mb {\theta}})$ for different ($n$,$s$) pairs using 50 Monte Carlo iterations for $k=3$ and $d=200$ with Gaussian (left) and Uniform (right) covariate distributions. The red curves are fitted with respect to $s\log (d/s)$ at the phase transition boundary for both figures.
}
\label{fig:S_vs_N_Noise}
\end{figure}

\begin{figure}
        \centering
        \includegraphics[scale=0.8]{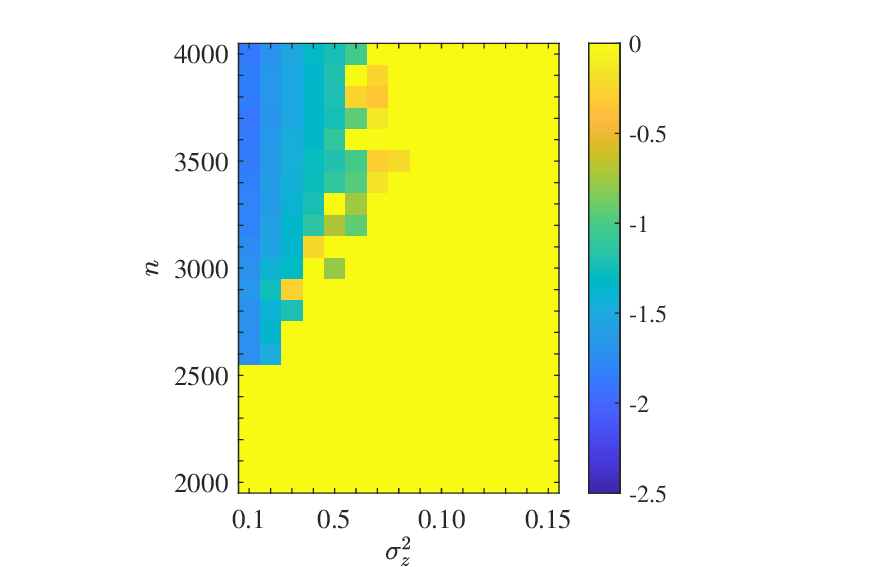} 
\caption{ Median of $ \mathtt{err}(\hat{\mb {\theta}})$ for different ($n$,$\sigma_z^2$) pairs using 50 Monte Carlo iterations for $s=50$, $d=200$ and $k=3$ with Gaussian covariates and local initial estimate.
}
\label{fig:sig_vs_n}
\end{figure}

\subsection{Subspace Estimation and Random Search for Initialization}\label{Sec:numinit}
This section delves into the detailed empirical analysis of the initialization method employed in the previous section. 
First, Fig. \ref{fig:PCAvsspca} demonstrates the gain of Algorithm \ref{algo:sPCA} with SPCA over the analogous spectral method with PCA. 
The regularization parameter $\lambda >0$ required by Algorithm \ref{algo:sPCA} is set by a parameter sweep over a sampling grid. 
Alternatively, this parameter can be tuned through cross-validation. 
\begin{figure}
    \centering
    \includegraphics[width=0.6\linewidth]{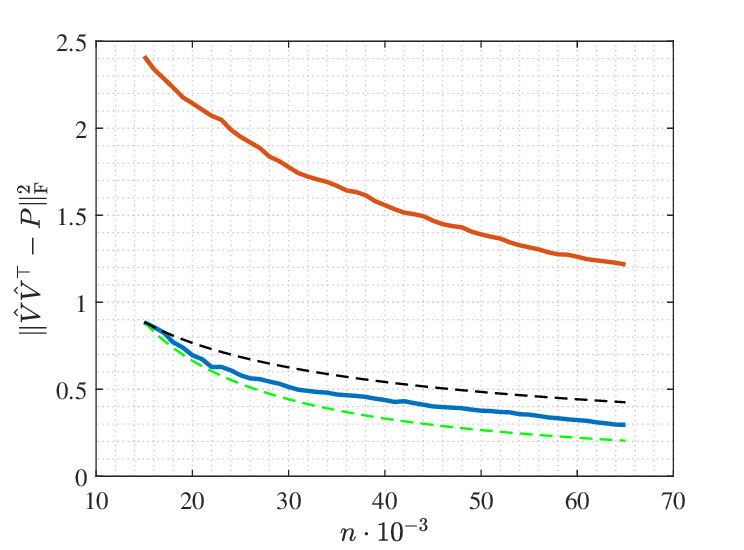}
    \caption{Projection error difference using PCA (red), and Algorithm \ref{algo:sPCA} (blue), and dashed guidelines showing $1/\sqrt{n}$ decay (black) and $1/n$ decay (green) with $s=20$, $d=200$, $k=3$, $\sigma_z=0.1$ and 50 Monte Carlo iterations.}
    \label{fig:PCAvsspca}
\end{figure}
The subspace estimation error using Algorithm \ref{algo:sPCA} shown in blue is significantly less than that by the PCA-based spectral method shown in red. 
Furthermore, the estimation error by Algorithm~\ref{algo:sPCA} decays at a rate between $1/\sqrt{n}$ and $1/n$. These observations are consistent with the error bound in Theorem \ref{Theo:Projection}.

Next, we investigate the empirical performance of the repeated random initialization that replaces the exhaustive discrete search in Algorithm \ref{algo:ksearch} by random sampling followed by Sp-GD $10$ iterations. 
In particular, we observe the estimation accuracy as the function in the number of random initializations $M$. 
Fig. \ref{fig:PCAvsspca_all} compares the estimation error by the repeated random initialization when the parameter subspace is estimated by the PCA-based spectral initialization \citep{ghosh2021max} (red) and by Algorithm~\ref{algo:sPCA} (blue). 
For both of the two subspace estimation methods, the initialization error monotonically decreases as the number of trials $M$ increases. However, the initialization performance by Algorithm \ref{algo:sPCA} is more accurate since it provides better subspace estimation performance.
\begin{figure}
    \centering
    \includegraphics[scale=0.7]{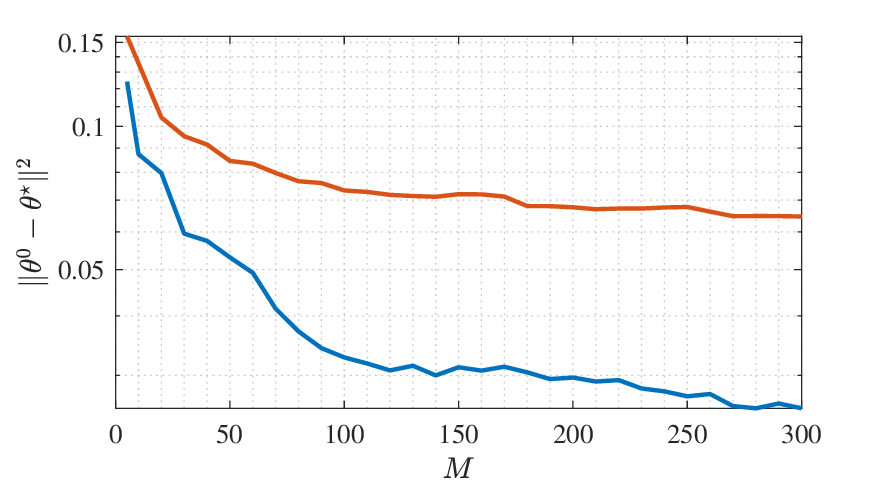}
    \caption{Parameter estimation error using PCA (red) and Algorithm \ref{algo:sPCA} (blue) when followed by $M$ random samples with $s=20$, $d=200$, $k=3$, $\sigma_z=0.1$ averaged over 50 Monte Carlo iterations.}
    \label{fig:PCAvsspca_all}
\end{figure}

\section{Discussion} \label{Conclusion}
We consider variable selection for a class of nonlinear regression models given by the maximum of $k$ affine models $\mb  x \mapsto \max_{j \in [k]} \langle \mb  a_j^\star, \mb  x \rangle + b_j^\star$ for $j = 1,\dots,k$ where $\{\mb  a_j^\star\}_{j=1}^k$ and $\{b_j^\star\}_{j=1}^k$ denote the ground-truth weight vectors and intercepts. 
The weight vectors $\{\mb  a_j^\star\}_{j=1}^k$ satisfy the joint $s$-sparse structure as we assume that only $s$ out of $d$ covariates contribute to explaining the response variable. 
This paper proposes a variant of the projected gradient algorithm, Sp-GD, to estimate the sparse model parameters.  
We provide non-asymptotic local convergence guarantees for Sp-GD under independent sub-Gaussian noise when the covariates follow a sub-Gaussian distribution satisfying the anti-concentration property.
Under these assumptions, when the ground-truth model order and parameters are fixed, a suitably initialized Sp-GD converges linearly to an $\epsilon$-accurate parameter estimate given $\mathcal{O}(\max(\epsilon^{-2}\sigma_z^2,1)s\log (d/s))$. In particular, when the observations are noise-free ($\sigma_z^2 =0$), Sp-GD guarantees exact parameter recovery. Since minimizing the squared loss of sparse max-affine models is non-convex, starting Sp-GD within the basin of attraction is crucial for its convergence to the desired estimate. 
For this purpose, we propose a modification of the spectral method by \citep{ghosh2021max} which estimates the span of the max-affine weight vectors $\{\mb  a_j^\star\}_{j=1}^k$ so that the subspace estimation utilizes the jointly sparse structure in the weight vectors via sparse principal component analysis. 
Combined with the $r$-covering search over the estimated subspace, the initialization scheme provides an $\epsilon$-accurate estimate when $r=\mathcal{O}(\epsilon)$ given $\mathcal{O}(\epsilon^{-2}\max (\sigma_z^4,\sigma_z^2,1)s^2\log^4d)$ observations when the ground-truth model parameters are fixed, and the covariates and noise follow Gaussian distributions. 
%This initialization scheme reduces 
The dominating factor of the sample complexity is $s^2$, which is significantly smaller than $d$ in the non-sparse case. 

One noteworthy limitation of the initialization is the $r$-covering search inherited from the previous work on non-sparse max-affine regression \citep{ghosh2021max}. 
The cost of constructing the $r$-covering and the exhaustive search over it increases exponentially in the subspace dimension. 
Therefore, the $r$-covering search is not practical for small separation parameter $r$ which is necessary for accurate estimates. 
However, this is the only known algorithm with theoretical guarantees for max-affine parameter initialization. Therefore, it would be a fruitful future direction to develop a practical initialization scheme that avoids the $r$-covering search and provides theoretical performance guarantees. 
Alternatively, it would also be intriguing to explore a potential analysis of Sp-GD from random initialization that will extend the known theoretical results on single-index models 
\citep{tan2019phase,chandrasekher2022alternating}.

\clearpage
%% References:

\bibliography{ref}
\bibliographystyle{plainnat}
\clearpage

\begin{appendices}
\section{Proof of Theorem~\ref{THM:MAIN}}\label{ProofSketch}

The proof is obtained by showing that each update in Sp-GD monotonically decreases the distance to the ground truth $\mb \theta^\star$ by a factor $\tau \in [0,1)$ up to an additive distortion and remains in the neighborhood of $\mb \theta^\star$, i.e.
\begin{subequations} \label{eq:recursion}
\begin{gather} 
\left\|\mb \theta^{t+1}-\mb \theta^\star\right\|^2_2 \leq \tau \left\|\mb \theta^t-\mb \theta^\star\right\|^2_2 +C_1\sigma^2_z\left(\frac{sk\log\left(n/s\right)+s\log\left({d}/{s}\right)+\log\left({1}/{\delta}\right)}{n}\right), \label{eq:recursion_linear} \\
    \mb \theta^{t+1} \in \mathcal{N}(\mb \theta^\star,\sqrt{2}\Delta\rho),\label{eq:recursion_nghb}
\end{gather}
\end{subequations} 
hold for all $ t\in\mathbb{N} \cup \{0\}$. 
{\color{black}
Here $\delta$ refers to the error probability, i.e. these statements hold with probability at least $1-\delta$.} 
We prove this statement by induction. Let $t$ be arbitrarily fixed. 
We assume that \eqref{eq:recursion} holds for all previous iterates. 
To show that \eqref{eq:recursion_linear} also holds for the current iterate, we introduce the following notation. 
Let $\mathcal{S}^\star$ and $\mathcal{S}^t$ denote the joint support of $\{\mb \theta_j^\star\}_{j=1}^k$ and $\{\mb \theta_j^t\}_{j=1}^k$, respectively.
The union of $\mathcal{S}^\star$ and $\mathcal{S}^t$ is denoted by $\mathcal{U}^t$. 
We use $\mb{ \Pi}_\mathcal{U} : \mathbb{R}^{d} \rightarrow \mathbb{R}^{d}$ to denote the orthogonal projection onto the subspace spanned by the standard basis vectors $\{\mb e_j\}_{j \in \mathcal{U}}$ for $\mathcal{U} \subset [d]$. 
Then the augmented operator $\mb{ \tilde\Pi}_\mathcal{U} : \mathbb{R}^{d+1} \rightarrow \mathbb{R}^{d+1}$ is defined as
\begin{equation}\label{eq:Proj_oper}
     \mb{ \tilde\Pi}_{\mathcal{U}}=
  \left[ {\begin{array}{cc}
   \mb{ \Pi}_{\mathcal{U}} & \mb 0_{d \times 1} \\
   \mb 0_{1 \times d} & 1 \\
  \end{array} } \right].
\end{equation}
The step size used in Sp-GD stated in Algorithm \ref{algo:SparseGD} is stated as
\begin{equation} \label{eq:def_mu}
    \mu_j(\mb \theta) = \left( \frac{1}{n} \sum_{i=1}^n \bbone_{\left\{\mb x_i\in \mathcal{C}_j(\mb \theta)\right\}}\right)^{-1}, \quad \forall j \in [k].
\end{equation}
For notational simplicity, let $\mu_j^t \triangleq\mu_j(\mb \theta^t)$, $\mathcal{C}_j^t\triangleq\mathcal{C}_j(\mb \theta^t)$, and $\mathcal{C}_j^\star\triangleq\mathcal{C}_j(\mb \theta^\star)$ for all $j \in [k]$.
Then the left-hand side of \eqref{eq:recursion} is upper-bounded by
\begin{align} 
\label{Induction}
\left\|\mb\theta^{t+1}-\mb \theta^\star \right\|_2
&= \left\|
\left(\mb I_k \otimes \mb{\tilde \Pi}_{\mathcal{U}^{t+1}}\right)
\left( \mb\theta^{t+1}-\mb \theta^\star \right) \right\|_2 \nonumber \\
&\leq \left\|\left(\mb I_k \otimes \mb{\tilde \Pi}_{\mathcal{U}^{t+1}}\right) \left( \mb \theta^{t+1} - \mb \alpha^{t+1} \right)\right\|_2 + \left\|\left(\mb I_k \otimes \mb{\tilde \Pi}_{\mathcal{U}^{t+1}}\right)\left(\mb \alpha^{t+1} - \mb \theta^\star \right) \right\|_2 \nonumber \\
& =\left( \sum_{j=1}^k \left\| \mb{\tilde \Pi}_{\mathcal{U}^{t+1}} \left( \mb \theta^{t+1}_{j}  - \mb \alpha^{t+1}_{j} \right) \right\|_2^2 \right)^{1/2}  + \left( \sum_{j=1}^k \left\| \mb{\tilde \Pi}_{\mathcal{U}^{t+1}} \left( \mb \alpha^{t+1}_{j} - \mb \theta^\star_{j} \right) \right\|_2^2 \right)^{1/2}  \nonumber \\
&\leq 2 \left( \sum_{j=1}^k \left\| \mb{\tilde \Pi}_{\mathcal{U}^{t+1}} \left( \mb \alpha^{t+1}_{j} - \mb \theta^\star_{j} \right) \right\|_2^2 \right)^{1/2} \nonumber \\
&= 2 \left( \sum_{j=1}^k \left\| \mb{\tilde \Pi}_{\mathcal{U}^{t+1}} \left( \mb \theta^t_{j} - \mu_j^t \nabla_{\mb\theta_j} \ell(\mb \theta^t) - \mb \theta^\star_{j} \right) \right\|_2^2 \right)^{1/2},
\end{align}
where the second inequality holds since 
\[
\left\| \mb{\tilde \Pi}_{\mathcal{U}^{t+1}}\left(\mb \theta^{t+1}_j- \mb \alpha_j^{t+1}\right)\right\|_2 \leq \left\| \mb{\tilde \Pi}_{\mathcal{U}^{t+1}}\left(\mb \alpha_j^{t+1}- \mb \theta_j^\star\right)\right\|_2, \quad \forall j \in [k]
\]
which follows from the fact that $ \mb{\tilde \Pi}_{\mathcal{U}^{t+1}} \mb \theta_j^{t+1} = \mb \theta_j^{t+1}$ coincides with the projection of $\mb{\tilde \Pi}_{\mathcal{U}^{t+1}}\mb \alpha_j^{t+1}$ onto $\Gamma_s$, and $ \mb{\tilde \Pi}_{\mathcal{U}^{t+1}} \mb \theta_j^\star = \mb \theta_j^\star$ belongs to $\Gamma_s$ for all $j \in [k]$.  
%where the first inequality follows from the triangular inequality, and the second inequality follows from the optimality of $\mb\theta^{t+1}_{j,\mathcal{U}^{t+1}} $,

Let $j \in [k]$ be arbitrarily fixed. We further proceed with the additional shorthand notations: $\mb{h}_j^t \triangleq \mb{\theta}_j^t -\mb{\theta_j}^\star$, $\mb v_{jj'}^t \triangleq \mb \theta_j^t -\mb \theta_{j'}^t$, and $\mb v^\star_{jj'} \triangleq \mb \theta^\star_j -\mb \theta^\star_{j'}$ for all $j'\neq j \in [k]$. 
Then, by the definition of $\mathcal{S}^\star$, we have 
\begin{equation}
\mb{\tilde \Pi}_{\mathcal{S}^\star}\mb v^\star_{jj'} 
= \mb{\tilde \Pi}_{\mathcal{S}^\star} \mb \theta^\star_j - \mb{\tilde \Pi}_{\mathcal{S}^\star} \mb \theta^\star_{j'} 
= \mb \theta^\star_j -\mb \theta^\star_{j'} = \mb v^\star_{jj'}
\end{equation}
and
\begin{equation}\label{eq:h_decomp}
\mb{h}_j^t 
= \mb{\tilde \Pi}_{\mathcal{U}^{t}}\mb{h}_{j}^t 
= \mb{\tilde \Pi}_{\mathcal{U}^{t} \cap \mathcal{U}^{t+1}}\mb{h}_{j}^t + \mb{\tilde \Pi}_{\mathcal{U}^{t}\backslash\mathcal{U}^{t+1}}\mb{h}_{j}^t
= \mb{\tilde \Pi}_{\mathcal{U}^{t+1}}\mb{h}_{j}^t + \mb{\tilde \Pi}_{\mathcal{U}^{t}\backslash\mathcal{U}^{t+1}}\mb{h}_{j}^t.
\end{equation}
Recall that the partial gradient in the right-hand side of \eqref{Induction} is written as 
\begin{equation}\label{eq:par_first}
\mb{\tilde \Pi}_{\mathcal{U}^{t+1}}\nabla_{\mb \theta_j}\ell(\mb \theta^t) 
=\frac{1}{n} \sum_{i=1}^n \bbone_{\{\mb x_i\in \mathcal{C}^t_j\}} \left(\langle\mb \xi_i,\mb \theta_j^t \rangle-\max_{j\in[k]} \langle\mb \xi_i,\mb \theta_j^\star\rangle - z_i\right)\mb{\tilde \Pi}_{\mathcal{U}^{t+1}}\mb \xi_{i}.
\end{equation}
We can obtain the following decomposition
\begin{align}\label{eq:decomp_det}
& \bbone_{\{\mb x_i\in \mathcal{C}^t_j\}} \left(\langle\mb \xi_i,\mb \theta_j^t \rangle-\max_{j\in[k]} \langle\mb \xi_i,\mb \theta_j^\star\rangle \right) \nonumber\\
&= \sum_{j'=1}^k\bbone_{\{\mb x_i\in \mathcal{C}^t_j \cap \mathcal{C}^\star_{j'}\}} \left(\langle\mb \xi_i,\mb \theta_j^t - \mb \theta_j^\star + \mb \theta_j^\star- \mb \theta_{j'}^\star\rangle\right)\nonumber\\
&=
\bbone_{\{\mb x_i\in \mathcal{C}^t_j\}} \langle\mb \xi_i,\mb h_j^t\rangle + \sum_{j'\neq j}\bbone_{\{\mb x_i\in \mathcal{C}^t_j \cap \mathcal{C}^\star_{j'}\}} \langle\mb \xi_i,\mb v_{jj'}^\star\rangle\nonumber \\
&=
\bbone_{\{\mb x_i\in \mathcal{C}^t_j\}} \langle\mb \xi_i,\mb{\tilde \Pi}_{\mathcal{U}^{t+1}}\mb{h}_{j}^t\rangle 
+ \bbone_{\{\mb x_i\in \mathcal{C}^t_j\}} \langle\mb \xi_i,\mb{\tilde \Pi}_{\mathcal{U}^{t}\backslash\mathcal{U}^{t+1}}\mb h_j^t\rangle 
+ \sum_{j'\neq j}\bbone_{\{\mb x_i\in \mathcal{C}^t_j \cap \mathcal{C}^\star_{j'}\}} \langle\mb \xi_i,\mb v_{jj'}^\star\rangle
\end{align}
where the first equality follows from $\{\mathcal{C}^\star_{j'}\}_{j'=1}^k$ being a partition of $\mathbb{R}^d$, the second inequality follows from the definitions of $\mb h_j^t$ and $\mb v_{jj'}^\star$, and the last equality follows from \eqref{eq:h_decomp}. We now use \eqref{eq:decomp_det} to rewrite \eqref{eq:par_first} as
%from (\ref{p_gradient}) as
\begin{equation} 
\label{Expanded_gradient2}
\begin{aligned}
\mb{\tilde \Pi}_{\mathcal{U}^{t+1}}\nabla_{\mb \theta_j}\ell(\mb \theta^t)      &=\underbrace{\frac{1}{n} \sum_{i=1}^n \bbone_{\{\mb x_i\in \mathcal{C}^t_j  \}} \langle \mb{\tilde \Pi}_{\mathcal{U}^{t+1}} \mb \xi_{i}, \mb h^t_{j}\rangle\mb{\tilde \Pi}_{\mathcal{U}^{t+1}}\mb\xi_{i}}_{\mb p_j} \\
&\quad+\underbrace{\frac{1}{n} \sum_{i=1}^n \bbone_{\{\mb x_i\in \mathcal{C}^t_j  \}} \langle \mb{\tilde \Pi}_{\mathcal{U}^{t}\backslash\mathcal{U}^{t+1}}\mb \xi_{i},\mb h^t_{j}\rangle \mb{\tilde \Pi}_{\mathcal{U}^{t+1}}\mb\xi_{i}}_{\mb q_j} \\
&\quad+\underbrace{\frac{1}{n} \sum_{\substack{i=1\\j'\neq j}}^n \bbone_{\{\mb x_i\in \mathcal{C}^t_j \cap \mathcal{C}_{j'}^\star \}} \langle\mb \xi_{i},\mb v_{jj'}^\star\rangle\mb{\tilde \Pi}_{\mathcal{U}^{t+1}}\mb \xi_{i}}_{\mb{c}_j}  \quad-\underbrace{\frac{1}{n} \sum_{i=1}^n \bbone_{\{\mb x_i\in \mathcal{C}^t_j  \}} z_i\mb{\tilde \Pi}_{\mathcal{U}^{t+1}}\mb\xi_{i}}_{\mb d_j}.
\end{aligned}
\end{equation}
Plugging \eqref{Expanded_gradient2} into \eqref{Induction} yields 
\begin{align} \label{eq:abcd}
\frac{1}{4}\left\|\mb h^{t+1}   \right\|_2^2
&\leq \sum_{j=1}^k \left[\left\|\mb{\tilde \Pi}_{\mathcal{U}^{t+1}} \left(\mb h^t_{j}-\mu_j^t \mb p_j\right)\right\|_2 + \mu_j^t \left(\left\| \mb q_j\right\|_2 +  \left\| \mb c_j\right\|_2+\left\|\mb d_j\right\|_2\right)\right]^2.
\end{align}
We now need to derive an upper bound on each term on the right-hand side of \eqref{eq:abcd}. Let $\pi_j^t \triangleq\mathbb{P}(\mb x \in \mathcal{C}_j(\mb \theta^t))$ and $\pi_j^\star \triangleq\mathbb{P}(\mb x \in \mathcal{C}_j(\mb \theta^\star))$. 
We now state a lemma which combines all the events that hold with high probability that are used for proving Theorem \ref{THM:MAIN}.
\begin{lemma}\label{lem:combined}
    Instate the assumptions and definitions in Theorem \ref{THM:MAIN}, and let $\epsilon_\mathrm{min} \triangleq k^{-3/2} \pi_{\mathrm{min}}^{2(1+\zeta^{-1})}$. Then, the following events hold jointly for all $t \in \mathbb{N} \cup \{0\}$ and $j \in [k]$ with probability at least $1-\delta$:
\begin{equation}
\label{eq:bnd_concentstate}
\sup\limits_{\substack{|\mathcal{U}| \leq s }}\left\|\mb{\tilde \Pi}_{\mathcal{U}}\left(\frac{1}{n}\sum_{i=1}^{n}\bbone_{\{\mb x_i\in \mathcal{C}_j^t \}}\left(\mb \xi_{i}\mb \xi_{i}^\T-\mb I_{d+1}\right)\right)\mb{\tilde \Pi}_{\mathcal{U}}\right\|\leq \epsilon_\mathrm{min},
\end{equation}
\begin{equation}\label{eq:empiricalstate}
    \left|\frac{1}{n}\sum_{i=1}^{n}\bbone_{\{\mb x_i \in \mathcal{C}_j^t\}} - \pi_j^t \right|\leq \epsilon_\mathrm{min},
\end{equation}
\begin{equation} \label{eq:vboundstate}
    \frac{1}{n} \sum_{j'\neq j}\sum_{\substack{i=1 }}^{n} \bbone_{\{\mb x_i\in \mathcal{C}_j^t \cap \mathcal{C}_{j'}^\star \}} \langle\mb{\tilde \Pi}_{\mathcal{S}^\star}\mb \xi_{i},\mb v_{jj'}^\star\rangle^2 \leq \frac{2}{5\gamma k}\left(\frac{\pi_{\min}}{16}\right)^{1+\zeta^{-1}}  \sum_{j'\neq j} \|\mb v_{jj'}^t - \mb v_{jj'}^\star\|_2^2,
\end{equation}
\begin{equation}
\label{eq:noise_state}
   \sup\limits_{\substack{|\mathcal{U}| \leq s }} \norm{\frac{1}{n} \sum_{i=1}^n \bbone_{\{\mb x_i\in \mathcal{C}_j^t  \}} z_i\mb{\tilde \Pi}_{\mathcal{U}}[\mb x_{i};1]}_2\leq C\sigma_z\sqrt{\frac{sk\log\left({n}/{s}\right)+s\log\left({d}/{s}\right)+\log\left({1}/{\delta}\right)}{n}}\triangleq\tau_{\mathrm{noise}}, 
\end{equation}
\begin{equation} \label{eq:rip_state}
    {\color{black}\sup\limits_{\substack{|\mathcal{U}| \leq s }}\norm{\frac{1}{\sqrt{n}}\mb{\tilde \Pi}_{\mathcal{U}} [\bbone_{\{\mb x_1\in \mathcal{C}^t_j \cap \mathcal{C}_{j'}^\star \}}\mb \xi_1,\ldots,\bbone_{\{\mb x_n\in \mathcal{C}^t_j \cap \mathcal{C}_{j'}^\star \}}\mb \xi_n]}\leq {\frac{\pi_{\min}^{(1+\zeta^{-1})/2}}{k^{1/2}}}\quad \forall j' \neq j.}
\end{equation}
\end{lemma}
The proof of all statements in this lemma is deferred to Appendix \ref{sec:lemmas}. We proceed with the proof under the assumption that all the statements in Lemma \ref{lem:combined} hold.
We can now begin upper bounding the terms in \eqref{eq:abcd}. The first summand in \eqref{eq:abcd} is upper-bounded as
\begin{align} \label{eq:ha_center}
\|\mb{\tilde \Pi}_{\mathcal{U}^{t+1}}\left(\mb h^t_{j}-\mu \mb p_j\right)\|_2 
 \leq \norm{\mb{\tilde \Pi}_{\mathcal{U}^{t+1}} \left(\mb{I}_{d+1} - \frac{\mu_j k}{n} \sum_{i=1}^n \bbone_{\{\mb x_i\in \mathcal{C}_j  \}} \mb\xi_{i} \mb\xi_{i}^\mathsf{T}\right) \mb{\tilde \Pi}_{\mathcal{U}^{t+1}}} \cdot \norm{\mb{\tilde \Pi}_{\mathcal{U}^{t+1}}\mb h^t_{j}}_2 .
\end{align}
The first factor on the right-hand side can be upper bound by the triangle inequality as
\begin{align}\label{eq:bndtermsstate}
& \norm{\mb{\tilde \Pi}_{\mathcal{U}^{t+1}}\mb \left(\mb{I}_{d+1} - \frac{\mu^t_j}{n} \sum_{i=1}^n \bbone_{\{\mb x_i\in \mathcal{C}^t_j \}} \mb\xi_{i} \mb\xi_{i}^\mathsf{T}\right) \mb{\tilde \Pi}_{\mathcal{U}^{t+1}} } \nonumber\\
& \leq \mu^t_j  \left(\underbrace{\left\|  \mb{\tilde \Pi}_{\mathcal{U}^{t+1}} \left(\frac{1}{n}\sum_{i=1}^{n}\bbone_{\{\mb x_i\in \mathcal{C}^t_j  \}}\left(\mb \xi_{i}\mb \xi_{i}^\T-\mb I_{d+1}\right)\right)\mb{\tilde \Pi}_{\mathcal{U}^{t+1}}\right\|}_{\mathcal{A}}+\underbrace{\left|\frac{1}{n}\sum_{i=1}^{n}\bbone_{\{\mb x_i\in \mathcal{C}^t_j  \}}-\frac{1}{\mu^t_j }\right|}_{\mathcal{B}}\right).
\end{align}
\begin{comment}
By applying Lemma \ref{lem:bnd_singular}, we have 
\begin{align}
    \mathcal{A} &= 
    \sup_{\mb \theta_j ~\neq~ \mb 0}\left\|  \mb{\tilde \Pi}_{\mathcal{U}^{t+1}}\left(\frac{1}{n}\sum_{i=1}^{n}\bbone_{\{\mb x_i\in \mathcal{C}_j  \}}\left(\mb \xi_{i}\mb \xi_{i}^\T-\mb I_{d+1}\right)\right)\mb{\tilde \Pi}_{\mathcal{U}^{t+1}}\right\| \nonumber \\
    &=  \sup\limits_{\substack{ \{\omega_i\}_{i=1}^n ~\in~ \Omega \\ \mb{ \Pi}_{\mathcal{U}^{t+1}} \in \mathcal{Z}(2s)}}
    \left\|\mb{ \Pi}_{\mathcal{U}^{t+1}}\left(\frac{1}{n}\sum_{i=1}^{n}\omega_i\left(\mb x_{i}\mb x_{i}^\T-\mb I_{d}\right)\right)\mb{ \Pi}_{\mathcal{U}^{t+1}}\right\|\leq \epsilon,
\end{align}
with probability at least $1-\delta$ provided if $n\geq C\epsilon^{-2}s\log\left(\frac{nd}{s\delta} \right)$ for some $C>0$.
\end{comment}
By \eqref{eq:bnd_concentstate} we have that $\mathcal{A}\leq \epsilon_{\mathrm{min}}$. Furthermore, $\mathcal{B} = 0$ by the definition of $\mu_j$. Also, since $\mb \theta^t \in \mathcal{N}(\mb \theta^\star,\sqrt{2} \Delta\rho)$, by Lemma \ref{lem:model_pis} we have that 
\begin{equation} 
\label{eq:piboundstate}
(1-\varrho)  \leq \frac{\pi_j^t}{\pi_j^\star} \leq \left(\frac{1-\varrho}{1-2\varrho}\right), \quad \varrho \triangleq {R}^{2\zeta}{k^{-2(1+\zeta^{-1})}}
\end{equation}
Therefore, using \eqref{eq:empiricalstate} and \eqref{eq:piboundstate}, we can upper bound the step size as
\begin{equation} \label{eq:mujbound22}
    \mu^t_j \triangleq \frac{1}{\frac{1}{n}\sum_{i=1}^{n}\bbone_{\{\mb x_i \in \mathcal{C}^t_j\}}} \leq \frac{1}{\pi^t_j - \epsilon_{\mathrm{min}}}\leq  \frac{1}{(1-\varrho)\pi^\star_j - \epsilon_{\mathrm{min}}} 
\end{equation}
Finally, the first summand in \eqref{eq:abcd} is upper-bounded as 
\begin{equation}\label{eq:ha_center_final}
 \|\mb{\tilde \Pi}_{\mathcal{U}^{t+1}}\left(\mb h^t_{j}-\mu_j^t \mb p_j\right)\|_2 \leq \frac{\epsilon_{\mathrm{min}}}{(1-\varrho)\pi^\star_j - \epsilon_{\mathrm{min}}} \norm{\mb{\tilde \Pi}_{\mathcal{U}^{t+1}}\mb h^t_{j}}_2.
\end{equation}

{\color{black}
The second summand of \eqref{eq:abcd} is written as 
\begin{align} \label{eq:bbound}
    \mu^t_j \| \mb q_j\|_2 &= \mu^t_j\norm{\frac{1}{n}\sum_{i=1}^n\bbone_{\{ \mb x \in \mathcal{C}_j^t\}} \tilde{\mb \Pi}_{\mathcal{U}^{t+1}} \mb \xi_i \mb \xi_i^\mathsf{T} \tilde{\mb \Pi}_{\mathcal{U}^t\setminus \mathcal{U}^{t+1}}  \mb h_j^t } \nonumber \\
    & = \mu^t_j \norm{\left[\frac{1}{n}\sum_{i=1}^n\bbone_{\{ \mb x \in \mathcal{C}_j^t\}} \tilde{\mb \Pi}_{\mathcal{U}^{t+1}} \left(\mb \xi_i \mb \xi_i^\mathsf{T} - \mb I_{d+1}\right) \tilde{\mb \Pi}_{\mathcal{U}^t\setminus \mathcal{U}^{t+1}}\right]  \tilde{\mb \Pi}_{\mathcal{U}^t\setminus \mathcal{U}^{t+1}}\mb h_j^t } \nonumber \\
    & \leq \frac{1}{(1-\varrho)\pi^\star_j - \epsilon_{\mathrm{min}}} \norm{\frac{1}{n}\sum_{i=1}^n\bbone_{\{ \mb x \in \mathcal{C}_j^t\}} \tilde{\mb \Pi}_{\mathcal{U}^{t+1}} \left(\mb \xi_i \mb \xi_i^\mathsf{T} - \mb I_{d+1}\right) \tilde{\mb \Pi}_{\mathcal{U}^t\setminus \mathcal{U}^{t+1}}} \cdot\norm{  \tilde{\mb \Pi}_{\mathcal{U}^t\setminus \mathcal{U}^{t+1}}\mb h_j^t } \nonumber \\
    & \leq \frac{\epsilon_{\mathrm{min}}}{(1-\varrho)\pi^\star_j - \epsilon_{\mathrm{min}}} \norm{  \tilde{\mb \Pi}_{\mathcal{U}^t\setminus \mathcal{U}^{t+1}}\mb h_j^t }, 
\end{align}
where the second equality follows from the idempotency of projection matrices and the observation that $\tilde{\mb \Pi}_{\mathcal{U}^{t+1}} \mb I_{d+1}  \tilde{\mb \Pi}_{\mathcal{U}^t\setminus\mathcal{U}^{t+1}} = \mb 0$, the first inequality follows from the upper bound on $\mu_j^t$ in \eqref{eq:mujbound22} and the definition of the operator norm, and the last inequality follows from \eqref{eq:bnd_concentstate}. 
}

The vector $\mb c_j$ in the last term of \eqref{Expanded_gradient2} is factorized as $\mb c_j = \frac{1}{n}\mb E \mb v$, where
\begin{equation*}
\mb v \triangleq
\sum_{j'\neq j} 
\begin{bmatrix}
\bbone_{\{\mb x_1\in \mathcal{C}^t_j \cap \mathcal{C}_{j'}^\star \}} \langle\mb{\tilde \Pi}_{\mathcal{S}^\star}\mb \xi_{1},\mb v_{jj'}^\star\rangle 
\\ 
\vdots 
\\
\bbone_{\{\mb x_n\in \mathcal{C}^t_j \cap \mathcal{C}_{j'}^\star \}} \langle\mb{\tilde \Pi}_{\mathcal{S}^\star}\mb \xi_{n},\mb v_{jj'}^\star\rangle 
\end{bmatrix},
\end{equation*}
and $\mb E = \mb{\tilde \Pi}_{\mathcal{U}^{t+1}}[\bbone_{\{\mb x_1\in \mathcal{C}^t_j \cap \mathcal{C}_{j'}^\star \}}\mb \xi_1,\ldots,\bbone_{\{\mb x_n\in \mathcal{C}^t_j \cap \mathcal{C}_{j'}^\star \}}\mb \xi_n]$. Therefore, we have
\begin{align} \label{eq:cbound}
\| \mb c_j\|_2 
\leq \norm{\frac{1}{\sqrt{n}}\mb E} \cdot \norm{\frac{1}{\sqrt{n}}\mb v}_2 \leq
 {\frac{\pi_{\min}^{(1+\zeta^{-1})/2}}{k^{1/2}}}\norm{\frac{1}{\sqrt{n}}\mb v}_2,
\end{align}
where the second inequality follows from \eqref{eq:rip_state}. Next we bound the last term in \eqref{eq:cbound} as
\begin{align} 
\label{eq:vbound}
\frac{1}{n} \norm{\mb v }^2_2  
%&= \frac{1}{n} \sum_{i=1}^{n}\Big(     \sum_{j':j'\neq j} \bbone_{\{\mb x_i\in \mathcal{C}_j \cap \mathcal{C}_{j'}^\star \}} \langle\mb \xi_{i,\mathcal{S}^\star},\mb v_{jj', \mathcal{S}^\star}^\star\rangle        \Big)^2 \nonumber \\
&= \frac{1}{n} \sum_{i=1}^{n}    \sum_{j':j'\neq j} \bbone_{\{\mb x_i\in \mathcal{C}_j \cap \mathcal{C}_{j'}^\star \}} \langle\mb{\tilde \Pi}_{\mathcal{S}^\star}\mb \xi_{i},\mb v_{jj'}^\star\rangle^2 \nonumber \\
    &\leq \frac{2}{5\gamma}\left(\frac{\pi_{\min}}{16}\right)^{1+\zeta^{-1}}k^{-1}  \sum_{j':j'\neq j} \|\mb v_{jj'}^t - \mb v_{jj'}^\star\|_2^2 \nonumber \\ 
    %\end{align}
    %\begin{align}
    & =\frac{2}{5\gamma}\left(\frac{\pi_{\min}}{16}\right)^{1+\zeta^{-1}}k^{-1} \sum_{j':j'\neq j} \|\mb h_j^t - \mb h_{j'}^t\|_2^2 \nonumber\\
    &\leq \frac{{\color{black}4}}{5\gamma}\left(\frac{\pi_{\min}}{16}\right)^{1+\zeta^{-1}}k^{-1} \sum_{j':j'\neq j}\left( \|\mb h_j^t\|_2^2 +\| \mb h_{j'}^t\|_2^2\right),
\end{align}
%the second equality follows from $\mathcal{C}_j^\star\cap\mathcal{C}_l^\star = \varnothing ~ \forall~ l \neq j \in [k]$ as defined in (\ref{eq:def_calCjstar}), and 
{\color{black} where the first equality follows from the non-overlapping property of set partitions, i.e. $\mathcal{C}_j^t \cap \mathcal{C}_{j'}^t = \mathcal{C}_q^\star \cap \mathcal{C}_{q'}^\star  = \varnothing$ when $j\neq j'$ and $q\neq q'$}, and the first inequality follows from \eqref{eq:vboundstate}. Since the $\ell_1$ norm dominates the $\ell_2$ norm, we can write \eqref{eq:vbound} as
\begin{align} \label{eq:finalvbound}
   \norm{   \frac{1}{\sqrt{n}} \mb v }_2  \leq \sqrt{\frac{4}{5\gamma}\left(\frac{\pi_{\min}}{16}\right)^{1+\zeta^{-1}}k^{-1}}   \sum_{j':j'\neq j}\left( \|\mb h_j^t\|_2 +\| \mb h_{j'}^t\|_2\right) \leq \sqrt{\frac{4}{5\gamma}\left(\frac{\pi_{\min}}{16}\right)^{1+\zeta^{-1}}k}  \sum_{j'=1}^k  \|\mb h_{j'}^t\|_2.
\end{align}
Therefore, we have that 
\begin{equation}
    \norm{\mb c_j}_2 \leq\underbrace{\sqrt{\frac{4\cdot 16^{-(1+\zeta^{-1})}\pi_{\min}^{2(1+\zeta^{-1})}}{5\gamma}} }_{\lambda} \sum_{j'=1}^k  \|\mb h_{j'}^t\|_2
\end{equation}
Now, it remains to bound the last term on the right-hand side of \eqref{eq:abcd} using \eqref{eq:noise_state} such that
\begin{equation}\label{eq:dbound}
    \|\mb d_j\|_2 \leq \tau_{\mathrm{noise}}.
\end{equation}
Finally plugging the above upper bounds into \eqref{eq:abcd} yields 
\begin{align}\label{finaltau}
     & \| \mb h^{t+1}\|^2_2 \nonumber \\
     &\leq 4\sum_{j=1}^{k} \Bigg\{ \frac{1}{(1-\varrho)\pi^\star_j - \epsilon_{\mathrm{min}}}  \bigg[\epsilon_{\min}\| \mb{\tilde \Pi}_{\mathcal{U}^{t+1}}\mb h^{t}_{j}\|_2  + \epsilon_{\min} \| \mb{\tilde \Pi}_{\mathcal{U}^{t}\backslash\mathcal{U}^{t+1}}\mb h^{t}_{j}\|_2 + \lambda  \sum_{j'=1}^k  \| \mb h^{t}_{j'}\|_2 +\| \mb d_j\|_2\bigg]\Bigg\}^2 \nonumber \\
     &\leq  4\left[\frac{}{(1-\varrho)\pi^\star_j - \epsilon_{\mathrm{min}}}\right]^2\sum_{j=1}^{k}\left\{ \sqrt{2}\epsilon_{\min}  \| \mb h^{t}_{j}\|_2 + \lambda  \sum_{j'=1}^k  \| \mb h^{t}_{j'}\|_2 +\| \mb d_j\|_2\right\}^2\nonumber \\
     & \leq  \underbrace{{\color{black}12} \left(\frac{1}{(1-\varrho)\pi_{\mathrm{min}} - \epsilon_{\mathrm{min}}}\right)^2\bigg[ 2\epsilon^2_{\min}+  (\lambda k)^2\bigg]}_{\tau}  \sum_{j=1}^{k}\| \mb h^{t}_{j}\|^2_2+12k\left[\frac{\| \mb d_j\|_2 }{(1-\varrho)\pi_{\min} - \epsilon_{\mathrm{min}}}\right]^2\nonumber\\
     & = \tau  \| \mb h^{t}\|^2_2 + 12k\left[\frac{\| \mb d_1\|_2 }{(1-\varrho)\pi_{\min} - \epsilon_{\mathrm{min}}}\right]^2 \nonumber \\
     &= \tau  \| \mb h^{t}\|^2_2 +\underbrace{C\sigma^2_z\left(\frac{sk\log\left(\frac{n}{s}\right)+s\log\left(\frac{d}{s}\right)+\log\left(\frac{1}{\delta}\right)}{n}\right)\cdot \frac{k}{\left[(1-\varrho)\pi_{\min} - \epsilon_{\mathrm{min}}\right]^2}}_{\rho_{\mathrm{noise}}} \nonumber \\
     & = \tau  \| \mb h^{t}\|^2_2 + \rho_{\mathrm{noise}}
     ,
 \end{align}
for $\tau \in [0,1)$. The second inequality follows from 
\begin{align}
     \| \mb{\tilde \Pi}_{\mathcal{U}^{t+1}}\mb h^{t}_{j}\|_2 +\| \mb{\tilde \Pi}_{\mathcal{U}^{t}\backslash\mathcal{U}^{t+1}}\mb h^{t}_{j}\|_2 \leq \sqrt{2} \| \mb{\tilde \Pi}_{\mathcal{U}^{t}}\mb h^{t}_{j}\|_2,
 \end{align}
 and the third inequality follows trivially from $\pi_j^\star \geq \pi_{\mathrm{min}}^\star$ for all $ j \in [k]$ which finally verifies the first assertion in \eqref{eq:recursion_linear}. By the recursive nature of \eqref{eq:recursion_linear}, we have that
\begin{equation} \label{eq:recursion2}
    \left\|\mb \theta^{t+1}-\mb \theta^\star\right\|^2_2 \leq \tau^{t+1} \left\|\mb \theta^0-\mb \theta^\star\right\|^2_2 + \frac{\rho_{\mathrm{noise}}}{1-\tau}.
\end{equation}
This first term on the right-hand side of \eqref{eq:recursion2} is upper bounded as 
{\color{black}
\begin{equation}\label{eq:left_rec_bound}
    \tau^{t+1}\left\|\mb \theta^{0}-\mb \theta^\star\right\|^2_2 < \left\|\mb \theta^{0}-\mb \theta^\star\right\|^2_2 \leq(\Delta\rho)^2.
\end{equation}
where the first inequality follows from $\tau <1$ and the second inequality follows from $\mb \theta^0 \in \mathcal{N}(\mb \theta^\star, \Delta\rho)$.}
Next, the second term on the right-hand side of \eqref{eq:recursion2} is upper bounded as 
\begin{equation}\label{eq:noise_rec_bound}
    \frac{\rho_{\mathrm{noise}}}{1-\tau} \overset{\mathrm{(i)}}{\leq}\frac{\rho_{\mathrm{noise}}}{1-\tau}k^{-3}\pi_{\min}^{2(1+\zeta^{-1})}C_1/C \overset{\mathrm{(ii)}}{\leq}  (\Delta\rho)^2,
\end{equation}
where (i) follows from the sample complexity requirement in \eqref{eq:cond:lem:lwb_gradient}, and (ii) follows for large enough $C>0$. Combining \eqref{eq:left_rec_bound} and \eqref{eq:noise_rec_bound} implies 
\begin{equation}
    \left\|\mb \theta^{t+1}-\mb \theta^\star\right\|^2_2 \leq  2(\Delta\rho)^2 \implies \mb \theta^{t+1} \in \mathcal{N}(\mb \theta^\star,\sqrt{2}\Delta\rho),
\end{equation}
which concludes the proof using the strong law of induction.

\section{Auxiliary Lemmas for Theorem \ref{THM:MAIN}} \label{sec:lemmas}
This section will introduce several lemmas used to prove Theorem \ref{THM:MAIN}. To state these lemmas, we provide the following definitions. First, define the collection of all possible support sets of cardinality $s$ as 
\[
    \mathcal{Z}_s \triangleq \left\{ \mathcal{U} \subset [d] : | \mathcal{U}| =s \right\}.
\]
Next, the set of all polytopes determined by $k$ jointly $s$-sparse halfspaces is defined as 
\begin{equation}\label{eq:Pkds}
\mathcal{P}_{k,d,s} = \bigcup_{\mathcal{U} \in \mathcal{Z}_s} \mathcal{P}_{k,d}\left(\mathcal{U}\right),
\end{equation}
where
\begin{equation}\label{eq:Pkdtheta}
\mathcal{P}_{k,d}(\mathcal{U}) \triangleq \left\{ \mb x \in \mathbb{R}^d : [\mb x]_{\mathcal{U}^\mathsf{c}} = \mb 0, ~  M[\mb x]_{\mathcal{U}}\geq \mb b,~ M \in \mathbb{R}^{k \times s}, ~ \mb b \in \mathbb R^k  \right\}.
\end{equation}
We now proceed with the statement of the lemmas. 
The next lemma is used to upper bound the worst-case operator norm of sub-Gaussian matrices with independent and jointly sparse columns.
\begin{lemma}\label{lem:Sparse_RIP}
    Let $\{\mb x_i\}_{i=1}^{n}$ be independent copies of a random vector $\mb x \in \mathbb R^d$ which satisfies Assumptions \ref{assum:subg} and \ref{assum:anti}. Let $\{\omega_i\}_{i=1}^n \in \{0,1\}^n$ be fixed with $\sum_{i=1}^n \omega_i = \ell>0$. Then for all $\epsilon \in [0,1]$, there exists an absolute constant $C>0$ where it holds with probability at least $1-\delta$ that
    \begin{equation}
        \underset{\mathcal{U} \in \mathcal{Z}_s}{\mathrm{sup}}\norm{\mb{\tilde \Pi}_{\mathcal{U}}\sum_{i=1}^n \omega_i\left([\mb x_i;1][\mb x_i;1]^{\T} -\mb I_{d+1}\right)\mb{\tilde \Pi}_{\mathcal{U}}} \leq \ell\epsilon 
    \end{equation}
    if 
\begin{equation}\label{n:Sparse_RIP}
        l \geq C (\eta \vee 1)^4 \epsilon^{-2}\left[ s\log\left(\frac{d}{s}\right)+\log \left(\frac{1}{\delta}\right) \right].
    \end{equation}
\end{lemma}
\begin{proof}
    The proof of the lemma relies on the unitary invariance of the spectral norm. Let $\{r_i\}_{i=1}^{n}$ be independent copies of the random variable $r$ following the Rademacher distribution. Also, let $\mb{\tilde \xi}_i \triangleq r_i[\mb x_i;1]$ for all $i \in [n]$. Then for $\mb u \in \mathbb R^d$ and $\lambda \in \mathbb R$ and for all $i \in [n]$ we have
    \begin{align}
        \mathbb E \left[\exp\left(\langle[\mb u; \lambda],\mb{\tilde \xi}_i \rangle\right)\right] &= \frac{e^{\lambda}}{2}\mathbb E \left[\exp\left(\langle\mb u,\mb{ x}_i \rangle\right)\right] + \frac{e^{-\lambda}}{2}\mathbb E \left[\exp\left(-\langle\mb u,\mb{ x}_i \rangle\right)\right] \nonumber \\
        & \overset{\mathrm{(i)}}{\leq} \frac{1}{2}\exp\left(\frac{\norm{\mb u }^2 \eta^2}{2}\right)\left(e^\lambda + e^{-\lambda}\right) \nonumber \\
        & \overset{\mathrm{(i)}}{\leq} \exp\left(\frac{\norm{\mb u }^2 \eta^2 + \lambda^2}{2} \right) \leq \exp\left(\frac{(\eta \vee 1)^2\norm{[\mb u; \lambda] }^2 }{2} \right),
    \end{align}
    where (i) follows since $\{\mb x_i\}_{i=1}^{n}$ are $\eta$-sub-Gaussian random vectors, and (ii) follows using $e^{a^2/2}\geq \left(e^a + e^{-a}\right)/2$ for all $a \in \mathbb R$. Therefore, we have that $\{\mb{\tilde \xi}\}_{i=1}^{n}$ are identical copies of a random vector $\mb{\tilde \xi}$ which satisfies Assumption \ref{assum:subg} and is $(\eta \vee 1)$-sub-Gaussian.
    The proof of this lemma becomes a direct application of the union bound by inflating the probability of error by $|\mathcal{Z}_s|=$ $\binom{d}{s}$ 
    $\leq \left(\frac{ed}{s}\right)^s$ to the statement in \citep[Theorem 6.5]{wainwright2019high}.    
\end{proof}

The next lemma presents that the empirical measure of sparse polytopes concentrates around the expectation. 
%[Empirical probability measure for sparse covariates under convex poltypes]
\begin{lemma}\label{mu_lemma}
Let $\{\mb x_i\}_{i=1}^n$ be independent copies of a random vector $\mb x$ satisfying Assumption \ref{assum:subg}. Then there exists an absolute constant $C$ for which it holds with probability at least $1-\delta$ that 
\begin{equation} \label{eq:con_empirical_measure}
    \sup_{\substack{\mathcal{C} \in \mathcal{P}_{k,d,s} }} \left|\frac{1}{n}\sum_{i=1}^{n}\bbone_{\{\mb x_i \in \mathcal{C}\}} - \mathbb{P}(\mb x \in \mathcal{C}) \right|\leq \epsilon,
\end{equation}
if
\begin{equation}\label{n:Empirical_Measure}
    n \geq C \epsilon^{-2}\left[sk\log\left(\frac{n}{s}\right)+s\log\left(\frac{d}{s}\right)+\log\left(\frac{1}{\delta} \right)\right].
\end{equation}
\end{lemma}
\begin{proof}[Proof of Lemma~\ref{mu_lemma}]
The proof of this lemma is obtained by a direct application of the union bound to the statement in \citep[Corollary 6.7]{kim2023maxaffine}. 
The result of \citep[Corollary 6.7]{kim2023maxaffine} provides the concentration inequality in \eqref{eq:con_empirical_measure} when the supremum is over $\mathcal{P}_{k,d}(\mathcal{U})$ for a fixed $\mathcal{U}$ instead of the union over $\mathcal{U} \in \mathcal{Z}_s$. 
%Since it considers the supremum over $\mathcal{P}_{k,d}(\mathcal{U})$ for a fixed $\mathcal{U}$, 
To apply the union bound argument, we inflate the probability of error by the factor $|\mathcal{Z}_s| = $ $\binom{d}{s}$ $\leq \left(\frac{ed}{s}\right)^s$.    
\end{proof}

{\color{black}\begin{lemma}\label{lemm:poly_int}
    Fix $\delta \in (0,e^{-1})$. Let $\rho$ be defined as in \eqref{eq:choice_rho1} for some $R>0$. Let $\{\mathcal{C}_j(\mb \theta)\}_{j=1}^k$ and $\{\mathcal{C}_j(\mb \theta^\star)\}_{j=1}^k$ be respectively defined by $\mb \theta$ and $\mb \theta^\star \in \mathbb R^{dk}$ according to \eqref{eq:def_calCjstar}. Let $\mathcal{U} \cup \{d+1\}$ denote the joint support of $\{\mb \theta_j\}_{j=1}^k$. Let $\{\mb x_i\}_{i=1}^n$ be independent copies of a random vector $\mb x$ satisfying Assumptions \ref{assum:subg} and \ref{assum:anti}. Then with probability at least $1-\delta$ we have
    \begin{equation}
        \frac{1}{n}\sum_{i=1}^n \bbone_{\{\mb x_i \in \mathcal{C}_j(\mb \theta)\cap\mathcal{C}_{j'}(\mb \theta^\star)\}} \leq \frac{\pi^{2(1+\zeta^{-1})}_{\min}}{k^2},
    \end{equation}
    if 
    \begin{equation} \label{eq:n_int}
        n\geq C \left[s\log\left(\frac{d\vee n}{s}\right)+\log\left(\frac{k}{\delta}\right)\right]k^4 \pi_{\min}^{-4(1+\zeta^{-1})},
    \end{equation}
    for all $j'\neq j \in [k]$, $\mb \theta \in \mathcal{N}(\mb \theta^\star, \Delta\rho)$ and $\mathcal{U} \in \mathcal{Z}_s$.
\end{lemma}
The proof of this lemma is deferred to the Appendix~\ref{sec:proof:lem:model_pis}.}
To bound the closeness of the empirical measure in Lemma \ref{mu_lemma} to the expectation with ground truth model parameters, we present the following lemma.
%[Initialization of Model Probabilities Under polytopes]
\begin{lemma}\label{lem:model_pis}
    Suppose that $\mb x \in \mathbb R^d$ satisfies Assumptions \ref{assum:subg} and \ref{assum:anti}. 
    Let $\{\mathcal{C}_j(\mb \theta)\}_{j=1}^k$ and $\{\mathcal{C}_j(\mb \theta^\star)\}_{j=1}^k$ be respectively defined by $\mb \theta$ and $\mb \theta^\star \in \mathbb R^{dk}$ according to \eqref{eq:def_calCjstar}. 
    Let $\mathcal{N}(\mb \theta^\star,\Delta\rho)$ be defined with $R>0$ as in \eqref{eq:defnbr} and $\varrho \triangleq CR^{2\zeta}k^{-2(1+\zeta^{-1})}$. If $\mb \theta \in \mathcal{N}(\mb \theta^\star,\Delta\rho)$, then we have    \begin{equation} \label{eq:pibounds}
        (1-\varrho)  \leq \frac{\mathbb P (\mb x \in \mathcal{C}_j(\mb \theta))}{\mathbb P (\mb x \in \mathcal{C}_j(\mb \theta^\star))} \leq \left(\frac{1-\varrho}{1-2\varrho}\right), \quad \forall j \in [k]. 
    \end{equation}
\end{lemma}
The proof of this lemma is deferred to the Appendix~\ref{sec:proof:lem:model_pis}. The next lemma shows a tail bound on the sparsity-constrained operator norm of partial sums of the centered outer products of covariates. 
%[Sparsity-restricted isometry property under polytopes]
\begin{lemma}
\label{lem:bnd_singular} 
Let $\{\mb x_i\}_{i=1}^n$ be independent copies of a random vector $\mb x$ which satisfies Assumptions \ref{assum:subg} and \ref{assum:anti}. Then there exists an absolute constant $C$ for which it holds with probability at least $1-\delta$ that 
\begin{equation}
\label{eq:bnd_concent}
\sup\limits_{\substack{\mathcal{C} \in \mathcal{P}_{k,d,s} \\ \mathcal{U} \in \mathcal{Z}_s}}\left\|\mb{\tilde \Pi}_{\mathcal{U}}\left(\frac{1}{n}\sum_{i=1}^{n}\bbone_{\{\mb x_i\in \mathcal{C} \}}\left([\mb x_i;1][\mb x_i;1]^\T-\mb I_{d+1}\right)\right)\mb{\tilde \Pi}_{\mathcal{U}}\right\|\leq \epsilon 
\end{equation}
if
\begin{equation}\label{n:bnd_singular}
n \geq C (\eta \vee 1)^4 \epsilon^{-2}\left[ sk\log\left(\frac{n}{s}\right) +s\log\left(\frac{d}{s}\right)+\log \left(\frac{1}{\delta}\right) \right].
\end{equation}
\end{lemma}
The proof of this lemma is deferred to Appendix~\ref{sec:proof:lem:bnd_singular} 
The statement of the next lemma will require the notation
\begin{equation} \label{eq:defvjj}
    v_{jj'} \triangleq \mb \theta_j -\mb \theta_{j'}, \quad v^\star_{jj'} \triangleq \mb \theta^\star_j -\mb \theta^\star_{j'}, \quad \forall j'\neq j \in [k].
\end{equation}
\begin{lemma}\label{lem:model_difference}
 Let $\{\mathcal{C}_j(\mb \theta)\}_{j=1}^k$ and $\{\mathcal{C}_j(\mb \theta^\star)\}_{j=1}^k$ be respectively defined by $\mb \theta$ and $\mb \theta^\star \in \mathbb R^{dk}$ according to \eqref{eq:def_calCjstar}. Assume that $\mb x \in \mathbb R^d$ satisfies Assumptions \ref{assum:anti} and \ref{assum:subg}. Fix $\delta \in (0,1/e)$ and $R>0$. Assume $\mb \theta \in \mathcal{N}(\mb \theta^\star,\Delta\rho)$ as defined in \eqref{eq:defnbr} with $\rho$ as defined in \eqref{eq:choice_rho1}. Then there exists an absolute constant $C$ for which it holds with probability at least $1-\delta$ that
\begin{equation} 
    \frac{1}{n} \sum_{j'\neq j}\sum_{\substack{i=1 }}^{n}     \bbone_{\{\mb x_i\in \mathcal{C}_j(\mb \theta) \cap \mathcal{C}_{j'}(\mb \theta^\star) \}} \langle\mb \xi_{i},\mb v_{jj'}^\star\rangle^2 \leq \frac{2}{5\gamma k}\left(\frac{\pi_{\min}}{16}\right)^{1+\zeta^{-1}} \sum_{j'\neq j} \|\mb v_{jj'} - \mb v_{jj'}^\star\|_2^2
\end{equation}
if 
\begin{equation} \label{n:difference_lemma}
    n\geq C k^4 \pi_{\mathrm{min}}^{-4(1+\zeta^{-1})}\left[s\log\left(\frac{n\vee d}{s}\right)+\log\left(\frac{k}{\delta}\right)\right].
\end{equation}
\end{lemma}
\begin{proof}
    The proof of this lemma is a direct application of the union bound by inflating the probability of error by $|\mathcal{Z}_s| = $$\binom{d}{s}$ $\leq \left(\frac{ed}{s}\right)^s$ to the statement in \citep[Lemma 7.7]{kim2023maxaffine}.    
\end{proof}
The next lemma provides an upper bound of a noise-related term that will appear in the proof of the main theorem.
\begin{lemma}\label{lem:noise_grad_bound}
    Let $\{\mb x_i\}_{i=1}^n$ be independent copies of a random vector $\mb x$ satisfying Assumptions \ref{assum:subg} and \ref{assum:anti}. 
    Let $\{z_i\}_{i=1}^n$ be i.i.d. sub-Gaussian random variables with zero mean and variance $\sigma^2_z$, independent of everything else. Then there exists an absolute constant $C$ for which it holds with probability at least $1-\delta$ that
    \begin{equation}\label{n:noise_lem}
    \sup\limits_{\substack{\mathcal{C} \in \mathcal{P}_{k,d,s}\\\mathcal{U} \in \mathcal{Z}_s}}\norm{\frac{1}{n} \sum_{i=1}^n \bbone_{\{\mb x_i\in \mathcal{C}  \}} z_i\mb{\tilde \Pi}_{\mathcal{U}}[\mb x_{i};1]}_2 \leq C\sigma_z\sqrt{\frac{sk\log\left(n/s\right)+s\log\left(d/s\right)+\log\left({1}/{\delta}\right)}{n}}.
    \end{equation}
\end{lemma}
\begin{proof}
    The proof of the lemma follows directly from applying the union bound to the statement in \citep[lemma 8.1, Eq. 46]{kim2023maxaffine}. Since that statement only considers the supremum over non-sparse polytopes, whereas we consider the supremum over jointly sparse polytopes and sparse vector supports, we must inflate the probability of error $\delta$ by $|\mathcal{Z}_s|^2 = $$\binom{d}{s}^2$ $\leq \left(\frac{ed}{s}\right)^{2s}$.
\end{proof}

{\color{black}The next lemma provides a tail bound on the worst-case eigenvalue of the sum of covariate outer product with bounded cardinality.
\begin{lemma} \label{lemm:restriced_rip}
    Let $\delta\in (0,e^{-1})$ and $\alpha \in (0,1)$. Let $\{\mb x_i\}_{i=1}^n$ be independent copies of a random vector $\mb x$ that satisfies Assumption $\ref{assum:subg}$. Then, with probability at least $1-\delta$ we have
    \begin{equation}\label{eq:rip_bound}
        \sup_{\substack{\mathcal{I}:|\mathcal{I}|\leq\alpha n\\ \mathcal{U} \in\mathcal{Z}_s}} \lambda_1 \left[\tilde{\mb \Pi}_{\mathcal{U}}\left(\frac{1}{n}\sum_{i\in\mathcal{I}}[\mb x_i;1][\mb x_i;1]^\top\right)\tilde{\mb \Pi}_{\mathcal{U}}^\top\right] \leq C (\eta^2 \vee 1) \sqrt{\alpha},
    \end{equation}
    if
    \begin{equation}\label{eq:rip_bound_n}
        n\geq \alpha^{-1}\left[s\log(d/s)+\log(1/\delta)\right].
     \end{equation}
\end{lemma}
\begin{proof}
    For fixed $\mathcal{U} \in \mathcal{Z}_s$, \eqref{eq:rip_bound} follows directly from \citep[Theorem 5.7]{tan2019phase} if $n\geq d \vee \alpha^{-1}\log(1/\delta)$. Using the union bound and inflating the probability of error by $|\mathcal{Z}_s| = $$\binom{d}{s}$ $\leq \left(\frac{ed}{s}\right)^{s}$ completes the proof.
\end{proof}}

\section{Auxiliary Lemmas for Theorem \ref{theo:init}} \label{sec:lemmas_init}
The first lemma aims to show that the empirical moment matrix $\hat{\mb M}$, defined in Algorithm \ref{algo:sPCA}, is not rank deficient. In other words, its first $k$ dominant eigenvectors are a basis for the span of $\{\mb \theta_j^\star\}_{j=1}^k$. To state this lemma, we will make use of the definitions in \eqref{init_lemma_def} and \eqref{eq:varsigma}. The first lemma stated next is borrowed from \citep{ghosh2021max}.
\begin{lemma}\hspace{-0.1cm}\citep[Lemma 7]{ghosh2021max} \label{lem:Mrank}
Let $\mb x \sim \mathrm{Normal} (\mb 0,\mb I_d) $ , and $y$ be defined from $\mb x$ according to \eqref{eq:def_max_affine_model_xi}. Also, assume that $k\leq d$.
Then we have that the combination of the first and second moments satisfies
    \begin{align}
        & \mb M \succeq \mb 0_d, \quad \mathrm{rank}(\mb M) = k, \nonumber  \quad \lambda_{k}(\mb M)\geq \delta_{\mathrm{gap}},
    \end{align}
    for some numerical constant $\delta_{\mathrm{gap}}>0$ independent of the ambient dimension $d$.
\end{lemma}
The next lemma states the concentration of the empirical moments around their expectations.
\begin{lemma}\label{lem:empirical_init_concentration}
Let $\mb x \sim \mathrm{Normal} (\mb 0,\mb I_d) $ , and $y$ be defined from $\mb x$ according to \eqref{eq:def_max_affine_model_xi}. Then there exists absolute constants $C_1,C_2>0$ such that 
    \begin{align} 
    &\mathbb{P} \left( \norm{\mb m_1\mb -\mb{\hat m_1}}_\infty \geq C_1(\varsigma + \sigma_z)\frac{\log(n d)}{\sqrt{n}}\right) \leq n^{-11}, \nonumber \\
     &\mathbb{P} \left( \norm{\mb M_2-\mb{\hat M}_2}_\infty \geq C_2(  \varsigma+ \sigma_z)\log(nd)\left(\frac{1}{\sqrt{n}} \vee \frac{\sqrt[6]{\log(nd)}}{\sqrt[3]{n^2}}\right)\right) \leq n^{-12}.
\end{align}
\end{lemma}
A direct consequence of these two lemmas is the concentration of moment matrix $\hat{\mb M}$ around its expectation $\mb M$ stated in Lemma \ref{lem:Mtotal}.
\begin{lemma}[ Maximum of Sub-exponential Random Variables]
    \label{lem:max_subexpo}
    Let $\{X_j\}_{j=1}^k$ have $\norm{X_j}_{\psi_1} < \infty$ for all $j \in [k]$, then we have that 
    \[
    \norm{\max_{j \in [k]} |X_j|}_{\psi_1}\leq \log_2 (2k) \max_{j \in [k]} \norm{X_j}_{\psi_1}.
    \]
\end{lemma}
\begin{proof}
    The proof of this lemma follows from 
    \begin{align*}
       \mathbb{P}(\max_{j \in [k]} |X_j|>t) & \leq \sum_{j=1}^k\mathbb{P}(|X_j|>t) \leq 1\wedge\sum_{j=1}^k  2\exp\left\{{\frac{-t}{C\norm{X_j}}_{\psi_1}} \right\}, \\
       & \leq 1\wedge 2k\exp\left\{\frac{-t}{C\underset{j\in [k]}{\max} \norm{X_j}_{\psi_1}}\right\} = 1\wedge \exp \left\{\frac{-t}{C\underset{j\in [k]}{\max} \norm{X_j}_{\psi_1}} + \log 2k \right\}
       ,\\ & \leq 1\wedge 2\exp\left\{\frac{-t }{C\log _2(2k)\underset{j\in [k]}{\max} \norm{X_j}_{\psi_1}}\right\} ,
    \end{align*}
    for all $t>0$ where the second inequality follows from the tail probability of any sub-exponential random variable \citep[Proposition 2.7.1]{vershynin2018high}. Again, comparing this tail probability with the last inequality yields the assertion in this lemma. A simple graphical example is provided in Figure \ref{fig:maxsubexpo}, assuming that $C\underset{j\in [k]}{\max} \norm{X_j}_{\psi_1} =1$, where we compare the last two tail probabilities in the statement of the proof. 
    \end{proof}
\begin{figure}
        \centering
        \includegraphics[scale=0.8]{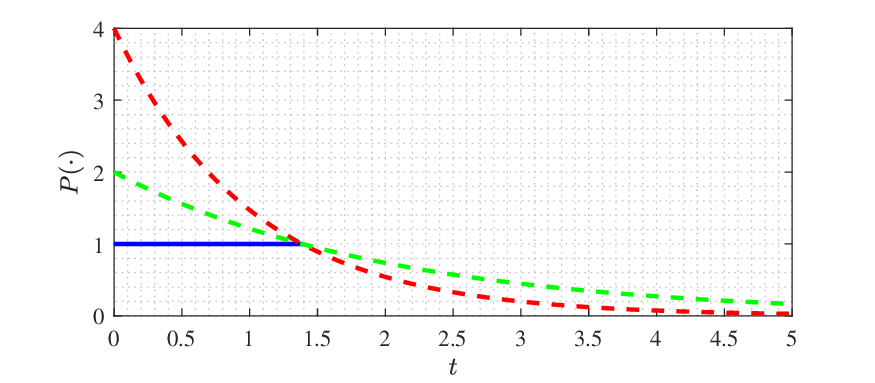} 
\caption{ Comparison of tail probabilities as a function of the threshold $t>0$ showing the trivial bound (blue), the tail $2 \exp\{\frac{-t}{\log_2(2k)}\}$ (green), and $ \exp\{-t+\log(2k)\}$ (red) with $k=2$.
}
\label{fig:maxsubexpo}
\end{figure}

\section{Proof of Lemmas} \label{App:Proofs}
\subsection{Proof of Lemma \ref{lem:Mtotal}}
    The proof of this lemma follows from
    \begin{align}
    \norm{\mb{\hat M} - \mb{  M}}_\infty&\leq \norm{\mb {\hat M}_2 - \mb M_2}_\infty+ \norm{\mb{\hat m}_1\mb{\hat m}_1^\T - \mb { m}_1\mb { m}_1^\T}_\infty  \nonumber \\
    & \leq \norm{\mb {\hat M}_2 - \mb M_2}_\infty + \norm{\mb{\hat m}_1 - \mb { m}_1}_\infty^2 +  2\norm{\mb{\hat m}_1 - \mb { m}_1}_\infty \cdot \norm{\mb m_1 }_\infty \nonumber \\
    & \leq  C(  \varsigma^2+ \sigma_z^2)\left(\frac{\log^2(nd)}{n} \vee \frac{\log(nd)}{\sqrt{n}}\right),
\end{align}
where the last inequality follows from Lemma \ref{lem:empirical_init_concentration}.

\subsection{Proof of Lemma \ref{lem:empirical_init_concentration}}
In what follows, we provide a proof for a stronger result that holds under a weaker assumption that $\mb x $ is a sub-Gaussian random vector symmetric about the origin, i.e. $\bm x \overset{\text{d}}{\sim}-\bm x$.
 Recall that a sub-Gaussian random vector $\mb x \in \mathbb R^d$ satisfies
\begin{equation}
     \mathbb P \left( \norm{\mb x}_\infty \geq t\right) \leq 2d \exp\left\{-\frac{c t^2}{\norm{\mb x}^2_{\psi_2}}\right\},
\end{equation}
for some absolute constant $c>0$. Next, define two events as
\begin{align} \label{eq:Trunc_events}
    \mathcal{E}_{i} = \left\{ \norm{\mb x_i}_\infty\leq 5 \sqrt{\log (end)}\right\}, \quad 
    \mathcal{F}_i = \left\{  z_i\leq 5\sigma_z \sqrt{\log(n)} \right\}, \quad \forall i\in [n],
\end{align}
where each event holds with probability at least $1-C_1n^{-12}$ for some absolute constant $C_1>0$. Next, for all $i \in [n]$, define the truncated covariate vector $\mb{\tilde x}_i = \mb x_i\bbone_{\mathcal{E}_{i}}$. With this definition we can begin bounding the moment concentrations. 
We notice that $\mb{\tilde x}_i - \mb x_i= \mb{\tilde x}_i- (\bbone_{\{\mathcal{E}_{i}\}}+ \bbone_{\{\mathcal{E}^\mathsf{c}_{i}\}})\mb x_i =  -\bbone_{\{\mathcal{E}^\mathsf{c}_{i}\}}\mb x_i$, and $\bbone_{\mathcal{E}_i}\max\limits_{j \in [k]}\langle \mb \theta^\star_j,[\tilde{\mb x}_i,1] \rangle= \bbone_{\mathcal{E}_i}\max\limits_{j \in [k]}\langle \mb \theta^\star_j,[\mb x_i,1] \rangle$ for all $i\in [n]$. Therefore, the difference of the first moments can be written as
\begin{align}\label{eq:expand_m1}
    \mb{\hat m}_1 -\mb m_1 &= \frac{1}{n} \sum\limits_{i=1}^n \left(\mb x_i\max_{j \in [k]}\langle \mb \theta^\star_j, [\mb x;1]\rangle  - \mathbb E\left[\mb x \max_{j \in [k]}\langle \mb \theta^\star_j, [\mb x;1] \rangle\right] +  z_i \mb x_i \right) \nonumber \\
    & =\frac{1}{n}\sum\limits_{i=1}^n (\mb x_i\max_{j \in [k]}\langle \mb \theta^\star_j, [\mb x;1]\rangle  -  \mb{\tilde x}_i \max_{j \in [k]}\langle \mb \theta^\star_j, [\mb{\tilde x}_i;1] \rangle \nonumber \\
    &\quad+ \mathbb{E}\left[\mb{\tilde x}_i \max_{j \in [k]}\langle \mb \theta^\star_j, [\mb{\tilde x}_i;1] \rangle\right] - \mathbb{E}\left[\mb{ x}_i \max_{j \in [k]}\langle \mb \theta^\star_j, [\mb{ x}_i;1] \rangle\right])  \nonumber \\
    &\quad +\frac{1}{n}\sum\limits_{i=1}^n\left( \mb{\tilde x}_i \max_{j \in [k]}\langle \mb \theta^\star_j, [\mb{\tilde x}_i;1] \rangle - \mathbb{E}\left[\mb{\tilde x}_i \max_{j \in [k]}\langle \mb \theta^\star_j, [\mb{\tilde x}_i;1] \rangle\right]\right)+ \frac{1}{n}\sum\limits_{i=1}^nz_i \mb x_i \nonumber \\
    & = \underbrace{ \frac{1}{n}\sum\limits_{i=1}^n \left(\bbone_{\mathcal{E}^\mathsf{c}_{i}}\mb{x}_i \max_{j \in [k]}\langle \mb \theta^\star_j, [\mb{ x}_i;1] \rangle-\mathbb{E}\left[\bbone_{\mathcal{E}^\mathsf{c}_{i}}\mb{x}_i \max_{j \in [k]}\langle \mb \theta^\star_j, [\mb{ x}_i;1] \rangle\right]\right)}_{\mathcal{M}_1}\nonumber \\
    &\quad + \underbrace{\frac{1}{n}\sum\limits_{i=1}^n\left( \mb{\tilde x}_i \max_{j \in [k]}\langle \mb \theta^\star_j, [\mb{\tilde x}_i;1] \rangle - \mathbb{E}\left[\mb{\tilde x}_i \max_{j \in [k]}\langle \mb \theta^\star_j, [\mb{\tilde x}_i;1] \rangle\right]\right)}_{\mathcal{M}_2} +\frac{1}{n}\sum\limits_{i=1}^nz_i \mb x_i
\end{align}
We now bound the max norm of the quantities in \eqref{eq:expand_m1}. First, we prove that each summand in $\mathcal{M}_1$ is sub-exponential. For this purpose, we write
\begin{align}\label{eq:M1_All_bound}
\norm{\bbone_{\mathcal{E}^\mathsf{c}_{i}}\mb{x}_i \max_{j \in [k]}\langle \mb \theta^\star_j, [\mb{ x}_i;1] \rangle}_{\psi_1} &\leq \norm{\bbone_{\mathcal{E}^\mathsf{c}_{i}}\mb{x}_i}_{\psi_2} \norm{\max_{j \in [k]}\langle \mb \theta^\star_j, [\mb{ x}_i;1] \rangle}_{\psi_2} \nonumber \\
& \overset{\mathrm{(i)}}{\leq} C \norm{\mb x_i}_{\psi_2}\cdot \sqrt{\log(k)} \max_{j \in [k]} \norm{\langle \mb \theta^\star_j, [\mb{ x}_i;1] \rangle}_{\psi_2} \nonumber \\
& \leq C \sqrt{\log(k)} \max_{j \in [k]}\left(\norm{\mb a_j^\star}_2 + |b_j^\star| \right) \overset{\mathrm{(ii)}}{\leq}C \varsigma\sqrt{\log(k)},
\end{align}
where (i) follows from Lemma \ref{lem:max_subexpo} and the fact that the indicator function is always dominated by $1$; and  (ii) follows by recalling the definition $\varsigma = \underset{j \in [k]}{\max}(\norm{\mb a^\star_j}_1+|b_j^\star|)$ from \eqref{eq:varsigma} and the fact the $\ell_1$ norm dominates the $\ell_2$ norm. Therefore, we can apply Bernstein's inequality for sub-exponential random variables along with the union bound over $d$ entries to get
\begin{equation}\label{eq:M1_ber}
    \mathbb{P} \left( \norm{\mathcal{M}_1}_\infty \geq C \epsilon \sqrt{\frac{\log(k)}{n}} \right) \leq de^{-\epsilon}.
\end{equation}
Choosing $\epsilon = 12\frac{\log(nd)}{\log(k)}$, we get 
\begin{equation}\label{eq:M1_final}
    \mathbb{P} \left( \norm{\mathcal{M}_1}_\infty \geq C  \frac{\log(nd)}{\sqrt{n}}\right) \leq n^{-12}.
\end{equation}
Next, we show that each summand in $\mathcal{M}_2$ is sub-Gaussian. For this purpose, we present a bound as
\begin{align} \label{eq:det_max_bound}
    \left| \max_{j \in [k]}\langle \mb \theta^\star_j, [\mb{\tilde x}_i;1] \rangle\right| &\leq \max_{j \in [k]}\left[  \left|\langle\mb a_j^\star,\mb {\tilde x_i} \rangle \right| + |b_j^\star|\right] \leq \max_{j \in [k]}\left[  \norm{\mb a_j^\star}_1\cdot\norm{\mb {\tilde x_i}}_\infty   + |b_j^\star|\right]\nonumber \\
    &\leq \max_{j \in [k]} \left[5\norm{\mb a^\star_j}_1\sqrt{\log(end)}+ |b_j^\star|\right] \leq 5\varsigma\sqrt{\log(end)},\quad \forall i \in [n],
\end{align}
using the definition of $\varsigma = \underset{j \in [k]}{\max}(\norm{\mb a^\star_j}_1+|b_j^\star|)$ from \eqref{eq:varsigma} and the upper bound on $\mb{\tilde x}_i$ by the truncation in \eqref{eq:Trunc_events}. Therefore, for all $i \in [n]$, we have that 
\begin{align}
    \norm{ \mb{\tilde x}_i \max_{j \in [k]}\langle \mb \theta^\star_j, [\mb{\tilde x}_i;1] \rangle}_{\psi_2}  \overset{}{\leq}C \norm{{\mb x}_i}_{\psi_2}\left| \max_{j \in [k]}\langle \mb \theta^\star_j, [\mb{\tilde x}_i;1] \rangle\right| \leq C \varsigma\sqrt{\log (end)},
\end{align}
where the last inequality follows from \eqref{eq:det_max_bound}.
Therefore, we have by Hoeffding's inequality and the union bound over $d$ entries that
\begin{equation}
\mathbb{P}\left(\left| \norm{\mathcal{M}_{2}}_\infty\right|\geq \epsilon\right) \leq 2 d\exp \left\{\frac{-n\epsilon^2}{2C \varsigma^2\log(nd)} \right\}.
\end{equation} 
The next statement follows by the union bound where we inflate the probability of error by $d $. Therefore, with $\epsilon= 5\varsigma\log(nd)/\sqrt{n}$, we have that
\begin{equation}\label{eq:M2_final}
    \mathbb{P}\left( \norm{\mathcal{M}_{2}}_\infty \geq C\varsigma \frac{\log(nd)}{\sqrt{n}}\right) \leq n^{-12}.
\end{equation}
{\color{black}
Let $\{r_i\}_{i=1}^n$ be independent copies of a Rademacher random variable. Notice that $\mb x_i \overset{\text{d}} {\sim}\mb x_i r_i$ for every $i\in [n]$ by the symmetry of the Gaussian distribution. Therefore under the event $\bigcap\limits_{i=1}^n \mathcal{F}_{i}$, it holds that 
\begin{align*}
\mathbb{P}\left( \norm{\frac{1}{n}\sum_{i=1}^n \frac{z_i}{5 \sigma_z \sqrt{\log n}} \mb x_i}_\infty \geq t \right)  &= \mathbb{P}\left( \norm{\frac{1}{n}\sum_{i=1}^n \frac{z_i}{5 \sigma_z \sqrt{\log n}} \mb x_i r_i}_\infty \geq t \right) 
\nonumber \\
&\leq 2 \mathbb{P}\left( \norm{\frac{1}{n}\sum_{i=1}^n \mb x_i r_i}_\infty \geq t \right)   \nonumber \\
&=  2\mathbb{P}\left( \norm{\frac{1}{n}\sum_{i=1}^n \mb x_i }_\infty \geq t \right),
\end{align*}
where the first inequality follows from the contraction principle in \citep[Theorem 4.4]{ledoux2013probability}.
}
Therefore, under the event $\bigcap\limits_{i=1}^n \mathcal{F}_{i}$, it holds with probability at least $1-2n^{-11}$ that  
\begin{equation}\label{eq:M4_start}
    \norm{\frac{1}{n}\sum\limits_{i=1}^n z_i\mb x_i}_\infty \leq C \sigma_z \sqrt{\log(n)}\norm{\frac{1}{n}\sum\limits_{i=1}^n \mb x_i}_\infty.
\end{equation}
The term on the right-hand side can be bounded by Hoeffding's inequality and the union bound over $d$ entries as
\begin{equation*}
    \mathbb P \left(\norm{\frac{1}{n}\sum\limits_{i=1}^n [\mb x_i]_l }_\infty\geq \epsilon \right) \leq 2d\exp\left\{\frac{- n \epsilon^2}{C}\right\},
\end{equation*}
for some index $l \in [d]$. Finally, setting $\epsilon= (12\log(nd)/  n)^{1/2}$, we get
\begin{equation} \label{eq:x_infty_bound}
    \mathbb P \left(\norm{\frac{1}{n}\sum\limits_{i=1}^n \mb x_i}_\infty\geq C\sqrt{\frac{12\log(nd)}{ n}} \right) \leq 2n^{-12}.
\end{equation}
Combining \eqref{eq:M4_start} and \eqref{eq:x_infty_bound}, we get 
\begin{equation}\label{eq:M4_final}
    \mathbb P \left(\norm{\frac{1}{n}\sum\limits_{i=1}^n z_i\mb x_i}_\infty\geq C\sigma_z\frac{\log(n d)}{ \sqrt{n}} \right) \leq 2n^{-11}.
\end{equation}
Finally, we combine \eqref{eq:M1_final},\eqref{eq:M2_final},  and \eqref{eq:M4_final} to get
\begin{equation}
    \norm{\mb m_1 -\mb{\hat m}_1}_\infty \leq  \norm{\mathcal{M}_1}_\infty +\norm{\mathcal{M}_2}_\infty+\norm{\frac{1}{n}\sum\limits_{i=1}^n z_i \mb x_i}_\infty \leq C (\varsigma + \sigma_z)\frac{\log(n d)}{ \sqrt{n}},
\end{equation}
with probability at least $1-n^{-11}$. We proceed in a similar fashion and decompose the difference of the second moments as
\begin{align}\label{eq:M2diff_decomp}
    \hat{\mb M}_2 - \mb M_2 &= \frac{1}{n}\sum_{i=1}^n \max_{j \in [k]}\langle \mb \theta^\star_j, [\mb{ x}_i;1] \rangle (\mb x_i \mb x_i^\T -\mb I_d) -\mathbb E \left[ \max_{j \in [k]}\langle \mb \theta^\star_j, [\mb{ x}_i;1] \rangle (\mb x_i \mb x_i^\T -\mb I_d)\right]\nonumber\\
    &\quad+\frac{1}{n} \sum_{i=1}^n z_i (\mb x_i \mb x_i^\T -\mb I_d) \nonumber \\
    & = \underbrace{\frac{1}{n}\sum_{i=1}^n \left(\max_{j \in [k]}\langle \mb \theta^\star_j, [\mb{ x}_i;1] \rangle (\mb x_i \mb x_i^\T ) -\mathbb E \left[ \max_{j \in [k]}\langle \mb \theta^\star_j, [\mb{ x}_i;1] \rangle (\mb x_i \mb x_i^\T )\right]\right)}_{\mathcal{Q}_1} \nonumber \\
    & \quad - \underbrace{\mb I_d \cdot \frac{1}{n}\sum_{i=1}^n \left(\max_{j \in [k]}\langle \mb \theta^\star_j, [\tilde{\mb x}_i;1] \rangle  -\mathbb E \left[ \max_{j \in [k]}\langle \mb \theta^\star_j, [\tilde{\mb x}_i;1] \rangle \right]\right)}_{\mathcal{Q}_2}+\frac{1}{n} \sum_{i=1}^n z_i (\mb x_i \mb x_i^\T -\mb I_d) 
\end{align}
To proceed, we define a generalization of sub-exponential and sub-Gaussian random variables as sub-Weibull random variables \citep{bakhshizadeh2023algebra}. 
\begin{definition}
    A random variable X is sub-Weibull of order $\alpha>0$, denoted as SW($\alpha$), if its sub-Weibull norm satisfies
    \begin{equation}
        \norm{X}_{\psi_{\alpha^{-1}}}  = \inf \left\{t: \exp \left(\left|\frac{X}{t}\right|^{\alpha^{-1}}\right) \leq 2\right\} < \infty.
    \end{equation} 
\end{definition}
Note that sub-Gaussian random variables are SW($1/2$), and sub-exponential random variables are SW($1$). Next, we provide a useful lemma on algebra of sub-Weibull norms.
\begin{lemma}
    Let X be SW($\alpha$) and Y be SW($\theta$) then we have that
     \begin{equation}
         \norm{XY}_{\psi_{(\alpha+\theta)^{-1}}} \leq  \norm{X}_{\psi_{\alpha^{-1}}} \cdot \norm{Y}_{\psi_{\theta^{-1}}}.
     \end{equation}
\end{lemma}
\begin{proof}
 By the sub-Weibull assumptions we have that 
 \begin{equation}
     \mathbb E\left[\exp\left(\frac{X^{\alpha^{-1}}}{\norm{X}_{\psi_{\alpha^-1}}} \right) \right] \leq 2, \quad \mathbb E\left[\exp\left(\frac{Y^{\theta^{-1}}}{\norm{Y}_{\psi_{\theta^-1}}} \right) \right] \leq 2.
 \end{equation}
 Now we recall Young's inequality as 
 \begin{equation}
     ab \leq \frac{a^p}{p} + \frac{b^q}{q},\quad  \forall a,b \in \mathbb R^+, \quad\forall p,q>1, \quad p^{-1}+ q^{-1}=1.
 \end{equation}
 We observe that one valid pair for Young's inequality is $(p,q) = \left( \frac{\alpha+\theta}{\alpha},\frac{\alpha+\theta}{\theta} \right)$. Therefore, we have that 
 \begin{align}
     &\mathbb E \left[\exp \left\{\frac{(XY)^{(\alpha+\theta)^{-1}}}{\norm{X}_{\psi_{\alpha^-1}}\norm{Y}_{\psi_{\theta^-1}}} \right\}\right] \nonumber \\ 
     &\quad \leq \mathbb E \left[\exp \left\{\frac{\alpha}{\alpha+\theta}\cdot\frac{X^{\alpha^{-1}} }{\norm{X}_{\psi_{\alpha^-1}}\norm{Y}_{\psi_{\theta^-1}}}\right\}\cdot \exp \left\{\frac{\theta}{\alpha+\theta}\cdot\frac{Y^{\theta^{-1}} }{\norm{X}_{\psi_{\alpha^-1}}\norm{Y}_{\psi_{\theta^-1}}}\right\}\right] \nonumber \\
     &\quad \leq \mathbb E \left[\frac{\alpha}{\alpha+\theta}\cdot\exp \left\{\frac{X^{\alpha^{-1}} }{\norm{X}_{\psi_{\alpha^-1}}\norm{Y}_{\psi_{\theta^-1}}}\right\}+ \frac{\theta}{\alpha+\theta} \cdot\exp \left\{\frac{Y^{\theta^{-1}} }{\norm{X}_{\psi_{\alpha^-1}}\norm{Y}_{\psi_{\theta^-1}}}\right\}\right] \nonumber \\
     %&\quad \leq \mathbb E \left[\frac{\alpha}{\alpha+\theta}\cdot\exp \left\{\frac{X^{\alpha^{-1}} }{\norm{X}_{\psi_{\alpha^-1}}\left(\norm{Y}_{\psi_{\theta^-1}}\vee 1\right)}\right\}+ \frac{\theta}{\alpha+\theta} \cdot\exp \left\{\frac{Y^{\theta^{-1}} }{\left(\norm{X}_{\psi_{\alpha^-1}}\vee 1\right)\norm{Y}_{\psi_{\theta^-1}}}\right\}\right] \nonumber \\
     &\quad \leq \frac{2\alpha}{\alpha+\theta} + \frac{2\theta}{\alpha+\theta}=2,
 \end{align}
where the first two inequalities follow from Young's inequality which concludes the proof.
\end{proof}
We next present a very useful corollary for the case of product of sub-exponential and sub-Gaussian random variables.
\begin{corollary} \label{Col:subweb}
    Let $\{X_i\}_{i=1}^{n_x}$ and $\{Y_i\}_{i=1}^{n_y}$ be collections of sub-Gaussian and sub-exponential random variables respectively, then we have
    \begin{equation}
        \norm{\prod\limits_{i=1}^{n_x} X_i \prod\limits_{j=1}^{n_y} Y_j }_{\psi_{({n_x}/{2} +n_y)^{-1}}} \leq \prod\limits_{i=1}^{n_x}\norm{X}_{\psi_2}\prod\limits_{j=1}^{n_y}\norm{Y}_{\psi_1}
    \end{equation}
\end{corollary}
Next, we show that each entry of $\mathcal{Q}_1$ is SW($3/2$). Let $l,m \in [d]$, then for all $i \in [n]$ we have
\begin{align}
   \norm{\max_{j \in [k]}\langle \mb \theta^\star_j, [\mb{ x}_i;1] \rangle ([\mb{ x}_i]_l[\mb{ x}_i]_m )}_{\psi_{2/3}} &\leq \norm{\max_{j \in [k]}\langle \mb \theta^\star_j, [\mb{ x}_i;1] \rangle}_{\psi_2}\norm{[\mb{ x}_i]_l}_{\psi_2}\norm{[\mb{ x}_i]_m}_{\psi_2}  \nonumber \\
   &\leq  C \sqrt{\log(d)} \max_{j \in [k]}\left(\norm{\mb a_j^\star}_2 + |b_j^\star| \right) \norm{\mb x_i}^2_{\psi_2} \leq C\varsigma \sqrt{\log(d)},
\end{align}
where the first inequality follows by Corollary \ref{Col:subweb}. We can now invoke the generalized version of Hoefdding's and Bernstein's inequality for sub-Weibull random variables. Using \citep[Proposition 3]{zhang2022sharper} with the union bound over $d^2$ entries, we have that
\begin{equation}
    \mathbb P \left(\norm{\mathcal{Q}_1}_\infty \geq C\varsigma \sqrt{\log(d)} \left(\sqrt{\frac{\epsilon}{n}} \vee \sqrt[3/2]{\frac{\epsilon}{n}}\right)\right) \leq d^2e^{-\epsilon},
\end{equation}
since the entries of $\mathcal{Q}_1$ are SW($3/2$). Choosing $\epsilon = 12 \log(nd)$ we get
\begin{equation}\label{M2:Q1}
    \mathbb P \left(\norm{\mathcal{Q}_1}_\infty  \geq C \varsigma \sqrt{\log(d)}\left(\sqrt{\frac{\log(nd)}{n}} \vee \sqrt[3/2]{\frac{\log(nd)}{n}}\right)\right) \leq n^{-12}.
\end{equation}
Recalling \eqref{eq:det_max_bound}, we have that the diagonal elements of $\mathcal{Q}_2$ are sub-Gaussian. Therefore, using Hoeffding's inequality and the union bound of the $d$ diagonal entries we have
\begin{equation}
    \mathbb{P}\left(\norm{\mathcal{Q}_{2}}_\infty\geq \epsilon\right) \leq 2 d\exp \left\{\frac{-n\epsilon^2}{2C \varsigma^2\log(nd)} \right\}.
\end{equation}
Choosing $\epsilon= 5\varsigma\log(nd)/\sqrt{n}$ we get 
\begin{equation} \label{M2:Q2}
     \mathbb{P}\left( \norm{\mathcal{Q}_{2}}_\infty \geq C\varsigma \frac{\log(nd)}{\sqrt{n}}\right) \leq n^{-12}.
\end{equation}
%The last term in \eqref{eq:M2diff_decomp} can be upper bounded as 
%\begin{equation}
%   \norm{ \frac{1}{n}\sum_{i=1}^n z_i (\mb x_i \mb x_i^\T-\mb I_d)}_\infty \leq \norm{\frac{1}{n}\sum_{i=1}^n z_i (\mb x_i \mb x_i^\T)}_\infty + \norm{\frac{1}{n}\sum_{i=1}^n z_i}_\infty
%\end{equation}
Next, for some $l,m \in [d]$, let the Kronecker delta $\delta_{lm} = 1$ only when $l=m$ and $0$ otherwise. Then we have for all $i \in [n]$ that
\begin{equation}
    \norm{z_i\left([\mb x_i]_l [\mb x_i]_m - \delta_{lm}\right)}_{\psi_{2/3}} \leq \norm{z_i}_{\psi_2} \norm{\mb x_i}^2_{\psi_2} \leq C \sigma_z.
\end{equation}
where the first inequality follows by Corollary \ref{Col:subweb}. Therefore, using \citep[Proposition 3]{zhang2022sharper} with the union bound over $d^2$ entries, we have that
\begin{equation}\label{M2:noise1}
    \mathbb P \left(\norm{\frac{1}{n}\sum_{i=1}^n z_i (\mb x_i \mb x_i^\T-\mb I_d)}_\infty \geq C \sigma_z \left(\sqrt{\frac{\log(nd)}{n}} \vee \sqrt[3/2]{\frac{\log(nd)}{n}}\right)\right) \leq n^{-12},
\end{equation}
Finally, we combine \eqref{M2:Q1}, \eqref{M2:Q2}, and \eqref{M2:noise1} to get the second assertion in the lemma which concludes the proof.
\subsection{Proof of Lemma \ref{lem:model_pis}} 
\label{sec:proof:lem:model_pis} 
We first derive a lower bound on
\begin{align}
\P\left(\mb x\in \mathcal{C}_{j}\cap \mathcal{C}_{j}^\star\right)
&= \P\left(\mb x\in \mathcal{C}_{j} | \mb x \in \mathcal{C}_{j}^\star\right) 
\cdot \P\left(\mb x \in \mathcal{C}_{j}^\star\right) \nonumber \\
&= \left(1 - \P\left(\mb x \not\in \mathcal{C}_{j} | \mb x \in \mathcal{C}_{j}^\star\right) \right)
\cdot \P\left(\mb x \in \mathcal{C}_{j}^\star\right).\label{eq:sbreslt1}
\end{align}
Then, by the construction of $\{\mathcal{C}^\star_j\}_{j=1}^k$ in \eqref{eq:def_calCjstar}, we have
\begin{align*}
& \P\left(\mb x \not\in \mathcal{C}_{j} | \mb x \in \mathcal{C}_{j}^\star\right) = \frac{\P(\mb x \not\in \mathcal{C}_{j}, \mb x \in \mathcal{C}_{j}^\star)}{\P(\mb x \in \mathcal{C}_{j}^\star)} \nonumber\\
&\leq \frac{1}{\pi_j^\star} \sum_{j' \neq j} \P\left(\langle [\mb x;\ 1], \mb \theta_{j'} \rangle \geq \langle [\mb x;\ 1], \mb \theta_{j} \rangle, \langle [\mb x;\ 1], \mb \theta_{j}^\star \rangle \geq \langle [\mb x;\ 1], \mb \theta_{j'}^\star \rangle\right) \nonumber\\
&\leq \frac{1}{\pi_j^\star} \sum_{j' \neq j} \P\left(\langle [\mb x;\ 1], \mb v_{j,j'} \rangle \langle [\mb x;\ 1], \mb v_{j,j'}^\star \rangle \leq 0\right) \\
&\leq \frac{1}{\pi_j^\star} \sum_{j' \neq j}  \P\left( 
\langle [\mb x; 1],\mb v_{j,j'}^\star\rangle^2\leq\langle [\mb x; 1],\mb v_{j,j'}-\mb v_{j,j'}^\star\rangle^2
\right),
\end{align*}
where the second inequality holds since $\mb v_{j,j'} = \mb \theta_j - \mb \theta_{j'}$ and $\mb v_{j,j'}^\star = \mb \theta_j^\star - \mb \theta_{j'}^\star$, and the last inequality follows from the fact that $ab \leq 0$ implies $|b| \leq |a-b|$ for $a,b \in \mathbb{R}$. 
Recall that $\mb \theta\in\mathcal{N}(\mb \theta^\star,\Delta\rho)$ implies $\|\mb v_{j,j'} - \mb v_{j,j'}^\star\|_2 \leq 2\rho \|(\mb v_{j,j'}^\star)_{1:d}\|_2$ due to \citep[Lemma~7.4]{kim2023maxaffine}. 
Furthermore, one can choose the absolute constant $R > 0$ in \eqref{eq:choice_rho1} sufficiently small (but independent of $k$ and $d$) so that $2\rho \leq 0.1$. 
Then it follows that 
\begin{align}
&\P(\mb x\not\in \mathcal{C}_{j}\,|\,\mb x\in \mathcal{C}_{j}^\star) \nonumber \\
&\overset{\mathrm{(i)}}{\leq} C \frac{k}{\pi_j^\star} \left(\frac{ \|\mb v_{j,j'}-\mb v_{j,j'}^\star\|_2^2}{\|(\mb v_{j,j'}^\star)_{1:d}
\|_2^2}\log\left(\frac{2\|(\mb v_{j,j'}^\star)_{1:d}\|_2}{\|\mb v_{j,j'}-\mb v_{j,j'}^\star\|_2}\right)\right)^{\zeta}\nonumber\\
&\overset{\mathrm{(ii)}}{\leq}C \frac{k}{\pi_j^\star} \left( (2\rho)^2 \log\left(\frac{1}{\rho}\right)\right)^{\zeta}\nonumber\\
&\overset{\mathrm{(iii)}}{\leq}C\frac{k}{\pi_j^\star} {\left(\frac{ R^2 \pi_{\min}^{2\zeta^{-1}(1+\zeta^{-1})}}{k^{2\zeta^{-1}}}\right)^{\zeta}} \overset{\mathrm{(iv)}}{\leq}C
\frac{R^{2\zeta} \pi_{\min}^{1+2\zeta^{-1}}}{k} \overset{\mathrm{(v)}}{\leq}C \frac{R^{2\zeta} }{k^{2(1+\zeta^{-1})}}\triangleq \varrho,\label{eq:subreslt0}
\end{align}  
where (i) follows from \citep[Lemma~7.5]{kim2023maxaffine}; (ii) holds since $a \log^{1/2}(2/a)$ is monotone increasing for $a \in (0,1]$; (iii) follows from the fact that $a\leq\frac{b}{2}\log^{-1/2}(1/b)$ implies $a\log^{1/2}(2/a)\leq b$ for $b\in(0,0.1]$; (iv) holds since $\pi^\star_j \geq \pi_{\mathrm{min}}$ for all $j \in [k]$ by the defintion of $\pi_{\mathrm{min}}$ in \eqref{eq:def_pimin_pimax}; and (v) holds since $\pi_{\mathrm{min}}\leq \frac{1}{k}$. Once again $R > 0$ can be made sufficiently small so that the right-hand side of \eqref{eq:subreslt0} defined as $\varrho$ is arbitrarily small. 
Then plugging in this upper bound by \eqref{eq:subreslt0} into \eqref{eq:sbreslt1} yields 
\begin{equation}
\label{eq:leftpibound}
\pi_j \geq \P(\mb x\in \mathcal{C}_{j}\cap \mathcal{C}_{j}^\star)\geq(1-\varrho)\cdot\pi_j^\star.
\end{equation} 
Similarly, and by symmetry, we can write
\begin{align}
\P(\mb x\not\in \mathcal{C}_{j}\,|\,\mb x\in \mathcal{C}_{j}^\star) &\leq\frac{R^{2\zeta} \pi_{\min}^{2(1+\zeta^{-1})}}{\pi_j k} \leq \frac{\pi_j^\star}{\pi_j}\cdot \frac{R^{2\zeta} \pi_{\min}^{1+2\zeta^{-1}}}{k} \leq \frac{\pi_j^\star}{\pi_j}\cdot \frac{R^{2\zeta} }{k^{2(1+\zeta^{-1})}} \leq\frac{\varrho}{1-\varrho},\label{eq:subreslt02}
\end{align}  
where the last inequality follows from the definition of $\varrho$ in \eqref{eq:subreslt0} and the bound derived in \eqref{eq:leftpibound}. Finally, we can write the bound
\begin{equation}
\label{eq:rightpibound}
\pi_j^\star \geq \P(\mb x\in \mathcal{C}_{j}\cap \mathcal{C}_{j}^\star)\geq\left(\frac{1-2\varrho}{1-\varrho}\right)\pi_j.
\end{equation} 
Combining \eqref{eq:leftpibound} and \eqref{eq:rightpibound} provides the assertion in \eqref{eq:pibounds} thus concluding the proof.

\subsection{Proof of Lemma \ref{lem:bnd_singular}}
\label{sec:proof:lem:bnd_singular}
\begin{comment}
The left-hand side of \eqref{eq:bnd_concent} is upper bounded as 
    \begin{align} \label{eq:center_first}
        &\left\|\mb{\tilde \Pi}_{\mathcal{U}}\left(\frac{1}{n}\sum_{i=1}^{n}\bbone_{\{\mb x_i\in \mathcal{C} \}}\left([\mb x_{i};1][\mb x_{i};1]^\T-\mb I_{d+1}\right)\right)\mb{\tilde \Pi}_{\mathcal{U}}\right\| \nonumber \\ 
      &\quad\leq  \left\|\mb{ \Pi}_{\mathcal{U}}\left(\frac{1}{n}\sum_{i=1}^{n}\bbone_{\{\mb x_i\in \mathcal{C} \}}\left(\mb x_{i}\mb x_{i}^\T-\mb I_{d}\right)\right)\mb{ \Pi}_{\mathcal{U}}\right\| + 2\left\| \frac{1}{n}\sum_{i=1}^{n}\bbone_{\{\mb x_i\in \mathcal{C} \}}\mb{ \Pi}_{\mathcal{U}}\mb x_{i}\right\|_2.
    \end{align}
    \end{comment} 
Let $\mathcal{D}$ be a collection of subsets in $\mathbb R^d$.  We denote the set of vectors whose entries are the indicator functions of $\mathcal{D}$ evaluated at samples $\{\mb x_i\}_{i=1}^n \subset \mathbb{R}^d$ by 
\begin{equation}
    \mathcal{H} (\mathcal{C}, \{\mb x_i\}_{i=1}^n) \triangleq \left\{\left(\bbone_{\{\mb x_1\in C\}},\ldots,\bbone_{\{\mb x_n\in C \}}\right): C\in \mathcal{D}\right\}.
\end{equation}
The Sauer-Shelah lemma (e.g. \citep[Section 3]{mohri2018foundations}) implies 
\begin{equation}
 \sup\limits_{\mb x_1,\ldots,\mb x_n}\left|\mathcal{H} (\mathcal{D}, \{\mb x_i\}_{i=1}^n)\right| \leq \left( \frac{en}{\vcdim(\mathcal{D})}\right)^{\vcdim(\mathcal{D})},
\end{equation}
where $\vcdim(\mathcal{D})$ denotes the Vapnik-Chervonenkis dimension of $\mathcal{D}$. 
\begin{comment}
Next, we define the restricted version of \eqref{eq:Pkds}. Let the set of all polytopes generated as the intersection of $k$ jointly $s$ sparse halfspaces in the neighborhood of $\mb \theta^\star = [\mb \theta^\star_1; \ldots;\mb \theta^\star_k]$ be defined as
\begin{equation}
   \mathcal{P}_{k,d,s}\left(\mathcal{N}(\mb \theta^\star,\Delta\rho)\right) = \bigcup_{\mathcal{U} \in \mathcal{Z}_s} \mathcal{P}_{k,d}(\mathcal{U},\mathcal{N}(\mb \theta^\star,\Delta\rho)),
\end{equation}
where
\begin{equation}
    \mathcal{P}_{k,d}\left(\mathcal{U}, \mathcal{N}(\mb \theta^\star,\Delta\rho)\right) \triangleq \left\{ \mb x \in \mathbb R^d : \mb \theta \in \mathcal{N}(\mb \theta^\star,\Delta\rho), [\mb x]_{\mathcal{U}^\mathsf{c}} = \mb 0,  \langle [\mb x;1],\mb \theta_l\rangle \geq \max_{l \in [k]\backslash\{j\}}\langle [\mb x;1],\mb \theta_j\rangle, l \in [k] \right\}.
\end{equation}
\end{comment}
Recall that each elements of $\mathcal{P}_{k,d}(\mathcal{U})$ is given as the intersection of $k$ jointly $s$-sparse halfspaces. 
Since the VC-dim of a single halfspace in $\mathbb R^s$ is $s+1$ \citep[Theorem A]{csikos2018optimal}, we have
\begin{equation}
\label{eq:poly_card_sing}
    \sup_{\mb x_1,\dots,\mb x_n}\left| \mathcal{H} (\mathcal{P}_{k,d}(\mathcal{U}), \{\mb x_i\}_{i=1}^n) \right| \leq \left( \frac{en}{s+1}\right)^{k(s+1)}, \quad \forall \mathcal{U} \in \mathcal{Z}_s.
\end{equation}
Furthemore, since $\mathcal{P}_{k,d,s} = \underset{\mathcal{U} \in\mathcal{Z}_s}{\bigcup}\mathcal{P}_{k,d}(\mathcal{U})$, it follows that 
\begin{equation}
    \mathcal{H} (\mathcal{P}_{k,d,s}, \{\mb x_i\}_{i=1}^n) = \underset{\mathcal{U} \in 
    \mathcal{Z}_s}{\bigcup} \mathcal{H} (\mathcal{P}_{k,d}(\mathcal{U}), \{\mb x_i\}_{i=1}^n).
\end{equation}
Therefore, the cardinality of $\mathcal{H} (\mathcal{P}_{k,d,s}, \{\mb x_i\}_{i=1}^n)$ satisfies 
\begin{equation}
\label{eq:ub_worstcase_H_Pkds}
    \sup_{\mb x_1,\dots,\mb x_n} \left|\mathcal{H} (\mathcal{P}_{k,d,s}, \{\mb x_i\}_{i=1}^n)\right| 
%    =  
%    \left| \bigcup_{\mathcal{U} \in \mathcal{Z}_s} \mathcal{H} (\mathcal{P}_{k,d}(\mathcal{U}), \{\mb x_i\}_{i=1}^n)\right| 
    \leq\left( \frac{en}{s+1}\right)^{k(s+1)}\left( \frac{ed}{s}\right)^{s}.
\end{equation}
Then, to upper bound the term in \eqref{eq:bnd_concent}, we use 
\begin{align}
    \label{eq:supall}
    & \sup\limits_{\substack{C \in \mathcal{P}_{k,d,s} \\ \mathcal{U} \in \mathcal{Z}_s}}
    \left\|\mb{ \tilde\Pi}_{\mathcal{U}} \left(\frac{1}{n}\sum_{i=1}^{n}\bbone_{\{\mb x_i\in C \}}\left(\mb \xi_{i}\mb \xi_{i}^\T-\mb I_{d+1}\right)\right)\mb{\tilde\Pi}_{\mathcal{U}}\right\| \nonumber \\
    &\quad\leq \sup_{\mb x_1',\ldots,\mb x_n'} ~
    \sup_{\{\omega_i\}_{i=1}^n \in \mathcal{H} (\mathcal{P}_{k,d,s}, \{\mb x_i'\}_{i=1}^n)}
    \underbrace{ \sup_{\mathcal{U} \in \mathcal{Z}_s}
    \left\|\mb{ \tilde\Pi}_{\mathcal{U}}\left(\frac{1}{n}\sum_{i=1}^{n}\omega_i\left(\mb \xi_{i}\mb \xi_{i}^\T-\mb I_{d+1}\right)\right) \mb{ \tilde\Pi}_{\mathcal{U}}\right\|}_{f(\omega_1,\dots,\omega_n)}. 
\end{align} 
%Let $\{\mb x_i'\}_{i=1}^n \subset \mathbb{R}^d$ and $\{\omega_i\}_{i=1}^n \in \mathcal{H} (\mathcal{P}_{k,d,s}, \{\mb x_i'\}_{i=1}^n)$ be arbitrarily fixed. 
Let $\{\omega_i\}_{i=1}^n \in \{0,1\}^n$ be arbitrarily fixed and $\alpha \triangleq \frac{1}{n}\sum_{i=1}^n \omega_i$. 
\begin{comment}
Then we have that $\norm{\sqrt{\omega_i}\mb x_i}_{\psi_2} \leq \sqrt{\omega_i} \eta$ for all $i \in [n]$. 
Since $\{\sqrt{\omega_i}\mb x_i\}_{i=1}^n$ are independent, by \citep[Theorem A.1]{krahmer2014suprema},    
\end{comment}
 Then, via Lemma \ref{lem:Sparse_RIP}, we obtain a tail bound on the sparsity-restricted spectral norm in the right-hand side of \eqref{eq:supall} given by
\begin{equation} \label{eq:Center_gauss}
\mathbb{P}\left(
f(\omega_1,\dots,\omega_n)
> \epsilon\right) \leq {\delta},
\end{equation}
if $n \geq C (\eta \vee 1)^4 \epsilon^{-2}\alpha\left[ s\log\left(\frac{d}{s}\right)+\log \left(\frac{1}{\delta}\right) \right]$. Trivially we have that $\alpha \leq 1$. Then, using the union bound, and inflating the probability of error $\delta$ by the worst-case cardinality in \eqref{eq:ub_worstcase_H_Pkds}, 
we obtain that 
\begin{equation} \label{eq:sRIP_bound} 
\mathbb{P}\left(    \sup\limits_{\substack{ \mb x_1',\ldots,\mb x_n' \\ \{\omega_i\}_{i=1}^n  \in \mathcal{H} (\mathcal{P}_{k,d,s}, \{\mb x_i'\}_{i=1}^n)}} f(\omega_1,\dots,\omega_n)
> \epsilon 
\right) \leq {\delta},
\end{equation}
if $n \geq C(\eta \vee 1)^4 \epsilon^{-2}\left[ sk\log\left(\frac{n}{s}\right) +s\log\left(\frac{d}{s}\right)+\log \left(\frac{1}{\delta}\right) \right]$ which concludes the proof.

{\color{black}
\subsection{Proof of Lemma \ref{lemm:poly_int}}
     By the definition of $\{\mathcal{C}_j(\mb \theta)\}_{j=1}^k$ in \eqref{eq:def_calCjstar}, we have for all $j'\neq j \in [k]$ that
    \begin{align*}
        \mb x_i \in \mathcal{C}_j(\mb \theta) \cap \mathcal{C}_{j'}(\mb \theta^\star) &\Rightarrow \langle\mb \xi_i, \mb \theta_j \rangle \geq \langle\mb \xi_i, \mb \theta_{j'} \rangle, \langle\mb \xi_i, \mb \theta^\star_{j'} \rangle\geq \langle\mb \xi_i, \mb \theta^\star_j \rangle \\
        & \Leftrightarrow \langle\mb \xi_i, \mb v_{j,j'} \rangle\geq 0, \langle\mb \xi_i, \mb v^\star_{j,j'} \rangle\leq 0 \\
        & \Rightarrow\langle\mb \xi_i, \mb v_{j,j'}\rangle  \langle\mb \xi_i, \mb v^\star_{j,j'}\rangle \leq 0 \\
        & \Rightarrow \langle\mb \xi_i, \mb v^\star_{j,j'}\rangle^2 \leq \langle\mb \xi_i, \mb v_{j,j'}-\mb v^\star_{j,j'}\rangle^2,
    \end{align*}
    where the last statement holds since $ab<0$ implies $|b| \leq |a-b|$ for all $a,b \in \mathbb R$. From \citep[Lemma 7.4]{kim2023maxaffinefull}, we have that every $\mb \theta \in \mathcal{N}(\mb \theta^\star, \Delta\rho)$ has $(\mb v_{j,j'}, \mb v_{j,j'}^\star)$ in  
    \begin{equation}\label{eq:M_v}
        \mathcal{M} \triangleq \left\{ (\mb v, \mb v^\star): \|\mb v - \mb v^\star \|\leq 2 \rho \|(\mb v^\star)_{1:d} \|\right\}.
    \end{equation}
    Also, define the set 
    \begin{equation}
        \mathcal{S}_{\mb v, \mb v^\star} \triangleq \left\{ \mb \xi: \langle\mb \xi, \mb v^\star\rangle^2 \leq \langle\mb \xi, \mb v-\mb v^\star\rangle^2\right\}
    \end{equation}
   Therefore it suffices to show that with probability $1-\delta$, we have 
    \begin{equation}
         \sup_{(\mb v, \mb v^\star )\in \mathcal{M}}\frac{1}{n}\sum_{i=1}^n \bbone_{\{\mb \xi_i \in \mathcal{S}_{\mb v, \mb v^\star}\}} \leq \frac{\pi^{2(1+\zeta^{-1})}_{\min}}{k^2},
    \end{equation}
    when the sample complexity satisfies \eqref{eq:n_int}. For fixed $\mathcal{U} \in \mathcal{Z}_s$, we have that $\mathcal{S}_{\mb v, \mb v^\star} \in \mathcal{P}_{2,d}(\mathcal{U})$ and similar to \eqref{eq:poly_card_sing} we have
    \begin{equation}
        \sup_{\mb x_1,\dots,\mb x_n}\left| \mathcal{H} (\mathcal{P}_{2,d}(\mathcal{U}), \{\mb x_i\}_{i=1}^n) \right| \leq \left( \frac{en}{s+1}\right)^{2(s+1)}.
    \end{equation}
    Therefore, for fixed $\mathcal{U} \in \mathcal{Z}_s$ we have by \citep[Lemma 6.3]{kim2023maxaffinefull} with probability at least $1-\delta$ that 
    \begin{equation}
        \sup_{(\mb v, \mb v^\star)\in \mathcal{M}}\left| \frac{1}{n}\sum_{i=1}^n\bbone_{\{\mb \xi_i \in \mathcal{S}_{\mb v, \mb v^\star }\} } - \mathbb{P} \left(\mb \xi \in  \mathcal{S}_{\mb v, \mb v^\star }\right)\right| \leq C \sqrt{\frac{\log(1/\delta)+s\log(n/s)}{n}}.
    \end{equation}
    Furthermore, this statement holds for all $\mathcal{U} \in \mathcal{Z}_s$ using the union bound when we inflate the probability of error by $|\mathcal{Z}_s|$, i.e.
    \begin{equation} \label{eq:prob_emp_2}
        \sup_{\substack{(\mb v, \mb v^\star)\in \mathcal{M} \\ \mathcal{U} \in \mathcal{Z}_s}}\left| \frac{1}{n}\sum_{i=1}^n\bbone_{\{\mb \xi_i \in \mathcal{S}_{\mb v, \mb v^\star }\} } - \mathbb{P} \left(\mb \xi \in  \mathcal{S}_{\mb v, \mb v^\star }\right)\right| \leq C \sqrt{\frac{\log(1/\delta)+s\log(n\vee d/s)}{n}},
    \end{equation}
    with probability at least $1-\delta$. Furthermore, by \citep[Lemma 7.5]{kim2023maxaffinefull} similar to \eqref{eq:subreslt0}, we have 
    \begin{equation} \label{eq:prob_int_2}
         \sup_{\substack{(\mb v, \mb v^\star)\in \mathcal{M} \\ \mathcal{U} \in \mathcal{Z}_s}} \mathbb{P} \left(\mb \xi \in  \mathcal{S}_{\mb v, \mb v^\star }\right) \leq C \left[ (2\rho)^2 \log(1/\rho)\right]^\zeta \leq C \left[ \frac{R^2\pi_{\min}^{2\zeta^{-1}(1+\zeta^{-1})}}{k^2 \zeta^{-1}}\right]^\zeta \leq C \frac{R^{2\zeta} \pi_{\min}^{2(1+\zeta^{-1})}}{k^2}
    \end{equation}
    Finally, choosing the numerical constant in \eqref{eq:n_int} large enough and $R>0$ small enough so that 
    \begin{equation}
        \mathbb{P} \left[  \sup_{\substack{(\mb v, \mb v^\star)\in \mathcal{M} \\ \mathcal{U} \in \mathcal{Z}_s}} \frac{1}{n}\sum_{i=1}^n \bbone_{\{\mb \xi_i \in \mathcal{S}_{\mb v,\mb v^\star}\}} > \left(\frac{\pi_{\min}^{(1+\zeta^{-1})}}{k}\right)^2\right] \leq \delta,
    \end{equation}
    which concludes the proof.}

\subsection{Proof of Lemma \ref{lem:combined}}
    This lemma is defined instate of the assumptions in \ref{THM:MAIN}. The proof of this lemma will simply involve invoking several auxiliary lemmas from Section \ref{sec:lemmas} which hold with high probability. We will now prove that each of the statements in Lemma \ref{lem:combined} holds with probability at least $1-\delta/5$. Statement \eqref{eq:bnd_concentstate} follows from Lemma \ref{lem:bnd_singular} since \eqref{eq:cond:lem:lwb_gradient} implies \eqref{n:bnd_singular} with $\epsilon =k^{-3/2} \pi_{\mathrm{min}}^{2(1+\zeta^{-1})}$, and   \eqref{eq:empiricalstate} holds from Lemma \ref{mu_lemma} since \eqref{eq:cond:lem:lwb_gradient} implies \eqref{n:Empirical_Measure} with $\epsilon =k^{-3/2} \pi_{\mathrm{min}}^{2(1+\zeta^{-1})}$. Also, \eqref{eq:vboundstate} follows from Lemma \ref{lem:model_difference} since $\mb \theta^t \in \mathcal{N}(\mb \theta^\star,\Delta\rho)$ and \eqref{eq:cond:lem:lwb_gradient} implies \eqref{n:difference_lemma}. Statement \eqref{eq:noise_state} follows directly from Lemma \ref{lem:noise_grad_bound}. Finally, \eqref{eq:rip_state} holds due to both Lemma \ref{lemm:poly_int}, since \eqref{eq:cond:lem:lwb_gradient} implies \eqref{eq:n_int} with probability $1-\delta/2$, and Lemma \ref{lemm:restriced_rip}, since \eqref{eq:cond:lem:lwb_gradient} implies \eqref{eq:rip_bound_n} with probability $1-\delta/2$.
{\color{black}    
\section{Proof of Theorem~\ref{Theo:maslov_approx}}\label{maslovproofapprox}

Recall that $w= g(u_1,\ldots,u_d)$ is a sparse generalized polynomial defined in \eqref{eq:poly_into}. Also recall that $y = \mathrm{Re}\{\varsigma\log w\}$ and $x_l = \varsigma \log u_l$ for all $l \in [d]$ for some temperature parameter $\varsigma> 0$.
Then \eqref{eq:poly_into} is rewritten as 
\begin{align}
y &= \mathrm{Re }\left\{\varsigma \log \left[\sum_{j=1}^k\exp\left(\frac{ \varsigma\log c_j+\sum_{l=1}^d \alpha_{j,l}x_l }{\varsigma}~ \right) \right]\right\} \nonumber \\
& = \underbrace{\mathrm{Re }\left\{\varsigma \log \left[\sum_{j=1}^k\exp\left(\frac{\varsigma \log c_j+\sum_{l=1}^d \alpha_{j,l}x_l - \max_q (\varsigma\log |c_q| + \sum_{l=1}^d \alpha_{q,l}x_l) }{\varsigma}~ \right) \right]\right\}}_{z_\varsigma} \nonumber\\
&\quad + \mathrm{Re }\left\{\varsigma \log \left[\exp\left(\frac{ \max_q (\varsigma\log |c_q| + \sum_{l=1}^d \alpha_{q,l}x_l) }{\varsigma}~ \right) \right]\right\} \nonumber \\
& = z_\varsigma + \max_j \left(\varsigma\log |c_j| + \sum_{l=1}^d \alpha_{j,l}x_l\right).
\end{align} 
This results in the following form: 
\begin{equation} \label{eq:max}
    y =   \max_{j \in [k]}\langle \mb \theta_j , [\mb x;1]\rangle +z_\varsigma,
\end{equation}
where $\mb \theta_j = [\alpha_{j,1}; \ldots; \alpha_{j,d}; \varsigma \log |c_j|]$ for all $j \in [k]$.
Note that the error in this transformation is only contained within $z_{\varsigma}$, which we would like to bound. Therefore, the worst-case approximation error is
\begin{align}
    \label{eq:phi_bound}
     \left| z_{\varsigma} \right|  &= \left|\mathrm{Re}\left\{\varsigma \log \left[ \sum_{j=1}^k \mathrm{exp}\left(\frac{\langle\mb \theta_j, [\mb x;1]  \rangle
     +\varsigma i\mathbf{1}_{\{c_j <0\}}\pi -\max_q\langle \mb \theta_q , [\mb x;1]  \rangle 
      }{\varsigma} \right) \right]\right\} \right|\nonumber\\
      &= \mathrm{Re}\left\{\varsigma \log \left| \sum_{j=1}^k \mathrm{exp}\left(\frac{\langle\mb \theta_j, [\mb x;1]  \rangle
     + \varsigma i\mathbf{1}_{\{c_j <0\}}\pi -\max_q\langle \mb \theta_q , [\mb x;1]  \rangle 
      }{\varsigma} \right) \right|\right\} \nonumber\\
     & \leq  \varsigma \left|\log \left[ 1-\sum_{j\neq q^\star} \mathrm{exp}\left(\frac{\langle \mb \theta_j - \mb \theta_{q^\star}, [\mb x;1]  \rangle }{\varsigma} \right) \right]\right|\nonumber \\
     & \leq  \varsigma \left|\log \left[ 1- (k-1)\max_{j\neq q^\star} \mathrm{exp}\left(\frac{\langle \mb \theta_j - \mb \theta_{q^\star}, [\mb x;1]  \rangle }{\varsigma} \right) \right]\right|, 
\end{align}
where $q^\star = \mathrm{argmax}_q \langle \mb \theta_q , [\mb x;1]  \rangle$. 
Furthermore, we know that for every $j\neq q \in [k]$, the anti-concentration assumption on $\mb x$ (Assumption \ref{assum:anti}) implies with probability at least $1-(\gamma \epsilon)^\zeta$ that
\begin{equation} \label{eq:anti_theta}
\langle \mb \theta_{q} - \mb \theta_{j}, [\mb x;1]  \rangle \geq \sqrt{\epsilon} \Delta,    
\end{equation}
where the minimum separation $\Delta$ is defined in \eqref{eq:defkappa} as
\[
\Delta = \min_{j\neq q \in [k]} \norm{ [\mb \theta_j]_{1:d} - [\mb \theta_q]_{1:d} }_2.
\]
A smaller $\Delta$ implies the model parameters are very similar and thus harder to estimate. 
Combining \eqref{eq:phi_bound} and \eqref{eq:anti_theta}, it holds with probability at least $1- (\gamma \epsilon)^\zeta$ that 
\begin{equation}
\label{eq:anti_exp}
   (k-1)\max_{j\neq q^\star} \mathrm{exp}\left(\frac{\langle \mb \theta_j - \mb \theta_{q^\star}, [\mb x;1]  \rangle }{\varsigma}\right)\leq   (k-1) \exp\left( \frac{-\sqrt{\epsilon}\Delta}{\varsigma}\right). 
\end{equation}
We choose a suitable value of $\varsigma >0$ such that the right-hand side of \eqref{eq:anti_exp} is upper bounded by $1/2$. Using the fact that $|\log(1-a)| \leq 2a$ whenever $a\leq 1/2$, and combining both \eqref{eq:phi_bound} and \eqref{eq:anti_exp} yields that
\begin{equation}
    \label{eq:quant_error}
    \left|z_\varsigma\right| \leq  2\varsigma (k-1) \exp\left( \frac{-\sqrt{\epsilon}\Delta}{\varsigma}\right),
\end{equation}
with probability at least $1- (\gamma \epsilon)^\zeta$ which concludes the proof.

\section{Proof of Corollary \ref{theo:bounded_SPGD}} \label{app:maslov_spgd}
Notice that the dataset $\{y_i, \mb x_i\}_{i=1}^n$ is generated according to 
\[
y_i  = \max_{j\in [k]}\langle \mb \theta_j, \mb \xi_i\rangle +z_{\varsigma,i}.
\]
Theorem \ref{THM:MAIN} provides theoretical guarantees under subGaussian noise that is independent of the covariates. The only difference here is that $\{z_{\varsigma,i}\}_{i=1}^n$ are independent copies of $z_\varsigma$ that is dependent on $\mb x$ and is bounded as
\begin{equation}
    |z_\varsigma| \leq 2\varsigma (k-1)\exp \left(\frac{\sqrt{\epsilon}\Delta}{\varsigma} \right)\triangleq M_{\varsigma},
\end{equation}
with probability at least $1- (\gamma \epsilon)^\zeta$ by Theorem \ref{Theo:maslov_approx}. Therefore, we only need to modify proof of Theorem \ref{THM:MAIN} to handle the new noise model. We use the main result of \eqref{eq:abcd} from Theorem \ref{THM:MAIN} that is written as 
\begin{align} 
\frac{1}{4}\left\|\mb h^{t+1}   \right\|_2^2
&\leq \sum_{j=1}^k \left[\left\|\mb{\tilde \Pi}_{\mathcal{U}^{t+1}} \left(\mb h^t_{j}-\mu_j^t \mb p_j\right)\right\|_2 + \mu_j^t \left(\left\| \mb q_j\right\|_2 +  \left\| \mb c_j\right\|_2+\left\|\mb d_j\right\|_2\right)\right]^2,
\end{align}
where $\mb d_j$ is the only noise-related term written as 
\begin{equation}
    \label{eq:d_spgd}
    \mb d_j= \frac{1}{n}\sum_{i=1}^n \bbone_{\{\mb x_i  \in \mathcal{C}_j^t\} }z_{\varsigma, i} \tilde{\mb \Pi}_{\mathcal{U}^{t+1} }\mb \xi_i.
\end{equation}
Therefore, we only need to find an upper bound on $\| \mb d_j\|_2$. We know by the variational characterization of the $\ell_2$ norm that 
\begin{align}
    \label{eq:variationl2}
    \| \mb d_j\|_2 &\leq \sup_{\mb u \in B_2, |\lambda| \leq 1} \left | \frac{1}{n}\sum_{i=1}^n \bbone_{\{\mb x_i  \in \mathcal{C}_j^t\} }z_{\varsigma, i} \langle\tilde{\mb \Pi}_{\mathcal{U}^{t+1} }\mb \xi_i,[\mb u; \lambda] \rangle\right| \nonumber\\ 
    & \leq \sup_{\mb u \in B_2} \left | \frac{1}{n}\sum_{i=1}^n \bbone_{\{\mb x_i  \in \mathcal{C}_j^t\} }z_{\varsigma, i} \langle\mb \Pi_{\mathcal{U}^{t+1} }\mb x_i,\mb u  \rangle\right| + \sup_{|\lambda|\leq 1} \left | \frac{1}{n}\sum_{i=1}^n \bbone_{\{\mb x_i  \in \mathcal{C}_j^t\} }z_{\varsigma, i} \lambda \right|. 
\end{align}
The first summand of \eqref{eq:variationl2} can be bounded as
\begin{align}\label{eq:left_bounded}
    \sup_{\mb u \in B_2} \left | \frac{1}{n}\sum_{i=1}^n \bbone_{\{\mb x_i  \in \mathcal{C}_j^t\} }z_{\varsigma, i} \langle\mb \Pi_{\mathcal{U}^{t+1} }\mb x_i,\mb u  \rangle\right| &\leq \max_{i \in [n]}|z_{\varsigma,i}| \cdot \sup_{\mb u \in B_2}  \frac{1}{n}\sum_{i=1}^n \bbone_{\{\mb x_i  \in \mathcal{C}_j^t\} }\left | \langle\mb \Pi_{\mathcal{U}^{t+1} }\mb x_i,\mb u  \rangle\right| \nonumber \\
    & \leq M_\varsigma \cdot  \sup_{\mb u \in B_2}  \frac{1}{n}\sum_{i=1}^n \bbone_{\{\mb x_i  \in \mathcal{C}_j^t\} }\left | \langle\mb \Pi_{\mathcal{U}^{t+1} }\mb x_i,\mb u  \rangle\right| \nonumber  \\
    & \leq C M_\varsigma \sqrt{\pi_j^t + k^{-3/2}\pi_{\min}^{2(1+\zeta^{-1})}}\leq C M_\varsigma \sqrt{1 + k^{-3/2}\pi_{\min}^{2(1+\zeta^{-1})}}
\end{align}
where the last inequality holds with probability at least $1-\delta$ by Lemma \ref{lem:combined}, Eq. \eqref{eq:empiricalstate} which holds with probability at least $1-\delta/2$, and Lemma \ref{lemm:abs_ysv} since \eqref{eq:cond:lem:lwb_gradient} implies \eqref{eq:abs_ysv_n} with probability $1-\delta/2$. 
The second summand of \eqref{eq:variationl2} can be bounded as 
\begin{align} \label{eq:right_bounded}
    \sup_{|\lambda|\leq 1} \left | \frac{1}{n}\sum_{i=1}^n \bbone_{\{\mb x_i  \in \mathcal{C}_j^t\} }z_{\varsigma, i} \lambda \right| &\leq  \max_{i\in [n]} |z_{\varsigma,i}| \cdot \left | \frac{1}{n}\sum_{i=1}^n \bbone_{\{\mb x_i  \in \mathcal{C}_j^t\} }  \right|  \nonumber \\
    &\leq M_{\varsigma}\left( \pi_j^t + k^{-3/2}\pi_{\min}^{2(1+\zeta^{-1})}\right)\leq  M_{\varsigma}\left( 1 + k^{-3/2}\pi_{\min}^{2(1+\zeta^{-1})}\right),
\end{align}
    where the second inequality holds with probability at least $1-\delta$ by Lemma \ref{lem:combined}. Combining \eqref{eq:left_bounded} and \eqref{eq:right_bounded} yields
    \[
    \norm{\mb d_j}_2 \leq C M_{\varsigma}\left( 1 + k^{-3/2}\pi_{\min}^{2(1+\zeta^{-1})}\right),
    \]
    with probability at least $1-[\delta +(\gamma \epsilon)^\zeta]$ for every $j \in [k]$. Therefore, we can use \eqref{finaltau} as
    \begin{equation}
        \| \mb h^{t+1}\|^2_2 \leq \tau  \| \mb h^{t}\|^2_2 + 12k\left[\frac{\| \mb d_1\|_2 }{(1-\varrho)\pi_{\min} - \epsilon_{\mathrm{min}}}\right]^2 \leq \tau  \| \mb h^{t}\|^2_2 + Ck\left[\frac{ M_\varsigma\left(1+ k^{-3/2}\pi_{\min}^{2(1+\zeta^{-1})} \right)}{(1-\varrho)\pi_{\min} - \epsilon_{\mathrm{min}}}\right]^2,
    \end{equation}
    for some $\tau \in [0,1)$. By the recursive nature we have that
\begin{align} 
    \left\|\mb \theta^{t+1}-\mb \theta^\star\right\|^2_2 &\leq \tau^{t+1} \left\|\mb \theta^0-\mb \theta^\star\right\|^2_2 + \frac{Ck}{1-\tau}\left[\frac{ M_\varsigma\left(1++ k^{-3/2}\pi_{\min}^{2(1+\zeta^{-1})} \right)}{(1-\varrho)\pi_{\min} - \epsilon_{\mathrm{min}}}\right]^2\nonumber \\
    & \leq \tau^{t+1} (\Delta \rho)^2 + (\Delta \rho)^2 \leq 2(\Delta \rho)^2,
\end{align}
where the first inequality follows from $\mb \theta ^0 \in \mathcal{N}(\mb \theta^\star,\sqrt{2}\Delta\rho)$ and for a sufficiently small choice of $\varsigma$. This yields
\begin{equation}
    \left\|\mb \theta^{t+1}-\mb \theta^\star\right\|^2_2 \leq  2(\Delta\rho)^2 \implies \mb \theta^{t+1} \in \mathcal{N}(\mb \theta^\star,\sqrt{2}\Delta\rho),
\end{equation}
which concludes the proof using the strong law of induction.

 \section{Auxiliary Lemma for Theorem \ref{theo:bounded_SPGD}}    \begin{lemma}\label{lemm:abs_ysv}
     Let $\delta\in (0,e^{-1})$ and $\alpha \in (0,1)$. Let $\{\mb x_i\}_{i=1}^n$ be independent copies of a random vector $\mb x$ that satisfies Assumption $\ref{assum:subg}$. Then, with probability at least $1-\delta$ we have
    \begin{equation}
        \sup_{\substack{\mathcal{I}:|\mathcal{I}|\leq\alpha n\\ \mathcal{U} \in\mathcal{Z}_s}}\norm{ \frac{1}{n}\sum_{i\in \mathcal{I}} \mb \Pi_{\mathcal{U}}\mb x_i}_2  \leq C (\eta^2 \vee 1) \sqrt{\alpha},
    \end{equation}
    if
    \begin{equation}\label{eq:abs_ysv_n}
        n\geq \alpha^{-1}\left[s\log(d/s)+\log(1/\delta)\right].
     \end{equation}
 \end{lemma}
 \begin{proof}
 Assume $\mathcal{U}$ is fixed and let $\mb x' = [\mb x]_{\mathcal{U}} \in \mathbb R^s$. Define the collection of all possible activation vectors $\mathcal{A}_{\alpha}= \left\{\mb a \in  \{0,1\}^n: \|\mb a\|_{0} = \alpha n\right\}$. We can now define the random process 
 \[
 Y_{\mb a , \mb v} =\sum_{i=1}^n [\mb a]_i |\langle \mb x'_i, \mb v \rangle|,
 \]
 for $\mb a\in \mathcal{A}_\alpha$ and $\mb v \in B_2$.
 Notice by the variational characterization of the $\ell_2$ norm that 
 \[
 \sup_{\substack{\mathcal{I}:|\mathcal{I}|\leq\alpha n\\ \mathcal{U} \in\mathcal{Z}_s}}\norm{ \frac{1}{n}\sum_{i\in \mathcal{I}} \mb \Pi_{\mathcal{U}}\mb x_i}_2 \leq  \sup_{\mb a\in \mathcal{A}_{\alpha}, \mb v \in B_2}Y_{\mb a , \mb v},
 \]
 so it suffices to find an upper bound of the right-hand side.
 First, we have that 
 \[
  Y_{\mb a , \mb v} -   Y_{\mb a' , \mb v}=\sum_{i=1}^n ([\mb a]_i-  [\mb a']_i) |\langle \mb x'_i, \mb v \rangle|.
 \]
 Using Hoeffding's inequality, we have that 
 \begin{equation} \label{eq:Ysv1}
 \mathbb  P \left(| Y_{\mb a , \mb v} -   Y_{\mb a' , \mb v}| \geq t \right) \leq 2 \exp \left(\frac{-ct^2}{\eta^2\|\mb a -\mb a'\|_2^2} \right).
\end{equation}
 Furthermore, we have that
 \[
  Y_{\mb a , \mb v} -   Y_{\mb a , \mb v'}=\sum_{i=1}^n [\mb a]_i\left( |\langle \mb x'_i, \mb v \rangle| -|\langle \mb x'_i, \mb v'\rangle|\right)\leq \sum_{i=1}^n [\mb a]_i\left( |\langle \mb x'_i, \mb v - \mb v'\rangle|\right),
 \]
 Since $\|\mb a\|_2^2 = \alpha n$, the by Hoeffding's inequality, we have that 
 \begin{equation}\label{eq:Ysv2}
 \mathbb{P}\left(|Y_{\mb a , \mb v} -   Y_{\mb a , \mb v'}|\geq t \right) \leq 2\exp\left( \frac{-c}{\alpha n \|\mb v- \mb v'\|}\right).
\end{equation}
Equations \eqref{eq:Ysv1} and \eqref{eq:Ysv2} are sufficient to invoke \citep[Lemma 5.4]{tan2019phase} which \citep[Theorem 5.7]{tan2019phase} requires. Let $\delta \in (0,1/e)$, then \citep[Theorem 5.7]{tan2019phase} implies that with probability at least $1-\delta$   that 
\[
\sup_{\mb a \in \mathcal{A}_{\alpha}, \mb v \in B_2}Y_{\mb a, \mb v} \leq C(\eta^2 \vee 1)\sqrt{\alpha}n,
\]
if $n\geq  C\alpha^{-1}[s+\log(1/\delta)]$. Using the union bound (to account for any possible $\mathcal{U}\in \mathcal{Z}_s$) and inflating the probability of error $\delta$ by $|\mathcal{Z}_s| = $$\binom{d}{s}$ $\leq \left(\frac{ed}{s}\right)^{s}$ completes the proof.  
 \end{proof}
 }
\end{appendices}
\end{document}